\documentclass{article}

 \usepackage[main, final]{neurips_2026}
\usepackage[utf8]{inputenc} % allow utf-8 input
\usepackage[T1]{fontenc}    % use 8-bit T1 fonts
\usepackage{hyperref}       % hyperlinks
\usepackage{url}            % simple URL typesetting
\usepackage{booktabs}       % professional-quality tables
\usepackage{amsfonts}       % blackboard math symbols
\usepackage{nicefrac}       % compact symbols for 1/2, etc.
\usepackage{microtype}      % microtypography
\usepackage{xcolor}         % colors
\usepackage{amsmath}
\usepackage{graphicx}
\usepackage{amssymb}
\usepackage{amsthm}
\usepackage{algorithm}
\usepackage{algpseudocode}
\usepackage{bm}
\usepackage{booktabs}
\usepackage{subcaption}
\usepackage{multirow}
\usepackage{thm-restate}
\usepackage{wrapfig}
\usepackage{makecell}
\usepackage{pifont}
\usepackage{diagbox}

\title{\raisebox{-0.2\height}{\includegraphics[height=1.1em]{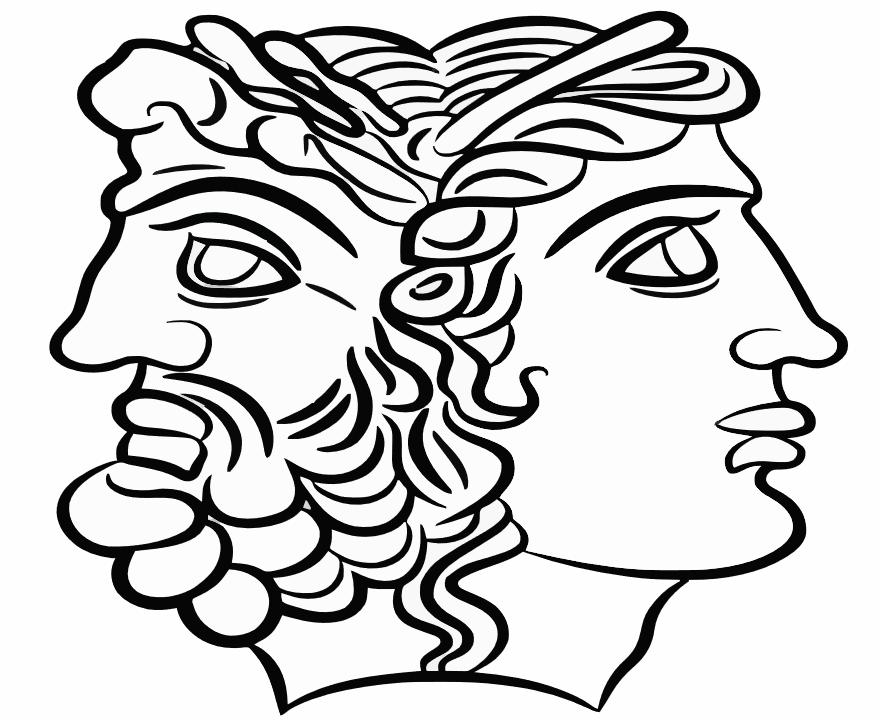}}
   Past, Future, All at Once: Mitigating Stability-Plasticity Dilemma via\\ Post-hoc JANUS Rectification}
\author{
  \textbf{Zhilong Zheng}\textsuperscript{1,2,*}, 
  \textbf{Letian Tao}\textsuperscript{1,2,*},
  \textbf{Yujie Yang}\textsuperscript{1,2}, 
  \textbf{Wei Xiong}\textsuperscript{2}, 
  \textbf{Kehua Sheng}\textsuperscript{2}\\ 
  \textbf{Bo Zhang}\textsuperscript{2}, 
  \textbf{Jingliang Duan}\textsuperscript{1,2}, 
  \textbf{Keqiang Li}\textsuperscript{1}, 
  \textbf{Yang Guan}\textsuperscript{1,2,\textdagger},  
  \textbf{Shengbo Eben Li}\textsuperscript{1,\textdagger} \\[0.5ex]
  \textsuperscript{1}School of Vehicle and Mobility \& College of AI, Tsinghua University \\
  \textsuperscript{2}DiDi Voyager Labs, DiDi Autonomous Driving \\
  *Equal contribution, \textsuperscript{\textdagger}Corresponding authors: \{yguan,lishbo\}@tsinghua.edu.cn \\
  \url{https://github.com/zhengzl18/JANUS}
}
\begin{document}

\maketitle

\begin{abstract}
Fine-tuning foundation models on new tasks inevitably suffer from catastrophic forgetting.
While existing works attempt to mitigate this on the basis of parameter-efficient fine-tuning methods, they adopted an overly restrictive Subspace Orthogonality condition.
In this paper, we introduce a purely \textit{post-hoc} and \textit{tuning-agnostic} weight rectification framework that achieves \textit{Parameter Space Orthogonality}, which is the necessary and sufficient condition for preserving historical performance to the first order.
By projecting parameter updates into the JAcobian NUll Space (JANUS), our method significantly recovers compromised historical knowledge without interfering with the underlying fine-tuning process.
To overcome the local validity of the Jacobian approximation, we further propose a Multi-step Adaptive Rectification mechanism that utilizes the JANUS shift to dynamically verify the valid trust region and adjust step sizes.
Coupled with our proposed ghost projection, ghost orientation comparison, and sequence-level singular value decomposition compression techniques, JANUS also achieves great temporal and spatial efficiency.
Experiments demonstrate that JANUS seamlessly integrates with various fine-tuning methods, significantly mitigating the stability-plasticity dilemma by recovering historical knowledge while preserving downstream task adaptation.
\end{abstract}

\section{Introduction}
\label{sec: intro}
In the era of large-scale models, fine-tuning (FT), ranging from supervised learning \cite{ouyang2022training} to reinforcement learning \cite{li2023reinforcement, guan2021direct}, has become a fundamental paradigm.
Adapting foundation models to downstream tasks have achieved unprecedented success across a wide range of domains, including question-answering \cite{devlin2019bert}, math \cite{ yu2023metamath}, code \cite{luo2023wizardcoder} and instruction following \cite{zheng2023judging,xuwizardlm}.
% To adapt the pretrained models to specific downstream tasks or align them with human preferences, fine-tuning is an indispensable step \cite{devlin2019bert}.
However, FT on new tasks inevitably suffers from catastrophic forgetting (CF) \cite{kemker2018measuring}, where the model significantly loses its previously acquired general knowledge and abilities while learning new skills.
This poses the well-known stability-plasticity dilemma: trading off between retaining historical knowledge (stability) and subsuming new information (plasticity). 

Concurrently, Parameter-Efficient Fine-Tuning (PEFT) is also essential due to the massive scale of modern large models \cite{houlsby2019parameter}.
Low-Rank Adaptation (LoRA) \cite{hu2022lora} is arguably the most widely adopted PEFT method, where the weight update $\Delta \bm W\in\mathbb{R}^{ d_\text{out}\times d_\text{in}}$ is given by the product of two low rank matrices $\bm A\in\mathbb{R}^{r\times d_\text{in}}$ and $\bm B\in\mathbb{R}^{ d_\text{out}\times r}$.
Recent literature has extensively investigated strategies to alleviate CF within LoRA-based FT frameworks.
One paradigm is to pursue \textit{output invariance}, ensuring that for every linear layer, the weight update $\Delta \bm{W}$ does not alter the layer's output \cite{qiao2025gradient,luo2026keeplora,yang2024corda,wang2025milora,tang2026put,saha2021gradient,liang2024inflora,luo2025sc}.
Another paradigm, which aligns more closely with the ultimate objective of forgetting-free FT, focuses on \textit{loss invariance}, ensuring that the $\Delta\bm W$ of any linear layer does not increase the loss on historical tasks \cite{wang2023orthogonal,cao2025orthogonal}.
% \cite{qiao2025gradient,saha2021gradient,liang2024inflora,luo2026keeplora,wang2023orthogonal,tang2026put,wang2025milora,yang2024corda}.
Their shared core idea is to enforce orthogonality, either achieved by projection \cite{qiao2025gradient,luo2026keeplora,saha2021gradient,liang2024inflora},  initialization \cite{yang2024corda,wang2025milora,tang2026put,luo2025sc}, or by penalty \cite{wang2023orthogonal,cao2025orthogonal}.
% The core idea is to enforce orthogonality, either achieved by projection, as in PEGP \cite{qiao2025gradient}, GPM\cite{saha2021gradient}, InfLoRA\cite{liang2024inflora},KeepLoRA\cite{luo2026keeplora}, or through initialization, as in CorDA\cite{yang2024corda}, MiLoRA\cite{wang2025milora}, LoRA-Null\cite{tang2026put}, or by penalty, as in O-LoRA\cite{wang2023orthogonal}.
However, all these methods essentially strive for \textit{Subspace Orthogonality}. 
They aim to make the row space of $\Delta \bm W$ orthogonal to the column space of a specific matrix $\bm M$, i.e., $\Delta \bm W \bm M=0$. 
% This constraint is designed to achieve layer-wise output-invariance ($\Delta Y=0$) or layer-wise performance-invariance ($\Delta \mathcal{L}=0$). 
For output invariance, the matrix $\bm M$ comes from the layer inputs $\bm X$ of historical tasks \cite{qiao2025gradient,luo2026keeplora,yang2024corda,wang2025milora,tang2026put,saha2021gradient,liang2024inflora,luo2025sc}.
For loss invariance, $\bm M$ is the parameter updates from historical tasks which is assumed to be a surrogate for the historical gradients \cite{wang2023orthogonal,cao2025orthogonal}.
The problem is that the subspace orthogonality constraint $\Delta \bm W \bm M=\bm 0$ is overly restrictive; it is a sufficient but unnecessary condition for forgetting-free FT, making it extremely difficult to satisfy in practice. 
Furthermore, methods forcing this orthogonality through initialization face a dilemma: the orthogonality guarantee relies on a specific initialization of the matrix $\bm A$, which requires it to be frozen, while a frozen $\bm A$ hinders plasticity. 
Moreover, these special initializations are inherently incompatible with more advanced LoRA initialization strategies, such as PiSSA \cite{meng2024pissa}.

To address these theoretical and practical bottlenecks, we introduce the \textit{Parameter Space Orthogonality} condition ($\text{Tr}(\Delta \bm W\bm M)=\bm 0$) that is necessary and sufficient for first-order forgetting-free FT.
Tab.~\ref{tab:forgetting-free cond} gives a comparison of different forgetting-free conditions.
Given the massive parameter space of large models, there is a vast space for this condition to hold, while maintaining plasticity.
Leveraging this insight, we propose JANUS (JAcobian NUll Space projection), a purely post-hoc and tuning-agnostic weight rectification framework.
% To achieve this condition, we propose a purely \textit{post-hoc} and \textit{tuning-agnostic} weight rectification framework that achieves \textit{Parameter Space Orthogonality} via JAcobian NUll Space (JANUS) projection, as illustrated in Fig.~\ref{fig:intro}. 
% In short, our method relaxes the strict subspace orthogonality constraint to $\text{Tr}(\Delta \bm W\bm M)=\bm 0$, which is the equivalent condition for historical performance invariance. 
As illustrated in Fig.~\ref{fig:intro}, JANUS rectifies the parameter updates $\Delta \bm W$ after the FT process is complete, and imposes no assumption or requirement on the FT method.
In other words, it recovers the historical knowledge that might have been compromised during the unconstrained FT phase.
The advantages of this post-hoc and tuning-agnostic nature are threefold: (1) it is inherently compatible with any FT method; (2) it can be applied directly to off-the-shelf fine-tuned models without requiring re-training; and (3) it adds no overhead to the FT process.
To overcome the local validity of the Jacobian approximation, we introduce the \textit{JANUS shift} as a proxy metric for forgetting.
Based on this metric, we design a Multi-step Adaptive Rectification mechanism that safely navigates the parameter space along the loss contour, progressively advancing toward the region of high plasticity. 
To ensure practical scalability, we develop ``ghost projection'' and ``ghost orientation comparison'' techniques and sequence-level Singular Value Decomposition (SVD) compression, which relieve the burden of instantiating the full Jacobian matrix or performing per-sample backpropagation. 
Consequently, we can execute the post-rectification of a LLaMA-2-7b model in just 30 minutes using a single NVIDIA A100 GPU.

% \begin{figure}
%   \centering
%   \includegraphics[width=\textwidth,trim=0cm 0cm 0cm 0cm, clip]{fig/intro.pdf}
%   \caption{Comparison between existing and our methods. Unlike conventional approaches (left) that impose subspace orthogonality during fine-tuning, our method (right) is post-hoc: it takes the already fine-tuned weights ($\bm W_0 + \Delta \bm W$) and rectifies them to ensure parameter space orthogonality.}
% \label{fig:intro}
% \end{figure}

\begin{figure}[t]
  \centering
  \begin{minipage}[t]{0.55\textwidth}
    \centering
    \vspace{0pt}
    \small
    \resizebox{\textwidth}{!}{
    \begin{tabular}{c|cc}
        \toprule
        \diagbox[innerwidth=2cm]{\makecell[c]{\vspace{0.1cm} \\ \textbf{Constraint\quad}}}{\makecell[c]{\vspace{0.001cm} \\ \textbf{Target}}} & \makecell[c]{Output \\ invariance \\ $\textcolor{blue}{\bm M}=\bm X$} & \makecell[c]{Loss \\ invariance \\ $\textcolor{blue}{\bm M} = \nabla_{\bm W^\top}\mathcal{L}_A$} \\
        \midrule
        \makecell[c]{Subspace \\ orthogonality \\ $\Delta \bm W \textcolor{blue}{\bm M}=\bm 0$ \\ Suff. \ding{51}, Nec. \ding{55}} & \makecell[c]{\cite{qiao2025gradient,luo2026keeplora,yang2024corda} \\ \cite{wang2025milora,tang2026put,saha2021gradient} \\ \cite{liang2024inflora,luo2025sc}}& \cite{wang2023orthogonal,cao2025orthogonal} \\
        \midrule
        \makecell[c]{Parameter space \\ orthogonality\\ $\text{Tr}(\Delta \bm W \textcolor{blue}{\bm M})=0$ \\ Suff. \ding{51}, Nec. \ding{51}} & -- & JANUS \\
        \bottomrule
    \end{tabular}
    }
    \makeatletter\def\@captype{table}\makeatother % 强制将 caption 类型转为 table
    \caption{Comparison of forgetting-free conditions. \textbf{Suff.}: sufficient; \textbf{Nec.}: necessary. $\mathcal{L}_A$: past task loss.}
    \label{tab:forgetting-free cond}
  \end{minipage}
    \hfill
  \begin{minipage}[t]{0.403\textwidth}
    \centering
    \vspace{0pt}
  \includegraphics[width=\textwidth,trim=0cm 0cm 0cm 0cm, clip]{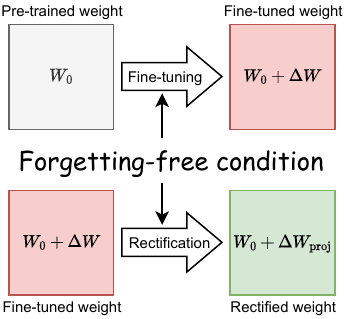}
  \caption{Comparison of active phases.}
  % \caption{Comparison of phases of taking effect. Unlike conventional approaches (left) that impose subspace orthogonality during fine-tuning, our method (right) is post-hoc: it takes the already fine-tuned weights ($\bm W_0 + \Delta \bm W$) and rectifies them to ensure parameter space orthogonality.}
\label{fig:intro}
  \end{minipage}
\end{figure}

The main contributions of our work are summarized as follows:
\begin{itemize}
    \item \textbf{Parameter Space Orthogonality Theory:} We systematically reveal that the widely adopted subspace orthogonality ($\Delta \bm W\bm M=\bm 0$) is overly strict and practically infeasible. Instead, we introduce the parameter space orthogonality ($\text{Tr}(\Delta \bm W\bm M)=\bm 0$), the equivalent condition for preventing CF, which is highly feasible for large models.
    
    \item \textbf{Post-hoc Multi-step Adaptive Rectification in JANUS:} We propose the JANUS shift as a proxy indicator of forgetting and design a plug-and-play, tuning-agnostic algorithm which adaptively and gradually rectifies parameter updates into the JANUS through projection to recover historical knowledge. It incurs no training overhead and seamlessly integrates with arbitrary FT methods. Furthermore, our proposed ``ghost'' operations and SVD compression ensure both temporal and spatial efficiency.
    
    \item \textbf{Empirical Validation of the Stability-Plasticity Breakthrough:} Extensive experiments across diverse models and tasks confirm that JANUS significantly mitigates the stability-plasticity dilemma. Our method consistently pushes the Pareto front outward, achieving near-perfect knowledge recovery with negligible degradation to plasticity.  Additionally, we empirically validate the use of JANUS shift as an indicator of forgetting by revealing a statistically significant positive correlation between this metric and model stability.
\end{itemize}

\section{Related Works}
\paragraph{Parameter-Efficient Fine-Tuning}
With the exponential growth in the parameter scale of large models, full FT has become computationally prohibitive \cite{houlsby2019parameter}. 
To address this, PEFT techniques adapt models to downstream tasks with only a marginal fraction of parameters, such as partial FT \cite{zaken2022bitfit,zhao2020masking,sung2021training,ansell2022composable,fu2023effectiveness} and parameter-efficient adaptation \cite{hu2022lora,rebuffi2017learning,lin2020exploring,aghajanyan2021intrinsic}.
Among these, LoRA \cite{hu2022lora} has emerged as the standard paradigm by injecting trainable low-rank matrices into the frozen pretrained weights. 
Building upon LoRA, a series of work seek to enhance representational capacity \cite{meng2024pissa,zhang2023adalora, li2023losparse, liu2024dora}. 
For instance, PiSSA \cite{meng2024pissa} leverages SVD to identify and separate the essential low-rank parts of the model to be the initializations of the adapters, achieving better plasticity and quicker convergence.
Being purely post-hoc and tuning-agnostic, our proposed weight rectification in JANUS method seamlessly integrates any PEFT technique, even full FT.

\paragraph{Mitigating Catastrophic Forgetting}
Catastrophic forgetting is a foundational challenge in continual learning \cite{kemker2018measuring, wang2024comprehensive}, where models drastically degrade in historical task performance when adapting to new data.
Plenty of methods have been proposed to mitigate this problem, including regularization-based \cite{kirkpatrick2017overcoming, ritter2018online, liu2018rotate, wang2021afec}, optimization-based \cite{lopez2017gradient, tang2021layerwise,riemer2019learning,farajtabar2020orthogonal} and architecture-based \cite{yan2021dynamically,aljundi2017expert,mallya2018piggyback,serra2018overcoming} approaches.
Regularization methods constrain updates during training, usually through an old-task penalty, and introduce an additional coefficient to tune.
Orthogonality-based methods instead treat the forgetting-free condition as a constraint, avoiding the trade-off introduced by tuning a penalty coefficient.
Architecture-based methods preserve old tasks by allocating additional task-dependent capacity.
In general, it is preferred that the model's parameter count does not increase as the number of tasks grows.
To be applied on modern large-scale models, the ability to integrate with PEFT methods is also essential.
Consequently, many works seek to mitigate CF built upon LoRA \cite{qiao2025gradient,luo2026keeplora,yang2024corda,wang2025milora,tang2026put,saha2021gradient,liang2024inflora,luo2025sc,wang2023orthogonal,cao2025orthogonal}.
These methods emphasize subspace orthogonality to prevent interference, which poses restriction on both stability and plasticity. 
In contrast, our work emphasizes parameter space orthogonality, establishing a highly compatible and efficient mechanism for unlocking the potential of models and mitigating the stability-plasticity dilemma.

\section{Theoretical Analysis of Forgetting-Free Mechanisms}
\label{sec:theoretical_analysis}
% In this section, we theoretically analyze the conditions required for forgetting-free FT. Through this analysis, we reveal the common underlying paradigm of existing methods and introduce our novel perspective of parameter space orthogonality.

\subsection{Problem Formulation}
Given network parameters $\bm \theta_0 \in \mathbb{R}^n$ pretrained on a past task $\mathcal{T}_A$ with dataset $\mathcal{D}_A$, where $n$ is the total number of parameters, our goal is to fine-tune the network to obtain updated parameters $\hat{\bm{\theta}}$ on a new task $\mathcal{T}_B$ while ensuring performance stability on $\mathcal{T}_A$. Mathematically, this requires the loss on the previous task to remain invariant, i.e., $\mathcal{L}_A(\hat{\bm{\theta}}; \mathcal{D}_A) \approx \mathcal{L}_A\left(\bm{\theta}_0; \mathcal{D}_A\right)$.

\subsection{The Conditions Towards Forgetting-Free Fine-Tuning}
Consider a specific linear layer in the network parameterized by $\bm W \in \mathbb{R}^{d_{\text{out}} \times d_{\text{in}}}$. Let $\bm X \in \mathbb{R}^{d_{\text{in}} \times T}$ denote its layer input under a specific data point $d_{A,i} \in \mathcal{D}_A$, and $\bm Y = \bm W\bm X \in \mathbb{R}^{d_{\text{out}} \times T}$ denote its corresponding output, where $T$ is the sequence length.
Based on the first-order Taylor expansion, for a small weight perturbation $\Delta \bm W$, the change in the loss function $\mathcal{L}_{A}^{(i)}$ evaluated on $d_{A,i}$ can be approximated as \cite{magnus2019matrix}

\begin{equation}
\Delta\mathcal{L}_{A}^{(i)} \approx \text{Tr}\left( \Delta \bm{W} \nabla_{\bm{W}^\top} \mathcal{L}_{A}^{(i)} \right) = \text{Tr}\left( \Delta \bm{W} \bm X \nabla_{\bm{Y}^\top} \mathcal{L}_{A}^{(i)} \right),
\end{equation}

where $\nabla_{\bm{W}} \mathcal{L}_{A}^{(i)} \in \mathbb{R}^{d_{\text{out}} \times d_{\text{in}}}$ and $\nabla_{\bm{Y}} \mathcal{L}_{A}^{(i)} \in \mathbb{R}^{d_{\text{out}} \times T}$ are the gradients of the loss w.r.t. the weights and the layer outputs, respectively, and we utilize the chain rule $\nabla_{\bm{W}} \mathcal{L}_{A}^{(i)} = \nabla_{\bm{Y}} \mathcal{L}_{A}^{(i)} \bm X^\top$. Consequently, the necessary and sufficient condition for a first-order forgetting-free update (i.e., $\Delta\mathcal{L}_{A}^{(i)} \approx 0$) is:

\begin{equation}
\text{Tr}\left( \Delta \bm{W} \bm X \nabla_{\bm{Y}^\top} \mathcal{L}_{A}^{(i)} \right) = 0.
\label{eq:nec_suf_cond}
\end{equation}

Existing methods attempt to achieve this either by restricting $\Delta \bm{W} \bm{X}=0$ (as in \cite{qiao2025gradient,luo2026keeplora,yang2024corda,wang2025milora,tang2026put,wang2023orthogonal}), or by forcing $\Delta \bm{W} \bm X  \nabla_{\bm{Y}^\top} \mathcal{L}_{A}^{(i)} = 0$ (as in \cite{saha2021gradient,liang2024inflora}). 
Clearly, both of these formulations are sufficient but not necessary conditions for Eq.~\eqref{eq:nec_suf_cond}. They impose excessively strong constraints that can hardly be satisfied in practice. To illustrate this, let $\bm M \in \mathbb{R}^{d_{\text{in}} \times p}$ serve as a generalized representation for either the layer input $\bm X \in \mathbb{R}^{d_{\text{in}} \times T}$ or the gradient matrix $ \nabla_{\bm{W}^\top} \mathcal{L}_{A}^{(i)} \in \mathbb{R}^{d_{\text{in}} \times d_{\text{out}}}$.
The general formulation of these existing works can be abstracted as:

\begin{equation}
\Delta \bm{W} \bm M = 0.
\label{eq:subspace_ortho}
\end{equation}

This strict subspace orthogonality requires finding $d_{\text{out}}$ distinct directions within $\mathbb{R}^{d_{\text{in}}}$, each of which is orthogonal to $p$ directions.
% Provided that $\bm M$ is of full rank, 
This is mathematically feasible only when $d_{\text{in}} - \text{rank}\left(\bm{M}\right) \geq d_{\text{out}}$.
% This implies we roughly need $d_{\text{in}} - p \geq d_{\text{out}}$.
% This is generally only feasible when the left null space of $\bm{M}$ has sufficiently high dimensions, i.e., $\text{dim}\left(\mathcal{N}\left(\bm{M}^\top\right)\right) \geq d_{\text{out}}$. Since $\text{dim}\left(\mathcal{N}\left(\bm{M}^\top\right)\right) = d_{\text{in}} - \text{rank}\left(\bm{M}\right) \geq d_{\text{in}} - p$, this implies we roughly need $d_{\text{in}} - p \geq d_{\text{out}}$. 
Unfortunately, this requirement is rarely met.
The situation exacerbates when considering a batch of $m$ data points, where each sample corresponds to a matrix $\bm M_i \in \mathbb{R}^{d_{\text{in}} \times p}$, collectively forming a joint matrix $\bm{M} \in \mathbb{R}^{d_{\text{in}} \times mp}$.
This joint matrix can easily have full row rank and thus no non-trivial left null space. 
Intuitively speaking, prior works attempt to mitigate forgetting by rigorously constraining every single row of every linear layer individually.
Not only does this overly restrictive constraint hinder the model's plasticity for adapting to the new task, but the inevitable violation of this strict condition in practice also undermines the theoretical foundation of these mechanisms.

However, if we revert to the necessary and sufficient condition defined in Eq.~\eqref{eq:nec_suf_cond}, the feasible space for $\Delta \bm{W}$ is vastly expanded. Specifically, Eq.~\eqref{eq:nec_suf_cond} can be reformulated into a vectorized form:

\begin{equation}
\text{rvec}\left(\nabla_{\bm{W}} \mathcal{L}_{A}^{(i)}\right)^\top \text{rvec}\left(\Delta \bm{W}\right) = 0,
\label{eq:vec_cond}
\end{equation}

where $\text{rvec}(\cdot)$ is the row-major vectorization operator:
$\text{rvec}\left(\bm A\right) \triangleq \left[\bm{a}_1^\top, \bm{a}_2^\top, \dots, \bm{a}_p^\top\right]^\top \in \mathbb{R}^{pq}, \quad \forall \bm A=\left[\bm{a}_1, \bm{a}_2, \dots, \bm{a}_p\right]^\top \in \mathbb{R}^{p \times q}.$
In its vectorized form, Eq.~\eqref{eq:vec_cond} considers all the parameters in $\Delta \bm W$ as a whole rather than row by row.
It merely requires finding a single direction in the parameter space $\mathbb{R}^{d_{\text{in}} d_{\text{out}}}$ that is orthogonal to one specific gradient direction.
This is why it is termed parameter space orthogonality.
Even when aggregating over $m$ data points, satisfying this condition remains highly tractable due to the massive dimensionality of $d_{\text{in}} d_{\text{out}}$ in modern large-scale models. 

By shifting our perspective from the strict subspace orthogonality characterized by Eq.~\eqref{eq:subspace_ortho} to the relaxed parameter space orthogonality formulated in Eq.~\eqref{eq:vec_cond}, we provide a vast optimization space that fully unlocks the model's capacity to learn new tasks. Simultaneously, this relaxation renders the orthogonality constraint practically achievable, thereby guaranteeing stability.

\section{Methods}
\subsection{Post-Hoc JANUS Rectification}
\label{sec:methods_projection}

To rectify the parameter updates, we replay a small subset of $m$ data points from the previous tasks and compute the individual gradients to construct the Jacobian matrix for each linear layer in the network. Suppose the network consists of $L$ linear layers.
For the simplicity of notation, we assume that all linear layers share the same shape $(d_{\text{out}}, d_\text{in})$.
The weight matrix of the $j$-th layer is denoted as $\bm{W}^{(j)} \in \mathbb{R}^{d_\text{out} \times d_\text{in}}$, and we define $\bm{\theta}^{(j)} \triangleq \text{rvec}({\bm{W}^{(j)}}) \in \mathbb{R}^{d_\text{in} d_\text{out}}$ as its row-major vectorization.
Let the gradient of the loss on the $i$-th data point with respect to $\bm{\theta}^{(j)}$ be denoted as $\bm{g}^{(i,j)} \triangleq \nabla_{\bm{\theta}^{(j)}} \mathcal{L}_{A}^{(i)}$.  The Jacobian matrix for the $j$-th layer can then be constructed by stacking these individual gradient vectors:
\begin{equation}
\label{eq: J^j}
    \bm{J}^{(j)} = \left[ {\bm{g}^{(1,j)}}, {\bm{g}^{(2,j)}}, \dots, {\bm{g}^{(m,j)}} \right]^\top \in \mathbb{R}^{m \times d_\text{in} d_\text{out}}.
\end{equation}

Upon the full completion of the FT process, we extract the overall parameter update $\Delta \bm{\theta}^{(j)} \triangleq \hat{\bm{\theta}}^{(j)} - \bm{\theta}_0^{(j)}$ between the initial pretrained weights $\bm{\theta}_0^{(j)}$ and the final adapted weights $\hat{\bm{\theta}}^{(j)}$.
Subsequently, we perform the post-hoc \textit{JANUS projection} on this update:
\begin{equation}
\label{eq:proj}
    \Delta \bm{\theta}^{(j)}_\text{proj} = \Delta \bm{\theta}^{(j)} - {\bm{J}^{(j)}}^\top  \left( \bm{J}^{(j)} {\bm{J}^{(j)}}^\top \right)^{-1} \bm{J}^{(j)} \Delta \bm{\theta}^{(j)}.
\end{equation}
% \begin{equation}
% \label{eq:proj}
%     \Delta \bm{\theta}^{(j)}_\text{proj} = \Delta \bm{\theta}^{(j)} - {\bm{J}^{(j)}}^\top \underbrace{ \left( \bm{J}^{(j)} {\bm{J}^{(j)}}^\top \right)^{-1} \bm{J}^{(j)} \Delta \bm{\theta}^{(j)} }_{\bm v^{(j)}}.
% \end{equation}

% We term the latter component as the \textit{parameter rectification step}, denoted by $P^{(j)} \triangleq {\bm{J}^{(j)}}^\top \left( \bm{J}^{(j)} {\bm{J}^{(j)}}^\top \right)^{-1} \bm{J}^{(j)} \Delta \bm{\theta}^{(j)}$.
After this projection-based rectification, the updated parameter strictly satisfies:
\begin{equation}
\label{eq:layer_wise_ortho}
    {\bm{g}^{(i,j)}}^\top \Delta \bm{\theta}^{(j)}_\text{proj} = 0, \quad \forall i \in [1, m] \cap \mathbb{Z}, \ \forall j \in [1, L] \cap \mathbb{Z}.
\end{equation}
Note that Eq.~\eqref{eq:layer_wise_ortho} is a reformulation of the parameter space orthogonality condition Eq.~\eqref{eq:vec_cond}, which serves as the necessary and sufficient condition for achieving performance invariance at the layer level. 
Building upon this layer-wise property, we can further formally establish the global forgetting-free guarantee across the entire network.

\begin{restatable}[Global Forgetting-Free Guarantee via Layer-wise Projection]{thm}{GlobalForgettingFree}
\label{thm:global_forgetting_free}
Applying the post-hoc JANUS projection independently to each linear layer guarantees that the overall loss change on the replayed past task samples approximates zero to the first-order, i.e., $\Delta\mathcal{L}_{A}^{(i)} \approx 0,\forall i=1,2,...,m$.
\end{restatable}
% \begin{theorem}[Global Forgetting-Free Guarantee via Layer-wise Projection]
% \label{thm:global_forgetting_free}
% Applying the post-hoc JANUS projection independently to each linear layer guarantees that the overall loss change on the replayed past task samples approximates zero to the first-order, i.e., $\Delta\mathcal{L}_{A}^{(i)} \approx 0,\forall i=1,2,...,m$.
% \end{theorem}
% \begin{proof}    
% The proof is deferred to Appendix \ref{sec:proof of thm global_forgetting_free}.
% \end{proof}
The proof is deferred to Appendix \ref{sec:proof of thm global_forgetting_free}.
Thm.~\ref{thm:global_forgetting_free} mathematically solidifies the effectiveness of our approach. Crucially, both the calculation of the Jacobian $\bm{J}^{(j)}$ and the subsequent projection procedure can be executed in a purely post-hoc manner, thereby adding no overload to the FT stage.
Meanwhile, this rectification procedure is tuning-agnostic: it merely requires the parameters before and after FT, imposing no assumptions on the underlying FT method.

\subsection{Multi-step Adaptive Rectification}
\label{sec: multistep rectification}
While JANUS identifies the directions along which the model can move with minimal degradation in performance on past tasks, it provides only local information.
Even when parameter updates are strictly constrained within the JANUS at the initial point $\bm{\theta}_0$, a substantially large update distance may cause the parameters to deviate from the valid approximation region of the local Jacobian.
To address this limitation, we propose a Multi-step Adaptive Rectification mechanism, which effectively navigates the parameter space along the loss contour, as illustrated in Fig.~\ref{fig:multi step rectification}.

\begin{figure}
  \centering
  \begin{subfigure}{0.55\textwidth}
      \includegraphics[width=\textwidth,trim=4.8cm 12.45cm 3.5cm 9.8cm, clip]{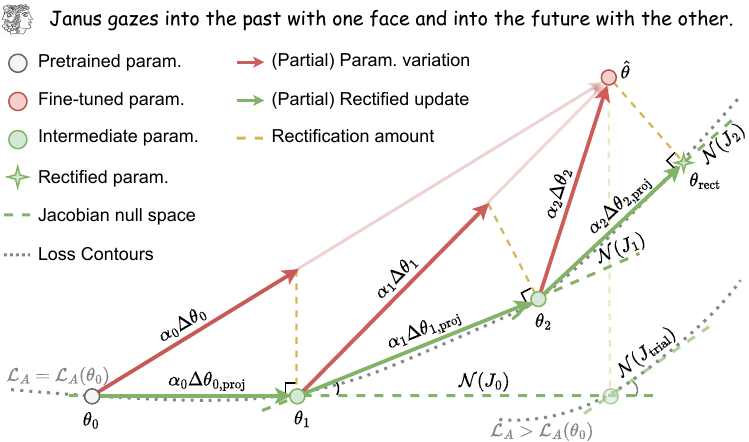}
      \caption{Schematic diagram}
      \label{fig:multi step rectification}
  \end{subfigure}
  \hfill
  \begin{subfigure}{0.42\textwidth}
        \includegraphics[width=\textwidth,trim=2.5cm 1.25cm 1.5cm 2cm, clip]{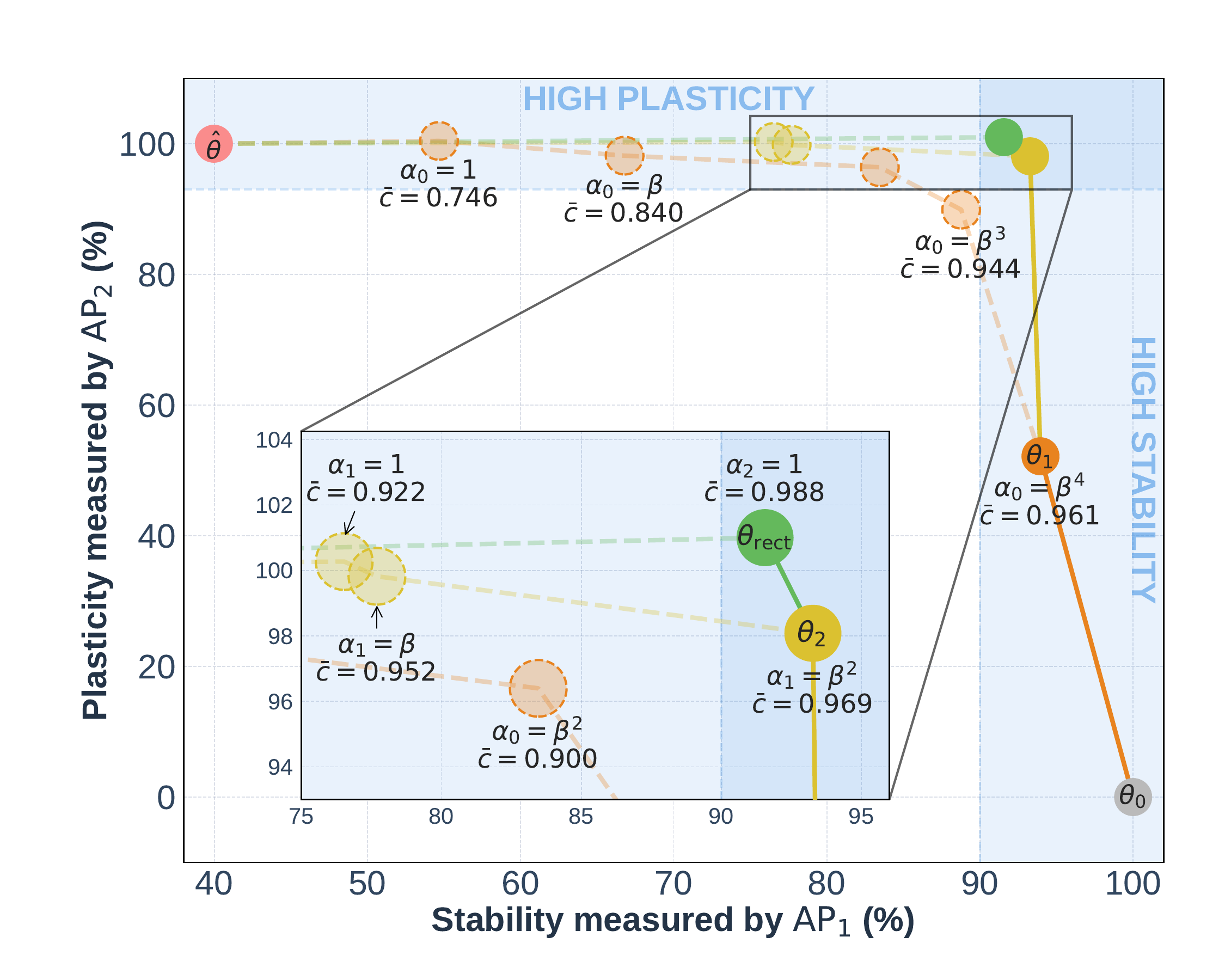}
        \caption{Visualization of the rectification process}
\label{fig:multi step rectification traj}
  \end{subfigure}
  \caption{Illustration of the Multi-step Adaptive Rectification in JANUS.}
\end{figure}

Suppose that after $t$ rectification steps, the current parameters are denoted as $\bm{\theta}_t^{(j)}$.
The remaining parameter variation to be rectified is defined as $\Delta\bm{\theta}_t^{(j)} \triangleq \hat{\bm{\theta}}^{(j)} - \bm{\theta}_t^{(j)}$.
For this variation, we compute the Jacobian $\bm{J}_t^{(j)}$ at $\bm{\theta}_t^{(j)}$ and then the projected update $\Delta\bm{\theta}^{(j)}_{t,\text{proj}}$ according to Eq.~\eqref{eq:proj}.
Subsequently, we initialize the step size $\alpha_t = 1$ to perform a trial rectification, yielding the trial parameters $\bm{\theta}_{\text{trial}}^{(j)} = \bm{\theta}_t^{(j)} + \alpha_t \Delta\bm{\theta}^{(j)}_{t,\text{proj}}$. 
To determine whether the current step size is overly large, we evaluate the Jacobian matrix $\bm{J}_{\text{trial}}^{(j)}$ at the trial point for every layer and then compute the average cosine of the principal angles \cite{jordan1875essai} between the row spaces of $\bm{J}_{\text{trial}}^{(j)}$ and the current Jacobian $\bm{J}_t^{(j)}$, denoted as $c^{(j)}$.
This procedure is termed \textit{JANUS orientation comparison}, and we adopt $\bar{c}\triangleq \text{mean}_j \left\{c^{(j)}\right\}$ to quantify the JANUS shift.
If $\bar{c}$ exceeds a predefined threshold $\tau$, i.e., the JANUS shift is acceptable, we accept the current step size and set $\bm{\theta}_{t+1} \gets \bm{\theta}_{\text{trial}}$.
Otherwise, we shrink the step size by a decay factor $\beta \in (0, 1)$ and repeat the trial rectification until the criterion is satisfied. 
This adaptive step size ensures that the intermediate parameters closely track the loss-invariant contour.
When a rectification step is accepted with a full step size $\alpha_t = 1$, all remaining variation has been rectified, and the rectification phase is complete.
If $\alpha_t < 1$, the mechanism proceeds to the next iterative step.

\begin{algorithm}[t]
\caption{Multi-step Adaptive Rectification in JANUS}
\label{alg:multi_step_rectification}
\begin{algorithmic}[1]
\Require Pretrained parameters $\bm{\theta}_0$, fine-tuned parameters $\hat{\bm{\theta}}$, decay factor $\beta$, acceptance threshold $\tau$.
\Ensure Final rectified parameters $\bm{\theta}_{\text{rect}}$.

\State Initialize step $t \gets 0$, current parameters $\bm{\theta}_t \gets \bm{\theta}_0$.
\While{True}
    \ForAll{$j$} \Comment{Ghost projection (see Sec.~\ref{sec:methods_projection} and Sec.~\ref{sec: ghost})}
        \State Compute $\bm{J}^{(j)}_t$ at $\bm{\theta}_t^{(j)}$ and the projected update $\Delta\bm{\theta}^{(j)}_{t,\text{proj}}$ via Eq.~\eqref{eq:proj}.
        % \State Compute the remaining residual: $\Delta\bm{\theta}_t^{(j)} \gets \hat{\bm{\theta}}^{(j)} - \bm{\theta}_t^{(j)}$.
        % \State Compute the projected update $\Delta\bm{\theta}^{(j)}_{t,\text{proj}}$ via Eq.~\eqref{eq:proj}.
    \EndFor
    
    \State Initialize trial step size: $\alpha_t \gets 1$.
    \While{True}
        \State Let $\bm{\theta}_{\text{trial}} \gets \bm{\theta}_t + \alpha_t \Delta\bm{\theta}_{t,\text{proj}}$
        \ForAll{$j$} \Comment{Ghost orientation comparison (see Sec.~\ref{sec: multistep rectification} and  Sec.~\ref{sec: ghost})}
            \State Compute trial Jacobian $\bm{J}_{\text{trial}}^{(j)}$ and avg. cosine $c^{(j)}$.
            % \State $\bm{\theta}_{\text{trial}}^{(j)} \gets \bm{\theta}_t^{(j)} + \alpha_t \Delta\bm{\theta}^{(j)}_{t,\text{proj}}$.
            % \State Compute trial Jacobian $\bm{J}_{\text{trial}}^{(j)}$ and avg. cosine $c^{(j)}$.
        \EndFor
        
        \If{$\bar{c}\triangleq\text{mean}_j \{c^{(j)}\} \ge \tau$}
            \State \textbf{break} \Comment{Step size is within confidence region}
        \EndIf
        \State $\alpha_t \gets \beta \cdot \alpha_t$ \Comment{Decay step size and retry}
    \EndWhile
    
    \State Update parameters: $\bm{\theta}_{t+1} \gets \bm{\theta}_{\text{trial}}$.
    \If{$\alpha_t = 1$}
        \State \textbf{break} \Comment{Successfully rectified}
    \EndIf
    \State $t \gets t + 1$
\EndWhile
\State \textbf{return} $\bm{\theta}_{\text{rect}} \equiv \bm{\theta}_{t+1}$
\end{algorithmic}
\end{algorithm}

\subsection{Efficient Calculation and Storage of the Jacobian}
\label{sec: ghost}
% Until now, we have formulated the multi-step rectification mechanism under the assumption that the Jacobian matrices are readily accessible.
Directly computing and storing the full Jacobian $\bm J^{(j)} \in \mathbb{R}^{m \times d_\text{in} d_\text{out}}$ for all layers is computationally and spatially prohibitive.
Temporally, it would require $m$ separate backward passes to get the gradients for each sample, and Algorithm~\ref{alg:multi_step_rectification} demands repeated evaluations of the Jacobian at various parameter locations.
Spatially, the complete Jacobian entails an $\mathcal{O}(Lmd^2)$ memory complexity in total (assuming $d_\text{in}\approx d_\text{out}\approx d$ for simplicity), posing a severe bottleneck even for offline storage.

Addressing the temporal challenges, we propose ghost projection and ghost orientation comparison techniques inspired by \cite{lee2021scaling,wang2024data}.
By utilizing intermediate tensors (i.e., the layer inputs $\bm X^{(j)}\in \mathbb{R}^{m\times T\times d_{\text{in}}}$ and the pre-activation gradients $\bm A^{(j)}\in \mathbb{R}^{m\times T\times d_{\text{out}}}$) from batch-wise backpropagation, we bypass per-sample backpropagation and implicitly perform JANUS projection and JANUS orientation comparison without ever instantiating the high-dimensional Jacobian matrices.
This approach enables the computation of all necessary quantities in a single backward pass, significantly reducing the computational overhead.

Regarding the spatial bottleneck, caching all $\bm X^{(j)}$ and $\bm A^{(j)}$ still requires $\mathcal{O}(2LmTd)$ space, which can still be spatially prohibitive under a large $T$.
Fortunately, it has been demonstrated that they tend to exhibit strong correlation along the sequence dimension \cite{wang2021spatten}.
Leveraging this intrinsic low-rank property, we can significantly compress these tensors along the sequence dimension using rank-$r$ SVD.
This yields two low-dimensional tensors $\hat{\bm{X}}^{(j)} \in \mathbb{R}^{m\times r\times d_{\text{in}}}$ and $\hat{\bm{A}}^{(j)} \in \mathbb{R}^{m\times r\times d_{\text{out}}}$, reducing the spatial complexity to $\mathcal{O}(2Lmrd)$ ($r\ll \min\{d,T\}$) while preserving the principal gradient information.
The details can be found in Appendix~\ref{sec:details of ghost}.

In the experiments, with a batch size of $m=256$ and a compressed rank of $r=32$, these techniques enables the rectification of a LLaMA-2-7b model on a single A100 GPU within 30 minutes (the exact time depending on the number of rectification iterations).
For a detailed time cost profile, please refer to Appendix~\ref{sec: time cost profiling}.
Furthermore, caching a complete set of $\hat{\bm{X}}^{(j)}$ and $\hat{\bm{A}}^{(j)}$ across all layers in \texttt{float16} format requires only less than 40 GB of disk storage.

\section{Experiments}
\subsection{Experimental Setup}
\label{sec: experimental setup}
\paragraph{Tasks}
To evaluate the effectiveness of our method, we fine-tune LLaMA-2-7b \cite{touvron2023llama} and LLaMA-3-8b \cite{llama3modelcard} across three tasks: Math, Code, and Instruction Following (IF). 
Following the protocol established by \cite{yang2024corda}, the world knowledge is evaluated by TriviaQA \cite{joshi2017triviaqa}, NQ open \cite{lee2019latent}, and WebQS \cite{berant2013semantic} (collectively termed knowledge datasets). 
The Math, Code, and IF tasks are evaluated by GSM8k/Math \cite{cobbe2021training,yu2023metamath}, HumanEval/MBPP \cite{chen2021evaluating,austin2021program}, and MTBench \cite{zheng2023judging}, respectively.
% The Math task is trained on MetaMathQA \cite{yu2023metamath} and tested on GSM8k \cite{cobbe2021training} and Math \cite{yu2023metamath}. 
% The Code task is trained on CodeFeedback \cite{zheng2024opencodeinterpreter} and tested on HumanEval \cite{chen2021evaluating} and MBPP \cite{austin2021program}. 
% The IF task is trained on WizardLM-Evol-Instruct \cite{xuwizardlm} and tested on MTBench \cite{zheng2023judging}.

\paragraph{Compared Methods}
We benchmark our approach against several baselines: (1) Full Fine-tuning (FF); (2) LoRA \cite{hu2022lora}; (3) PiSSA \cite{meng2024pissa}; (4) CorDA \cite{yang2024corda}; (5) MiLoRA \cite{wang2025milora}; and (6) LoRA-Null \cite{tang2026put}. Our post-hoc rectification mechanism is applied to the resulting parameters of FF, LoRA, and PiSSA, which we denote as \textbf{JANUS-F}, \textbf{JANUS-L}, and \textbf{JANUS-P}, respectively.

\paragraph{Metrics}
In addition to the raw scores, we define three aggregated metrics to comprehensively evaluate the stability-plasticity tradeoff: $\text{AP}_1$ for stability, $\text{AP}_2$ for plasticity, and $\text{AP}$ for overall performance. Specifically, $\text{AP}_1$ is the average of the performance scores across knowledge datasets, each normalized relative to the pretrained model.
$\text{AP}_2$ is the average of the performance scores across task-specific datasets, each normalized relative to the FF model.
The overall metric $\text{AP}$ is the arithmetic mean of $\text{AP}_1$ and $\text{AP}_2$. 
% Conversely, $\text{AP}_2$ represents the average relative performance on the task-specific datasets (e.g., GSM8k/MATH, HumanEval/MBPP, or MTBench), normalized against the performance of FF. The overall metric $\text{AP}$ is the arithmetic mean of $\text{AP}_1$ and $\text{AP}_2$. 

Further details are provided in the Appendix~\ref{sec: implementation details}.

\subsection{Main results}
\label{sec: main results}

\begin{figure}[ht]
    \centering
    \begin{subfigure}{0.325\textwidth} 
        \includegraphics[width=\textwidth,trim=0.6cm 0.58cm 0.6cm 1.7cm, clip]{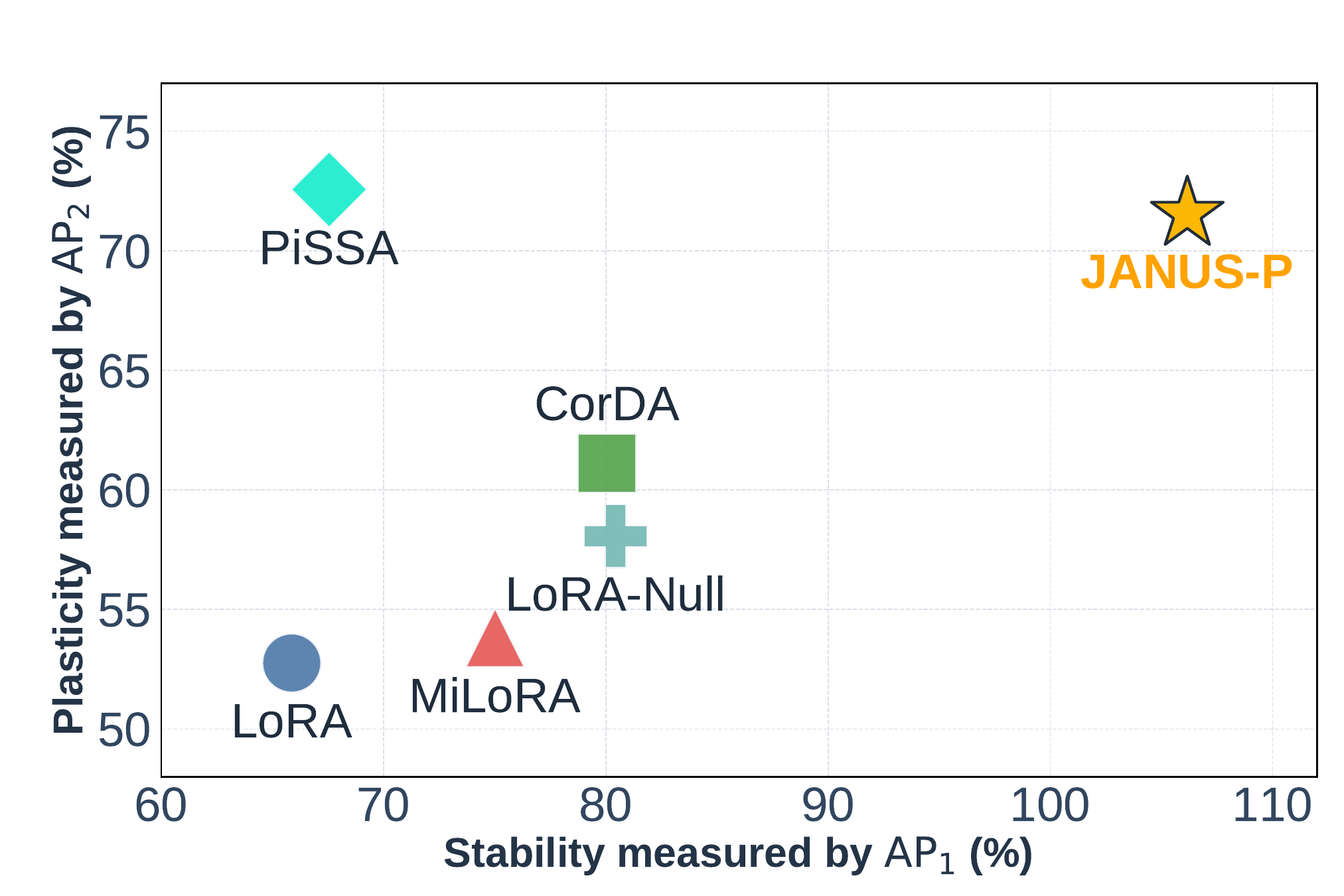}
        \caption{Math (LLaMA-2-7b)}
        \label{fig:llama-2-7b-Math}
    \end{subfigure}
    \hfill
    \begin{subfigure}{0.325\textwidth}
        \includegraphics[width=\textwidth,trim=0.6cm 0.58cm 0.6cm 1.7cm, clip]{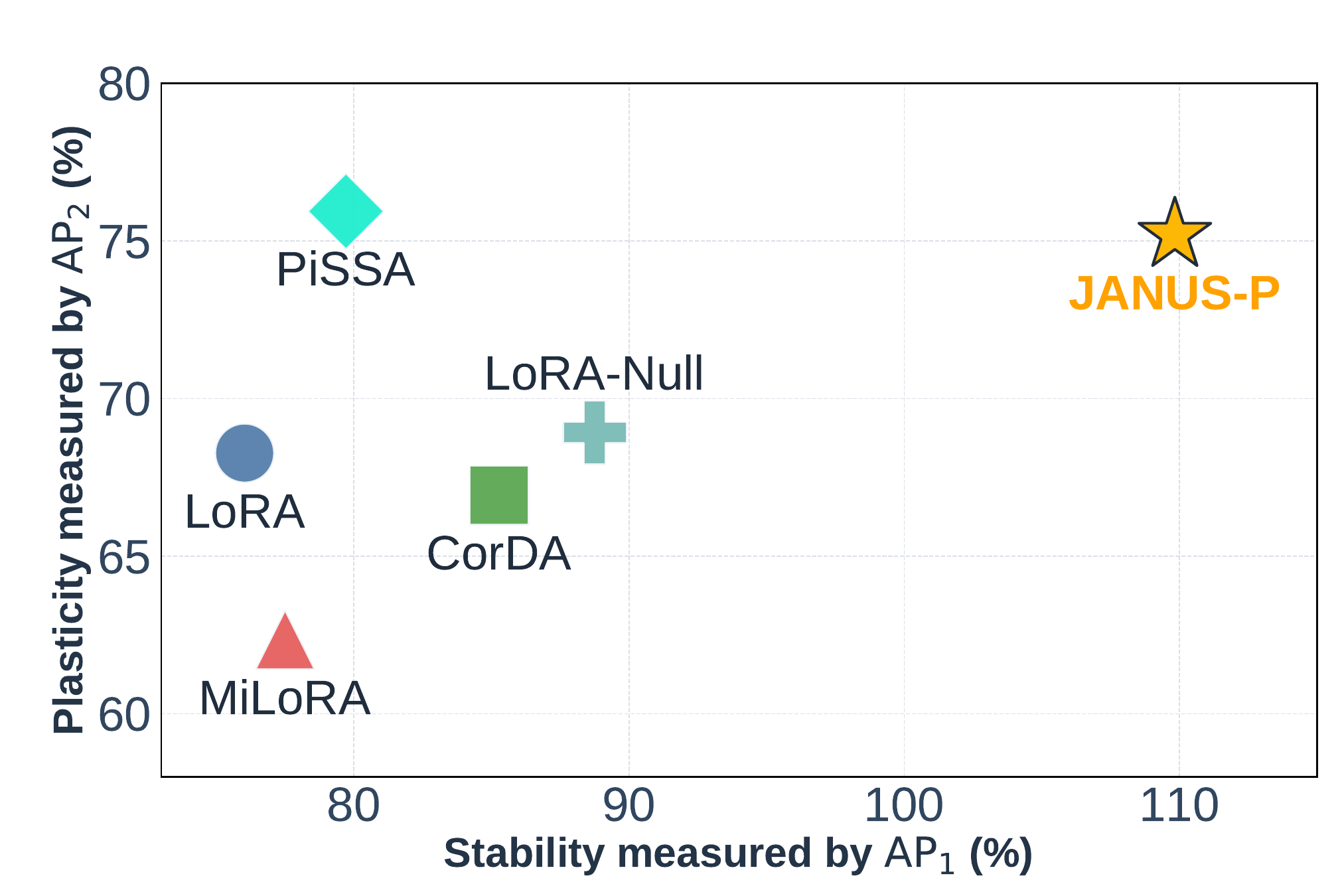}
        \caption{Code (LLaMA-2-7b)}
        \label{fig:llama-2-7b-Code}
    \end{subfigure}
    \hfill
    \begin{subfigure}{0.325\textwidth}
        \includegraphics[width=\textwidth,trim=0.6cm 0.58cm 0.6cm 1.7cm, clip]{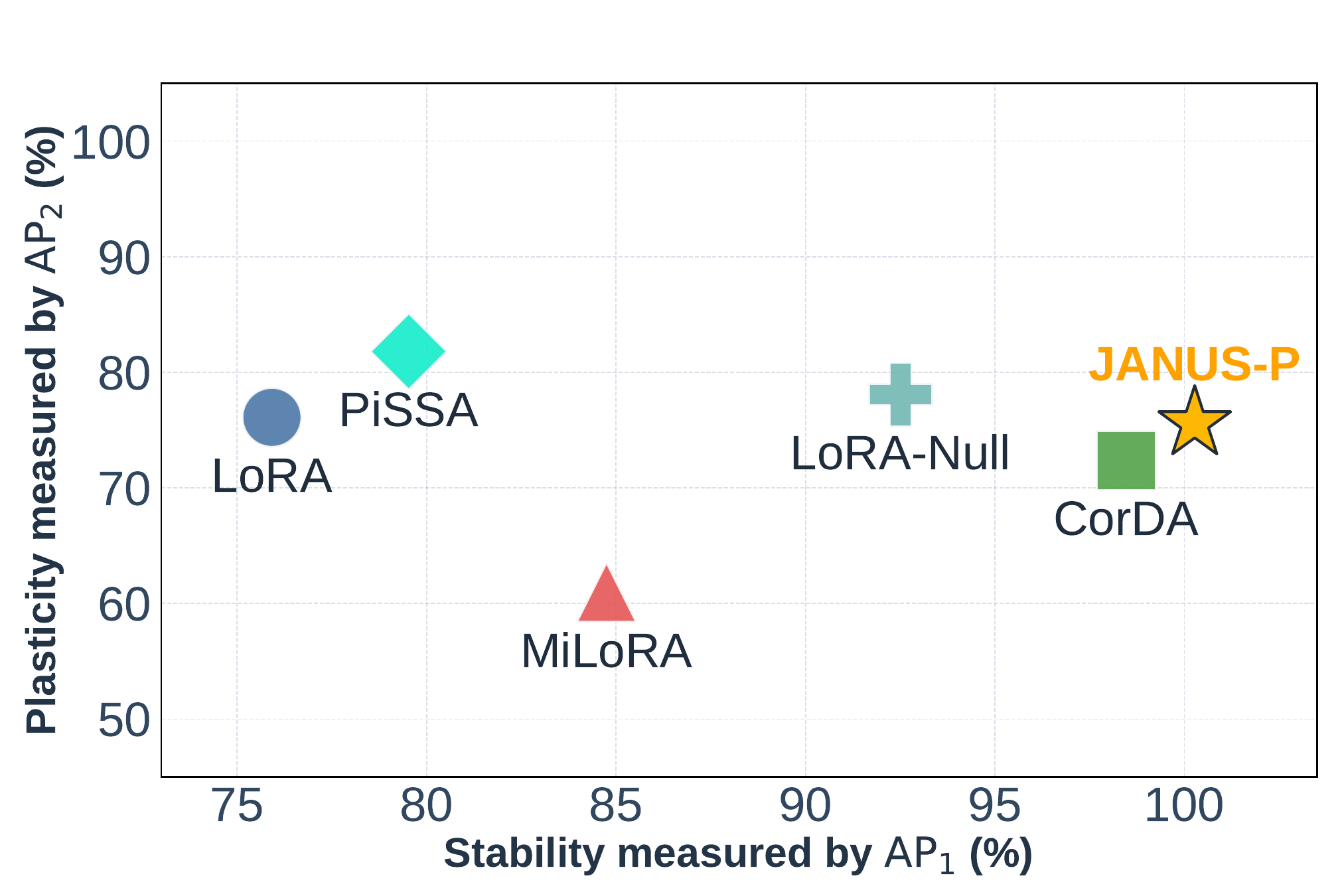}
        \caption{IF (LLaMA-2-7b)}
        \label{fig:llama-2-7b-Instruction Following}
    \end{subfigure}
    \\
    \begin{subfigure}{0.325\textwidth} 
        \includegraphics[width=\textwidth,trim=0.6cm 0.58cm 0.6cm 1.7cm, clip]{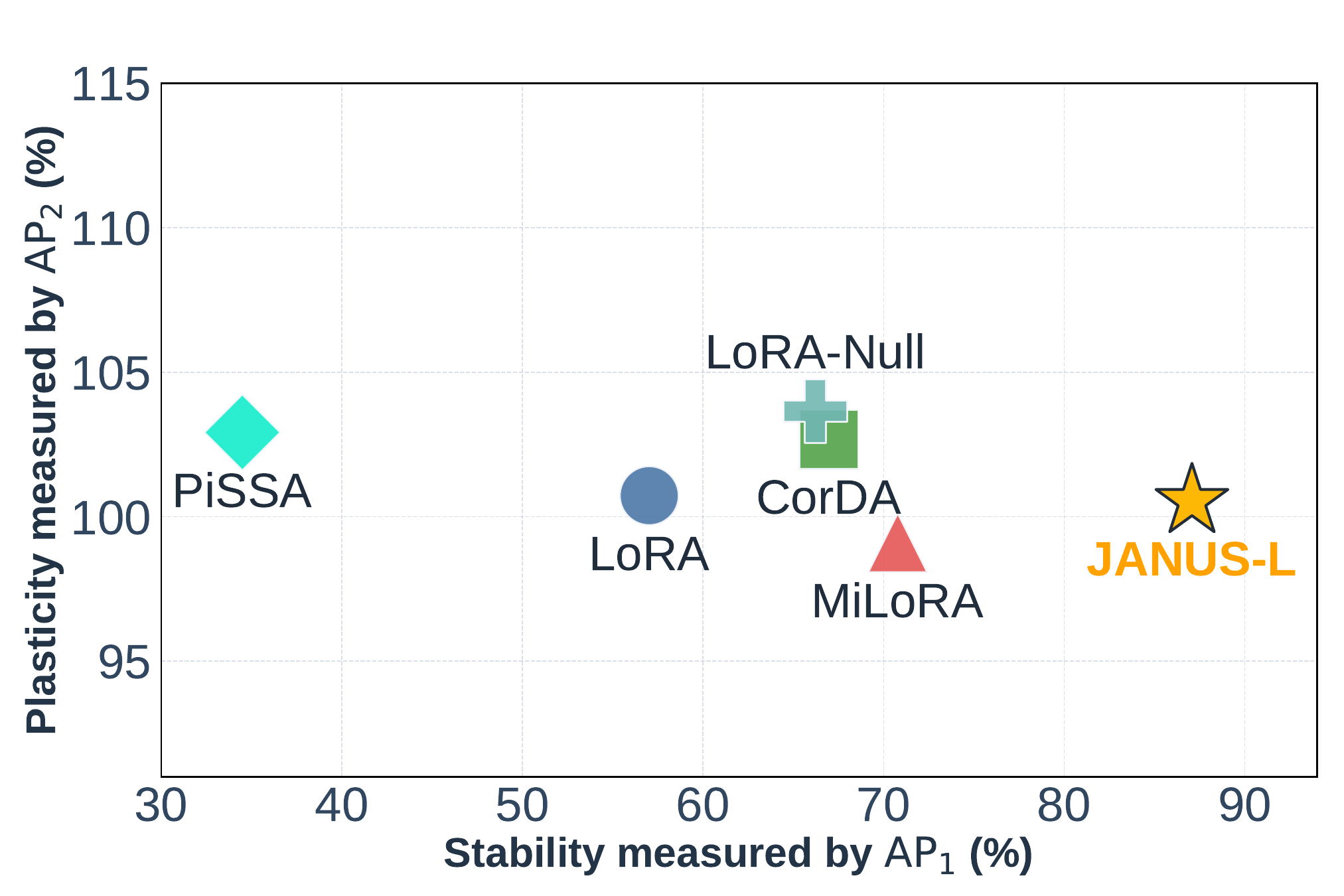}
        \caption{Math (LLaMA-3-8b)}
        \label{fig:llama-3-8b-Math}
    \end{subfigure}
    \hfill
    \begin{subfigure}{0.325\textwidth}
        \includegraphics[width=\textwidth,trim=0.6cm 0.58cm 0.6cm 1.7cm, clip]{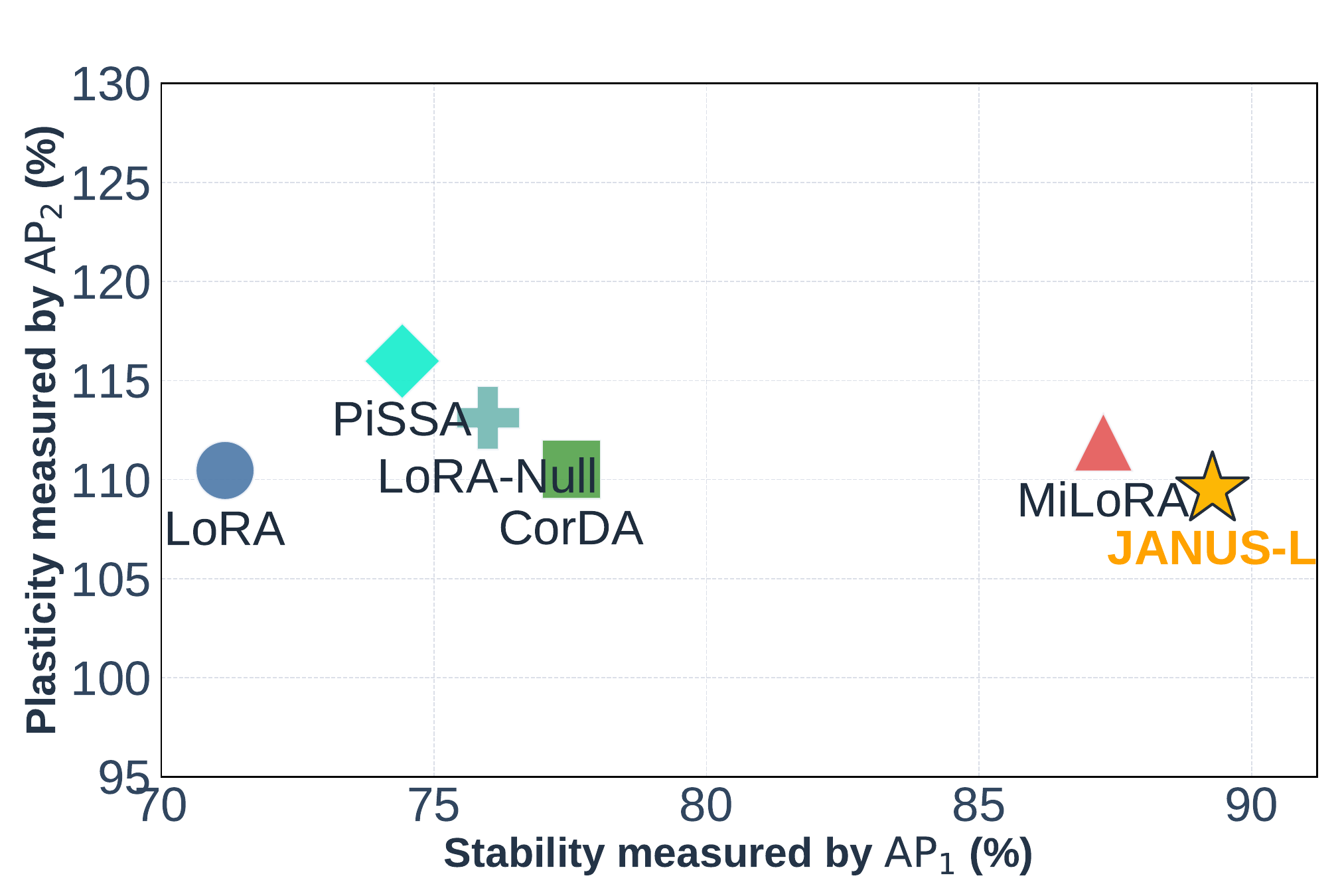}
        \caption{Code (LLaMA-3-8b)}
        \label{fig:llama-3-8b-Code}
    \end{subfigure}
    \hfill
    \begin{subfigure}{0.325\textwidth}
        \includegraphics[width=\textwidth,trim=0.6cm 0.58cm 0.6cm 1.7cm, clip]{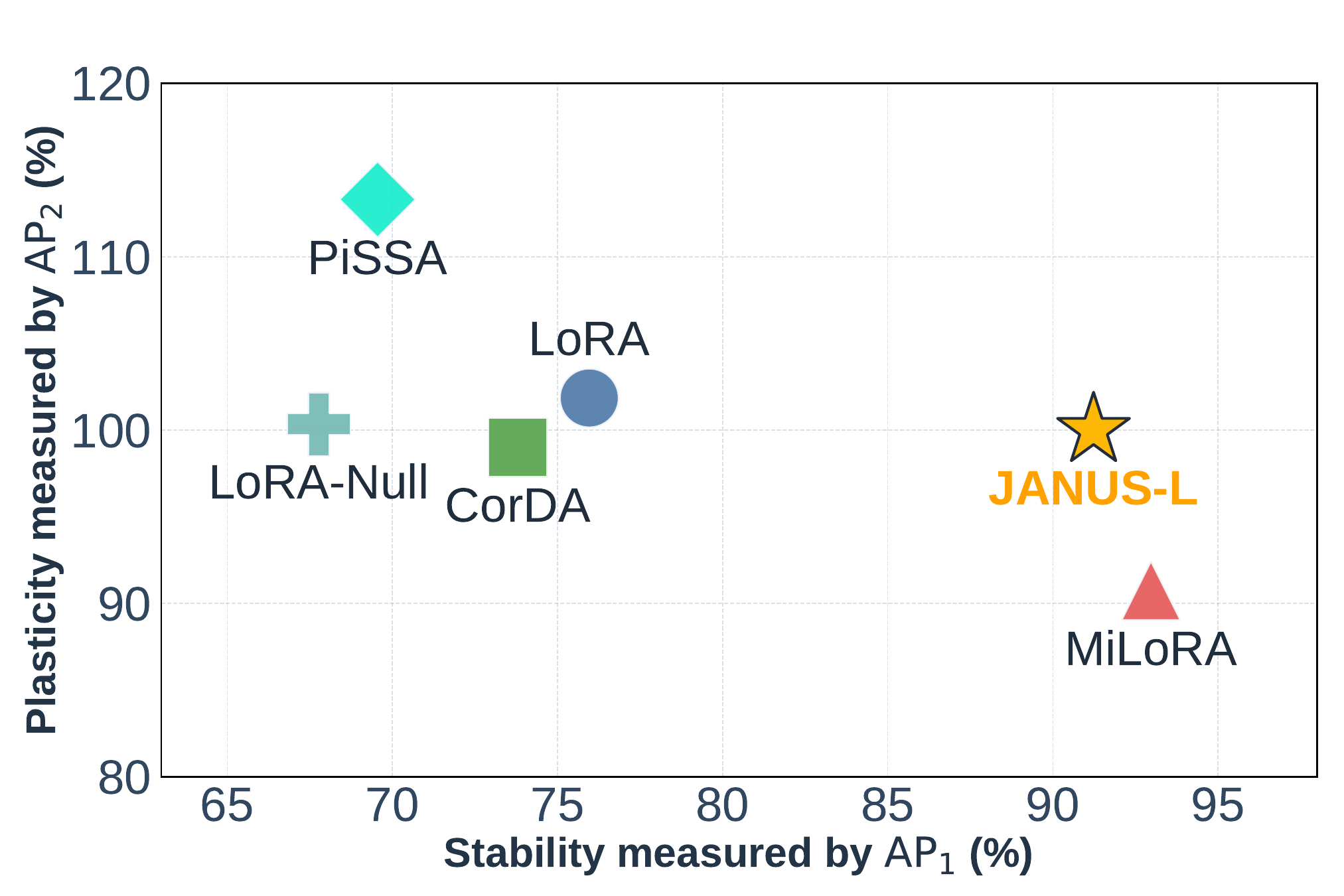}
        \caption{IF (LLaMA-3-8b)}
        \label{fig:llama-3-8b-Instruction Following}
    \end{subfigure}
    \caption{Stability-Plasticity comparison. The methods used for normalization are excluded.}
    \label{fig:llama-2-7b and llama-3-8b}
\end{figure}

\begin{table}[htbp]
    \centering
    \caption{Performance comparison of various methods on LLaMA-2-7b. For each column, the \textbf{bold} and \underline{underlined} values indicate the highest and second-highest scores, respectively. The scores used for normalization are marked in gray and excluded from the ranking.}
    \label{tab:llama-2-7b}
    \small 
    \resizebox{\textwidth}{!}{
    \begin{tabular}{l|c|r@{}l@{\ }r@{}l@{\ }r@{}l@{\ }r@{}l@{\ }r@{}l@{\ }r@{}l@{\ }r@{}l@{\ }r@{}l}
    \toprule
    Method & \#Param
    & \multicolumn{2}{c}{TriviaQA}
    & \multicolumn{2}{c}{NQ open}
    & \multicolumn{2}{c}{WebQS}
    & \multicolumn{2}{c}{$\text{AP}_1(\%)$}
    & \multicolumn{2}{c}{Bench 1}
    & \multicolumn{2}{c}{Bench 2}
    & \multicolumn{2}{c}{$\text{AP}_2(\%)$}
    & \multicolumn{2}{c}{$\text{AP}(\%)$} \\
    \midrule
    
    LLaMA-2-7b & --
    & \textcolor{gray}{52.52} & 
    & \textcolor{gray}{18.95} &
    & \textcolor{gray}{5.81} &
    & \textcolor{gray}{100.00} &
    & \multicolumn{2}{c}{--} 
    & \multicolumn{2}{c}{--} 
    & \multicolumn{2}{c}{--} 
    & \multicolumn{2}{c}{--}  \\
    
    \midrule
    \multicolumn{18}{c}{\textbf{Task: Math} (Bench 1: GSM8k, Bench 2: Math)} \\
    \midrule
    
    FF & 6.7B
    & 19.19 &
    & 0.86 &
    & 4.58 &
    & 39.97 &
    & \textcolor{gray}{60.20} &
    & \textcolor{gray}{12.56} &
    & \textcolor{gray}{100.00} &
    & 69.98 & \\
    
    LoRA & 320M
    & 41.78 &
    & 1.50 &
    & 6.40 &
    & 65.87 &
    & 40.71 &
    & 4.76 &
    & 52.76 &
    & 59.32 & \\
    
    PiSSA & 320M
    & 41.57 & \textcolor{gray}{\scriptsize $\pm$1.50}
    & 3.34 & \textcolor{gray}{\scriptsize $\pm$0.18}
    & 6.15 & \textcolor{gray}{\scriptsize $\pm$0.52}
    & 67.55 & \textcolor{gray}{\scriptsize $\pm$3.85}
    & \textbf{51.73} & \textcolor{gray}{\scriptsize $\pm$0.29}
    & \textbf{7.43} & \textcolor{gray}{\scriptsize $\pm$0.35}
    & \textbf{72.56} & \textcolor{gray}{\scriptsize $\pm$1.37}
    & \underline{70.05} & \textcolor{gray}{\scriptsize $\pm$2.34} \\
    
    CorDA & 320M
    & 41.94 &
    & \underline{7.09} &
    & 7.14 &
    & 80.05 &
    & \underline{43.97} &
    & 6.18 &
    & 61.12 &
    & 70.59 & \\
    
    MiLoRA & 320M
    & \underline{44.98} &
    & 3.13 &
    & 7.14 &
    & 75.02 &
    & 40.64 &
    & 5.04 &
    & 53.82 &
    & 64.42 & \\
    
    LoRA-Null & 320M
    & 44.64 &
    & 6.04 &
    & \underline{7.23} &
    & \underline{80.44} &
    & 42.30 &
    & 5.76 &
    & 58.06 &
    & 69.25 & \\
    
    \textbf{JANUS-P} & 320M
    & \textbf{47.31} & \textcolor{gray}{\scriptsize $\pm$0.57}
    & \textbf{16.05} & \textcolor{gray}{\scriptsize $\pm$0.65}
    & \textbf{8.35} & \textcolor{gray}{\scriptsize $\pm$0.61}
    & \textbf{106.16} & \textcolor{gray}{\scriptsize $\pm$4.98}
    & \underline{51.20} & \textcolor{gray}{\scriptsize $\pm$0.57}
    & \underline{7.29} & \textcolor{gray}{\scriptsize $\pm$0.29}
    & \underline{71.54} & \textcolor{gray}{\scriptsize $\pm$1.49}
    & \textbf{88.85} & \textcolor{gray}{\scriptsize $\pm$3.01} \\
    
    \midrule
    \multicolumn{18}{c}{\textbf{Task: Code} (Bench 1: HumanEval, Bench 2: MBPP)} \\
    \midrule
    
    FF & 6.7B
    & 37.61 &
    & 6.76 &
    & \underline{6.59} &
    & 73.57 &
    & \textcolor{gray}{33.88} &
    & \textcolor{gray}{28.41} &
    & \textcolor{gray}{100.00} &
    & \underline{86.78} & \\
    
    LoRA & 320M
    & 43.52 &
    & 9.22 &
    & 5.61 &
    & 76.03 &
    & 18.22 &
    & 23.51 &
    & 68.27 &
    & 72.15 & \\
    
    PiSSA & 320M
    & 44.93 & \textcolor{gray}{\scriptsize $\pm$0.25}
    & 10.43 & \textcolor{gray}{\scriptsize $\pm$0.65}
    & 5.73 & \textcolor{gray}{\scriptsize $\pm$0.18}
    & 79.71 & \textcolor{gray}{\scriptsize $\pm$2.20}
    & \textbf{21.51} & \textcolor{gray}{\scriptsize $\pm$0.51}
    & \textbf{25.11} & \textcolor{gray}{\scriptsize $\pm$0.30}
    & \textbf{75.94} & \textcolor{gray}{\scriptsize $\pm$1.24}
    & 77.83 & \textcolor{gray}{\scriptsize $\pm$1.61} \\
    
    CorDA & 320M
    & 44.23 &
    & 13.57 &
    & 5.81 &
    & 85.28 &
    & 19.25 &
    & 21.89 &
    & 66.93 &
    & 76.10 & \\
    
    MiLoRA & 320M
    & 42.58 &
    & 11.83 &
    & 5.17 &
    & 77.50 &
    & 16.74 &
    & 21.39 &
    & 62.35 &
    & 69.92 & \\
    
    LoRA-Null & 320M
    & \underline{46.03} &
    & \underline{14.57} &
    & 5.91 &
    & \underline{88.75} &
    & 18.59 &
    & 23.57 &
    & 68.92 &
    & 78.83 & \\
    
    \textbf{JANUS-P} & 320M
    & \textbf{50.21} & \textcolor{gray}{\scriptsize $\pm$0.45}
    & \textbf{16.39} & \textcolor{gray}{\scriptsize $\pm$0.43}
    & \textbf{8.56} & \textcolor{gray}{\scriptsize $\pm$0.31}
    & \textbf{109.83} & \textcolor{gray}{\scriptsize $\pm$2.51}
    & \underline{21.01} & \textcolor{gray}{\scriptsize $\pm$0.88}
    & \underline{25.10} & \textcolor{gray}{\scriptsize $\pm$0.48}
    & \underline{75.19} & \textcolor{gray}{\scriptsize $\pm$0.80}
    & \textbf{92.51} & \textcolor{gray}{\scriptsize $\pm$0.98} \\
    
    \midrule
    \multicolumn{18}{c}{\textbf{Task: IF} (Bench 1: MTBench, Bench 2: --)} \\
    \midrule
    
    FF & 6.7B
    & 20.38 &
    & 5.01 &
    & 4.97 &
    & 50.26 &
    & \textcolor{gray}{4.56} &
    & \multicolumn{2}{c}{--}
    & \textcolor{gray}{100.00} &
    & 75.13 & \\
    
    LoRA & 320M
    & 43.79 &
    & 8.25 &
    & 5.86 &
    & 75.92 &
    & 3.47 &
    & \multicolumn{2}{c}{--}
    & 76.10 &
    & 76.01 & \\
    
    PiSSA & 320M
    & 43.23 & \textcolor{gray}{\scriptsize $\pm$0.70}
    & 8.75 & \textcolor{gray}{\scriptsize $\pm$0.45}
    & 6.40 & \textcolor{gray}{\scriptsize $\pm$0.32}
    & 79.53 & \textcolor{gray}{\scriptsize $\pm$2.88}
    & \textbf{3.73} & \textcolor{gray}{\scriptsize $\pm$0.54}
    & \multicolumn{2}{c}{--}
    & \textbf{81.80} & \textcolor{gray}{\scriptsize $\pm$11.94}
    & \underline{80.67} & \textcolor{gray}{\scriptsize $\pm$6.34} \\
    
    CorDA & 320M
    & 45.63 &
    & \textbf{17.04} &
    & \underline{6.89} &
    & \underline{98.46} &
    & 3.30 &
    & \multicolumn{2}{c}{--}
    & 72.37 &
    & 85.42 & \\
    
    MiLoRA & 320M
    & 45.02 &
    & 10.28 &
    & 6.64 &
    & 84.75 &
    & 2.78 &
    & \multicolumn{2}{c}{--}
    & 60.96 &
    & 72.86 & \\
    
    LoRA-Null & 320M
    & \underline{47.55} &
    & 12.96 &
    & \underline{6.89} &
    & 92.51 &
    & \underline{3.56} &
    & \multicolumn{2}{c}{--}
    & \underline{78.07} &
    & 85.29 & \\
    
    \textbf{JANUS-P} & 320M
    & \textbf{48.06} & \textcolor{gray}{\scriptsize $\pm$0.20}
    & \underline{15.58} & \textcolor{gray}{\scriptsize $\pm$0.82}
    & \textbf{7.38} & \textcolor{gray}{\scriptsize $\pm$0.25}
    & \textbf{100.27} & \textcolor{gray}{\scriptsize $\pm$1.18}
    & {3.45} & \textcolor{gray}{\scriptsize $\pm$0.25}
    & \multicolumn{2}{c}{--}
    & {75.58} & \textcolor{gray}{\scriptsize $\pm$5.59}
    & \textbf{87.93} & \textcolor{gray}{\scriptsize $\pm$2.26} \\
    
    \bottomrule
    \end{tabular}}
\end{table}

\begin{table}[htbp]
    \centering
    \caption{Performance comparison of various methods on LLaMA-3-8b. For each column, the \textbf{bold} and \underline{underlined} values indicate the highest and second-highest scores, respectively. The scores used for normalization are marked in gray and excluded from the ranking.}
    \label{tab:llama-3-8b}
    \small 
    \tabcolsep=3.2pt % 调整列间距以适应宽度
    \resizebox{\textwidth}{!}{
\begin{tabular}{l|c|r@{}l@{\ }r@{}l@{\ }r@{}l@{\ }r@{}l@{\ }r@{}l@{\ }r@{}l@{\ }r@{}l@{\ }r@{}l}
    \toprule
    Method & \#Param
    & \multicolumn{2}{c}{TriviaQA}
    & \multicolumn{2}{c}{NQ open}
    & \multicolumn{2}{c}{WebQS}
    & \multicolumn{2}{c}{$\text{AP}_1(\%)$}
    & \multicolumn{2}{c}{Bench 1}
    & \multicolumn{2}{c}{Bench 2}
    & \multicolumn{2}{c}{$\text{AP}_2(\%)$}
    & \multicolumn{2}{c}{$\text{AP}(\%)$} \\
    \midrule

    LLaMA-3-8b & --
    & \textcolor{gray}{61.66} &
    & \textcolor{gray}{21.86} &
    & \textcolor{gray}{9.94} &
    & \textcolor{gray}{100.00} &
    & \multicolumn{2}{c}{--}
    & \multicolumn{2}{c}{--}
    & \multicolumn{2}{c}{--}
    & \multicolumn{2}{c}{--} \\

    \midrule
    \multicolumn{18}{c}{\textbf{Task: Math} (Bench 1: GSM8k, Bench 2: Math)} \\
    \midrule

    FF & 8.37B
    & 41.43 &
    & 5.21 &
    & 4.13 &
    & 44.19 &
    & \textcolor{gray}{75.36} &
    & \textcolor{gray}{24.84} &
    & \textcolor{gray}{100.00} &
    & 72.10 & \\

    LoRA & 336M
    & \underline{54.66} & \textcolor{gray}{\scriptsize $\pm$0.58}
    & 7.90 & \textcolor{gray}{\scriptsize $\pm$0.78}
    & 4.60 & \textcolor{gray}{\scriptsize $\pm$0.10}
    & 57.02 & \textcolor{gray}{\scriptsize $\pm$1.33}
    & 75.49 & \textcolor{gray}{\scriptsize $\pm$0.57}
    & 25.16 & \textcolor{gray}{\scriptsize $\pm$0.33}
    & 100.73 & \textcolor{gray}{\scriptsize $\pm$0.41}
    & 78.87 & \textcolor{gray}{\scriptsize $\pm$0.65} \\

    PiSSA & 336M
    & 30.22 &
    & 3.57 &
    & 3.79 &
    & 34.49 &
    & \textbf{76.72} &
    & \underline{25.84} &
    & \underline{102.92} &
    & 68.70 & \\

    CorDA & 336M
    & 49.76 &
    & \underline{16.09} &
    & 4.63 &
    & 66.96 &
    & 76.35 &
    & \underline{25.84} &
    & 102.67 &
    & 84.82 & \\

    MiLoRA & 336M
    & 54.15 &
    & 15.32 &
    & 5.41 &
    & \underline{70.78} &
    & 75.28 &
    & 24.42 &
    & 99.10 &
    & \underline{84.94} & \\

    LoRA-Null & 336M
    & 53.93 &
    & 11.86 &
    & \underline{5.66} &
    & 66.22 &
    & \underline{76.50} &
    & \textbf{26.28} &
    & \textbf{103.65} &
    & \underline{84.94} & \\

    \textbf{JANUS-L} & 336M
    & \textbf{58.49} & \textcolor{gray}{\scriptsize $\pm$0.77}
    & \textbf{20.53} & \textcolor{gray}{\scriptsize $\pm$0.11}
    & \textbf{7.20} & \textcolor{gray}{\scriptsize $\pm$0.63}
    & \textbf{87.07} & \textcolor{gray}{\scriptsize $\pm$2.67}
    & 75.44 & \textcolor{gray}{\scriptsize $\pm$0.92}
    & 25.08 & \textcolor{gray}{\scriptsize $\pm$0.16}
    & 100.54 & \textcolor{gray}{\scriptsize $\pm$0.37}
    & \textbf{93.80} & \textcolor{gray}{\scriptsize $\pm$1.32} \\

    \midrule
    \multicolumn{18}{c}{\textbf{Task: Code} (Bench 1: HumanEval, Bench 2: MBPP)} \\
    \midrule

    FF & 8.37B
    & 56.29 &
    & 14.02 &
    & 6.89 &
    & 74.91 &
    & \textcolor{gray}{43.11} &
    & \textcolor{gray}{45.69} &
    & \textcolor{gray}{100.00} &
    & 87.46 & \\

    LoRA & 336M
    & 58.54 & \textcolor{gray}{\scriptsize $\pm$0.33}
    & 13.30 & \textcolor{gray}{\scriptsize $\pm$0.68}
    & 5.74 & \textcolor{gray}{\scriptsize $\pm$0.13}
    & 71.17 & \textcolor{gray}{\scriptsize $\pm$0.97}
    & 50.05 & \textcolor{gray}{\scriptsize $\pm$0.97}
    & \underline{47.89} & \textcolor{gray}{\scriptsize $\pm$0.51}
    & 110.46 & \textcolor{gray}{\scriptsize $\pm$0.89}
    & 90.82 & \textcolor{gray}{\scriptsize $\pm$0.90} \\

    PiSSA & 336M
    & 54.43 &
    & 15.32 &
    & 6.45 &
    & 74.42 &
    & \textbf{56.52} &
    & 46.08 &
    & \textbf{115.98} &
    & 95.20 & \\

    CorDA & 336M
    & 56.26 &
    & 17.37 &
    & 6.15 &
    & 77.52 &
    & {51.99} &
    & 45.88 &
    & 110.51 &
    & 94.02 & \\

    MiLoRA & 336M
    & \underline{60.28} &
    & \underline{18.12} &
    & \textbf{8.07} &
    & \underline{87.28} &
    & 50.24 &
    & \textbf{49.01} &
    & 111.90 &
    & \textbf{99.59} & \\

    LoRA-Null & 336M
    & 56.96 &
    & 15.15 &
    & 6.59 &
    & 75.99 &
    & \underline{53.44} &
    & {46.72} &
    & \underline{113.11} &
    & 94.55 & \\

    \textbf{JANUS-L} & 336M
    & \textbf{60.69} & \textcolor{gray}{\scriptsize $\pm$0.13}
    & \textbf{19.93} & \textcolor{gray}{\scriptsize $\pm$0.18}
    & \underline{7.78} & \textcolor{gray}{\scriptsize $\pm$0.15}
    & \textbf{89.28} & \textcolor{gray}{\scriptsize $\pm$0.27}
    & 49.66 & \textcolor{gray}{\scriptsize $\pm$1.13}
    & 47.43 & \textcolor{gray}{\scriptsize $\pm$0.40}
    & 109.50 & \textcolor{gray}{\scriptsize $\pm$0.98}
    & \underline{99.39} & \textcolor{gray}{\scriptsize $\pm$0.50} \\

    \midrule
    \multicolumn{18}{c}{\textbf{Task: IF} (Bench 1: MTBench, Bench 2: --)} \\
    \midrule

    FF & 8.37B
    & 42.69 &
    & 10.08 &
    & 3.94 &
    & 51.66 &
    & \textcolor{gray}{5.94} &
    & \multicolumn{2}{c}{--}
    & \textcolor{gray}{100.00} &
    & 75.83 & \\

    LoRA & 336M
    & 58.78 & \textcolor{gray}{\scriptsize $\pm$0.62}
    & 15.38 & \textcolor{gray}{\scriptsize $\pm$0.24}
    & 6.18 & \textcolor{gray}{\scriptsize $\pm$0.25}
    & 75.96 & \textcolor{gray}{\scriptsize $\pm$0.93}
    & \underline{6.05} & \textcolor{gray}{\scriptsize $\pm$0.21}
    & \multicolumn{2}{c}{--}
    & \underline{101.85} & \textcolor{gray}{\scriptsize $\pm$3.50}
    & 88.91 & \textcolor{gray}{\scriptsize $\pm$1.55} \\

    PiSSA & 336M
    & 48.40 &
    & 14.71 &
    & 6.25 &
    & 69.55 &
    & \textbf{6.73} &
    & \multicolumn{2}{c}{--}
    & \textbf{113.30} &
    & 91.43 & \\

    CorDA & 336M
    & 53.18 &
    & \underline{19.25} &
    & 4.68 &
    & 73.80 &
    & 5.88 &
    & \multicolumn{2}{c}{--}
    & 98.99 &
    & 86.39 & \\

    MiLoRA & 336M
    & \textbf{62.49} &
    & 17.92 &
    & \textbf{9.50} &
    & \textbf{92.97} &
    & 5.39 &
    & \multicolumn{2}{c}{--}
    & 90.74 &
    & \underline{91.85} & \\

    LoRA-Null & 336M
    & 53.42 &
    & 15.32 &
    & 4.63 &
    & 67.77 &
    & 5.96 &
    & \multicolumn{2}{c}{--}
    & 100.34 &
    & 84.05 & \\

    \textbf{JANUS-L} & 336M
    & \underline{61.47} & \textcolor{gray}{\scriptsize $\pm$0.39}
    & \textbf{20.43} & \textcolor{gray}{\scriptsize $\pm$0.29}
    & \underline{8.00} & \textcolor{gray}{\scriptsize $\pm$0.62}
    & \underline{91.23} & \textcolor{gray}{\scriptsize $\pm$2.47}
    & 5.94 & \textcolor{gray}{\scriptsize $\pm$0.06}
    & \multicolumn{2}{c}{--}
    & 100.00 & \textcolor{gray}{\scriptsize $\pm$1.05}
    & \textbf{95.61} & \textcolor{gray}{\scriptsize $\pm$0.71} \\

    \bottomrule
\end{tabular}
    }
\end{table}
Tabs.~\ref{tab:llama-2-7b} and \ref{tab:llama-3-8b}, together with Fig.~\ref{fig:llama-2-7b and llama-3-8b}, compare the performance of various methods.
To demonstrate compatibility with different FT methods, we highlight \textbf{JANUS-P} on LLaMA-2-7b and \textbf{JANUS-L} on LLaMA-3-8b (full results in Appendix~\ref{sec: complete main results}).
To enhance reliability, we evaluate PiSSA and JANUS-P on LLaMA-2-7b, as well as LoRA and JANUS-L on LLaMA-3-8b, across three random seeds (233, 234, and 235).
All results presented without error bars (including those in ablation studies) are obtained using seed 233.

Across both models, standard PEFT methods (LoRA, PiSSA) exhibit high plasticity but suffer from severe forgetting, while knowledge-preserving methods provide varying tradeoffs between stability and plasticity.
JANUS substantially improves this balance by recovering historical knowledge while largely retaining the adaptation achieved during FT.
On LLaMA-2-7b, JANUS-P raises $\text{AP}_1$ from 67.55--79.71\% to 100.27--109.83\% across the three tasks while largely retaining new-task performance.
On LLaMA-3-8b, JANUS-L raises $\text{AP}_1$ from 57.02--75.96\% to 87.07--91.23\%, while retaining over 98\% of LoRA's $\text{AP}_2$ on every task.

When JANUS offers only a limited advantage over knowledge-preserving methods on new tasks, this largely reflects the performance of the underlying FT method.
For example, JANUS-L rectifies LoRA-fine-tuned models, which themselves do not outperform knowledge-preserving methods on several new-task benchmarks.
Relative to LoRA, however, JANUS-L substantially improves historical knowledge retention while preserving nearly all new-task performance.
As shown in Fig.~\ref{fig:llama-2-7b and llama-3-8b-janus}, this improvement over the corresponding FT checkpoint extends across all evaluated FT methods, downstream tasks, and base models, supporting JANUS as an effective plug-and-play rectification module.
Notably, all rectifications use the same acceptance threshold $\tau=0.95$ and decay factor $\beta=0.7$, without task-specific tuning.

\begin{figure}[htbp]
    \centering
    \begin{subfigure}{0.325\textwidth} 
        \includegraphics[width=\textwidth,trim=0.6cm 0.58cm 0.6cm 1.7cm, clip]{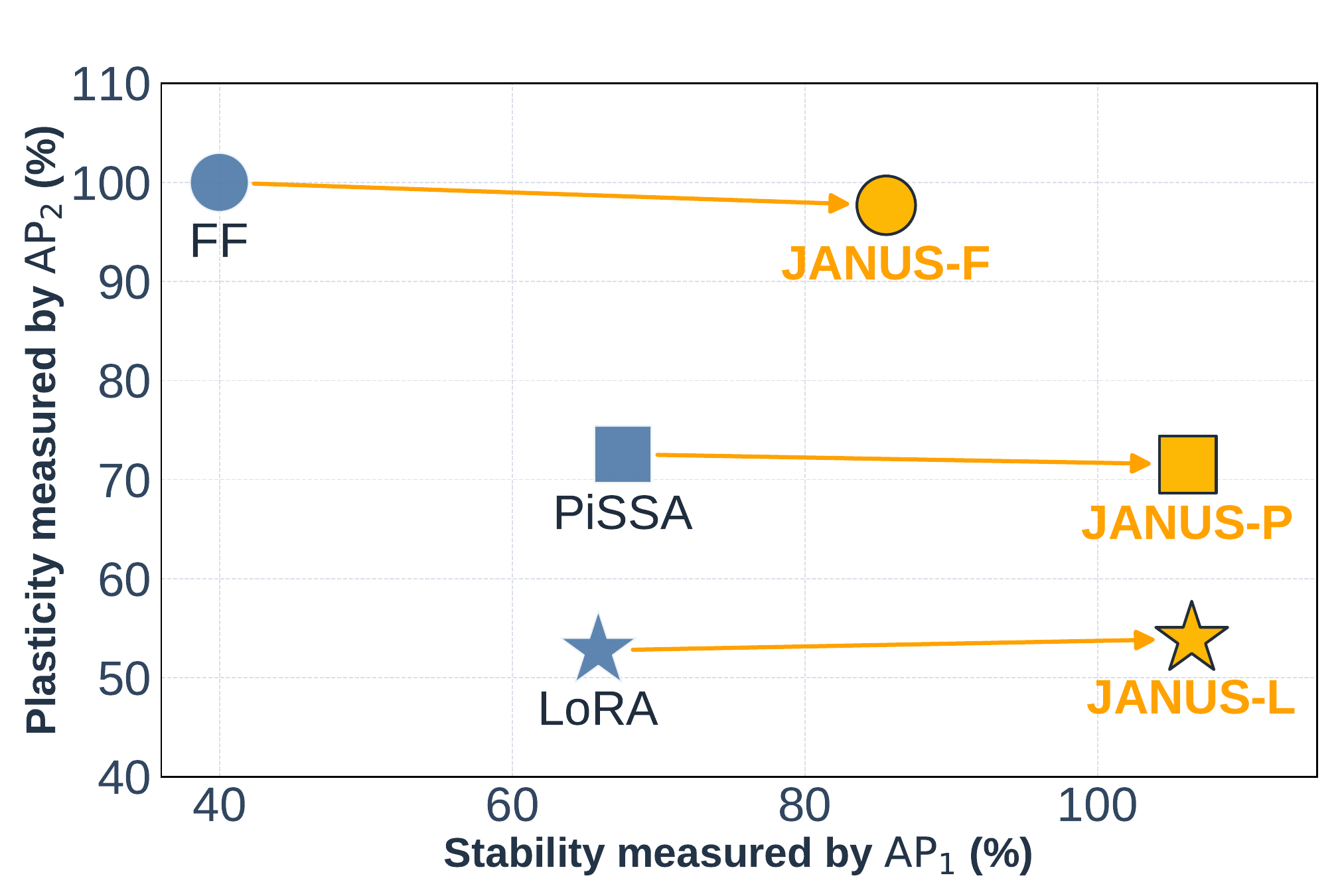}
        \caption{Math (LLaMA-2-7b)}
        \label{fig:llama-2-7b-Math-janus}
    \end{subfigure}
    \hfill
    \begin{subfigure}{0.325\textwidth}
        \includegraphics[width=\textwidth,trim=0.6cm 0.58cm 0.6cm 1.7cm, clip]{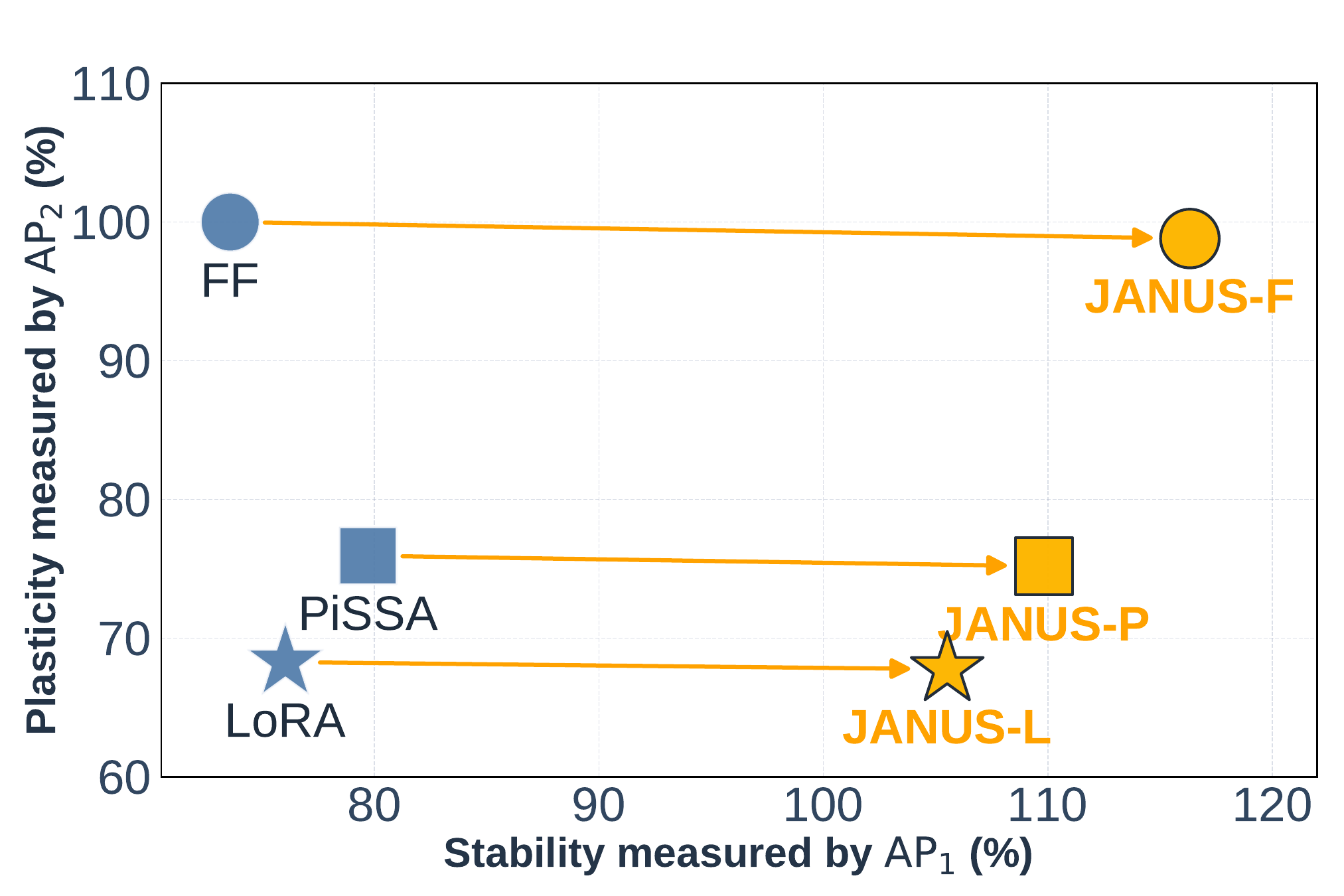}
        \caption{Code (LLaMA-2-7b)}
        \label{fig:llama-2-7b-Code-janus}
    \end{subfigure}
    \hfill
    \begin{subfigure}{0.325\textwidth}
        \includegraphics[width=\textwidth,trim=0.6cm 0.58cm 0.6cm 1.7cm, clip]{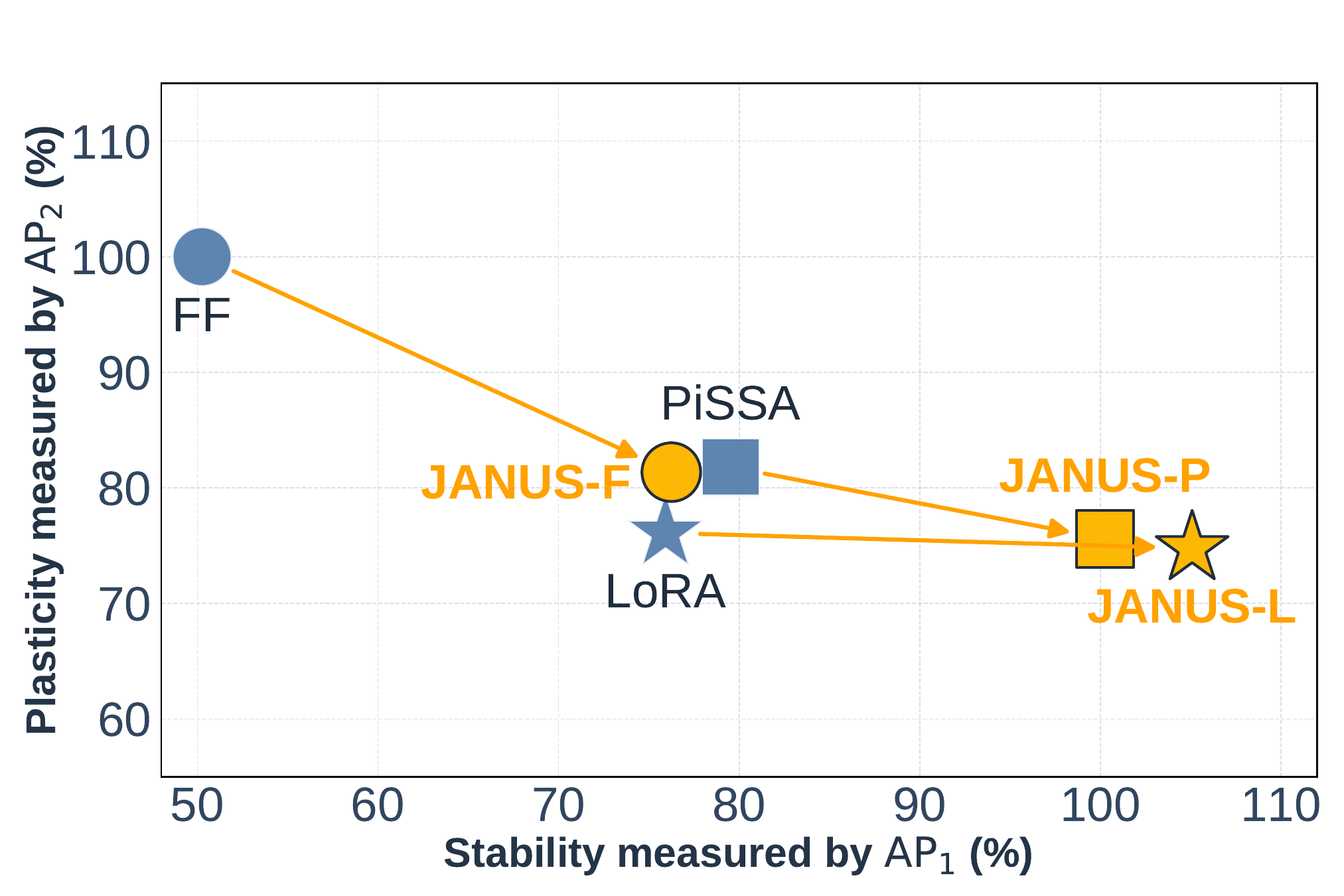}
        \caption{IF (LLaMA-2-7b)}
        \label{fig:llama-2-7b-Instruction Following-janus}
    \end{subfigure}
    
    \begin{subfigure}{0.325\textwidth} 
        \includegraphics[width=\textwidth,trim=0.6cm 0.58cm 0.6cm 1.7cm, clip]{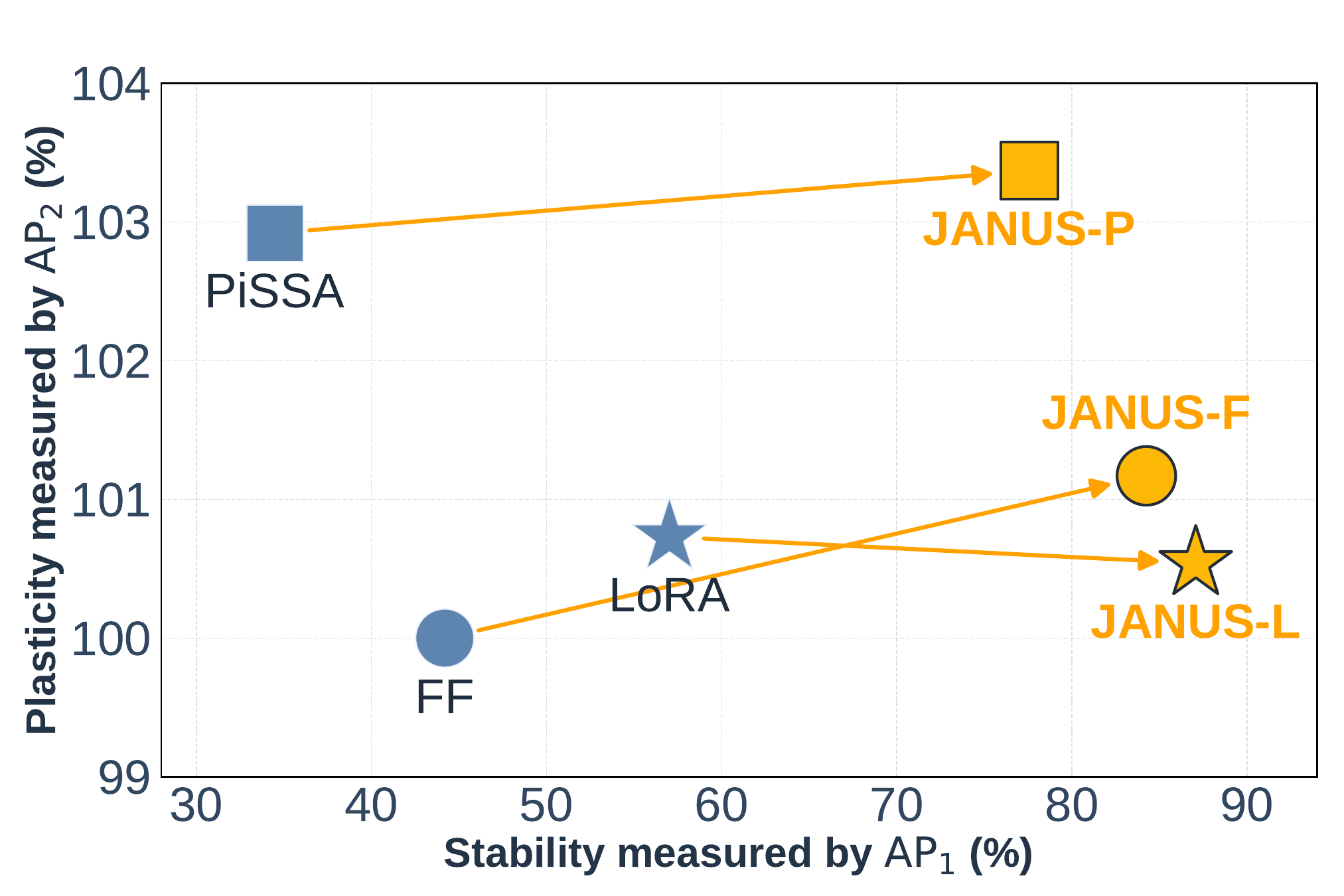}
        \caption{Math (LLaMA-3-8b)}
        \label{fig:llama-3-8b-Math-janus}
    \end{subfigure}
    \hfill
    \begin{subfigure}{0.325\textwidth}
        \includegraphics[width=\textwidth,trim=0.6cm 0.58cm 0.6cm 1.7cm, clip]{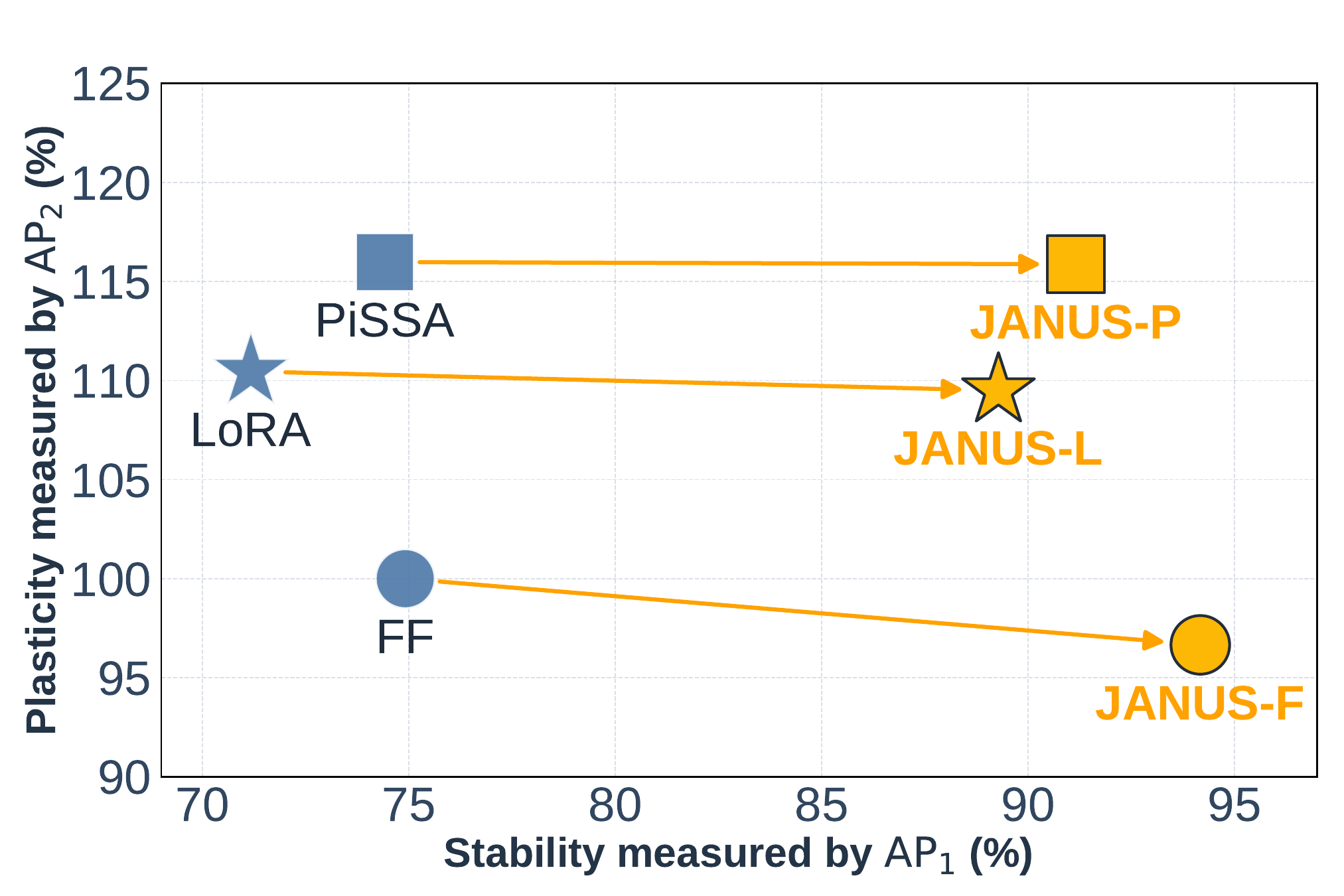}
        \caption{Code (LLaMA-3-8b)}
        \label{fig:llama-3-8b-Code-janus}
    \end{subfigure}
    \hfill
    \begin{subfigure}{0.325\textwidth}
        \includegraphics[width=\textwidth,trim=0.6cm 0.58cm 0.6cm 1.7cm, clip]{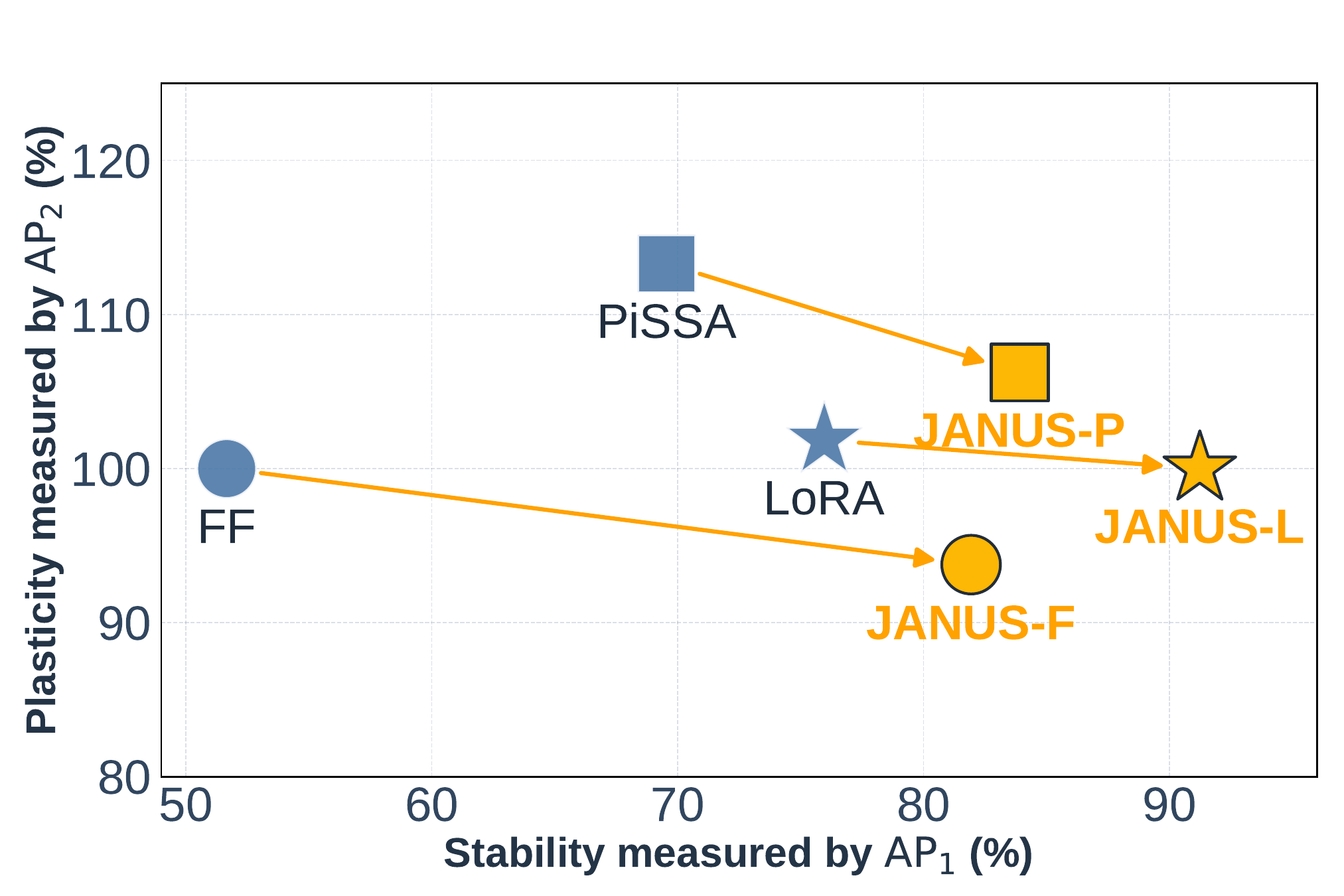}
        \caption{IF (LLaMA-3-8b)}
        \label{fig:llama-3-8b-Instruction Following-janus}
    \end{subfigure}
    \caption{Recovery in stability across three FT methods.}
    \label{fig:llama-2-7b and llama-3-8b-janus}
\end{figure}

\subsection{Ablation}
To evaluate the multi-step adaptive rectification, we ablate the mechanism by setting $\tau$ to zero, reducing it to single-step rectification (denoted as w/o).
As demonstrated in Tab.~\ref{tab:single step}, bypassing this mechanism drastically degrades stability across all tasks.
We also plot the recovered stability $\text{AP}_1$ after a single-step rectification and the corresponding JANUS shift $\bar{c}$ in Fig.~\ref{fig:ap1 vs meancos}.
Crucially, it reveals a statistically significant positive correlation ($p<0.05$) between $\text{AP}_1$ and $\bar{c}$.
This finding provides a solid foundation for our Multi-step Adaptive Rectification mechanism: by monitoring the JANUS shift, we can dynamically determine a proper step size to maximize stability recovery.
Fig.~\ref{fig:multi step rectification traj} provides a visualization of the complete multi-step rectification process for JANUS-F on LLaMA-3-8b and the Math task, using an acceptance threshold of $\tau=0.96$ and a decay factor of $\beta=0.8$.
It illustrates how our method navigates the parameter space along the loss contour, progressively advancing toward the region of high plasticity.
Complete results can be found in Appendix~\ref{sec: complete main results}.

For more ablations on the acceptance threshold $\tau$, decay factor $\beta$, compressed rank $r$, replay size $m$, and replay distribution, please refer to Appendix~\ref{sec: more ablations}.

\begin{table}[htbp]
    \centering
    \caption{Ablation study on the multi-step adaptive rectification mechanism.}
    \label{tab:single step}
    
    \resizebox{\textwidth}{!}{
    \begin{tabular}{lcccccc|cccccc}
        \toprule
        \multirow{5}{*}{Dataset} & \multicolumn{6}{c}{LLaMA-2-7b} & \multicolumn{6}{c}{LLaMA-3-8b} \\
        \cmidrule(lr){2-7} \cmidrule(lr){8-13}
        & \multicolumn{2}{c}{Math} & \multicolumn{2}{c}{Code} & \multicolumn{2}{c}{IF} & \multicolumn{2}{c}{Math} & \multicolumn{2}{c}{Code} & \multicolumn{2}{c}{IF} \\
        \cmidrule(lr){2-3} \cmidrule(lr){4-5} \cmidrule(lr){6-7} \cmidrule(lr){8-9} \cmidrule(lr){10-11} \cmidrule(lr){12-13}
        & JANUS-P & JANUS-P & JANUS-P & JANUS-P & JANUS-P & JANUS-P & JANUS-L & JANUS-L & JANUS-L & JANUS-L & JANUS-L & JANUS-L \\
        & (w/) & (w/o) & (w/) & (w/o) & (w/) & (w/o) & (w/) & (w/o) & (w/) & (w/o) & (w/) & (w/o) \\
        \midrule
        TriviaQA & \textbf{47.95} & 45.98 & \textbf{50.60} & 49.53 & \textbf{48.28} & 46.81 & \textbf{59.38} & 56.87 & \textbf{60.78} & 59.65 & \textbf{61.91} & 60.85 \\
        NQ open  & \textbf{16.73} & 11.16 & \textbf{16.26} & 13.16 & \textbf{14.82} & 10.97 & \textbf{20.64} & 13.38 & \textbf{19.94} & 17.06 & \textbf{20.66} & 17.15 \\
        WebQS    & \textbf{9.01}  & 7.48  & \textbf{8.32}  & 7.73  & \textbf{7.38}  & 6.74  & \textbf{7.92}  & 5.41  & \textbf{7.78}  & 6.45  & \textbf{8.66}  & 6.74  \\ \midrule
        $\text{AP}_1$(\%) & \textbf{111.55} & 91.73 & \textbf{108.45} & 98.93 & \textbf{99.05} & 87.67 & \textbf{90.13} & 69.29 & \textbf{89.35} & 79.89 & \textbf{94.01} & 81.65 \\
        \bottomrule
    \end{tabular}
    }
\end{table}

\begin{figure}[hbtp]
    \centering
    \begin{subfigure}{0.49\textwidth} 
        \includegraphics[width=\textwidth,trim=3.6cm 0cm 2.2cm 1.8cm, clip]{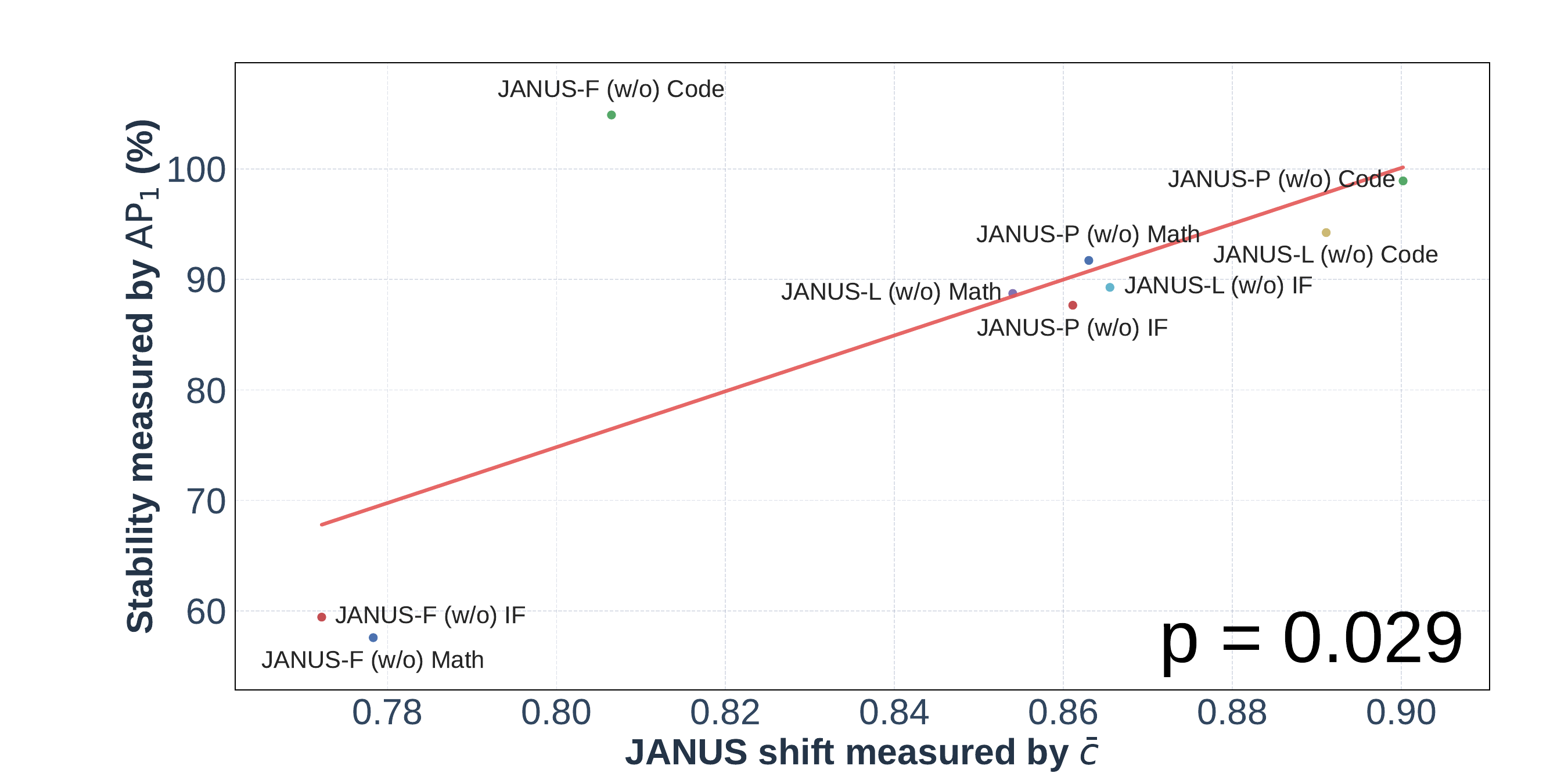}
        \caption{LLaMA-2-7b}
    \end{subfigure}
    \hfill
    \begin{subfigure}{0.49\textwidth}
        \includegraphics[width=\textwidth,trim=3.6cm 0cm 2.2cm 1.8cm, clip]{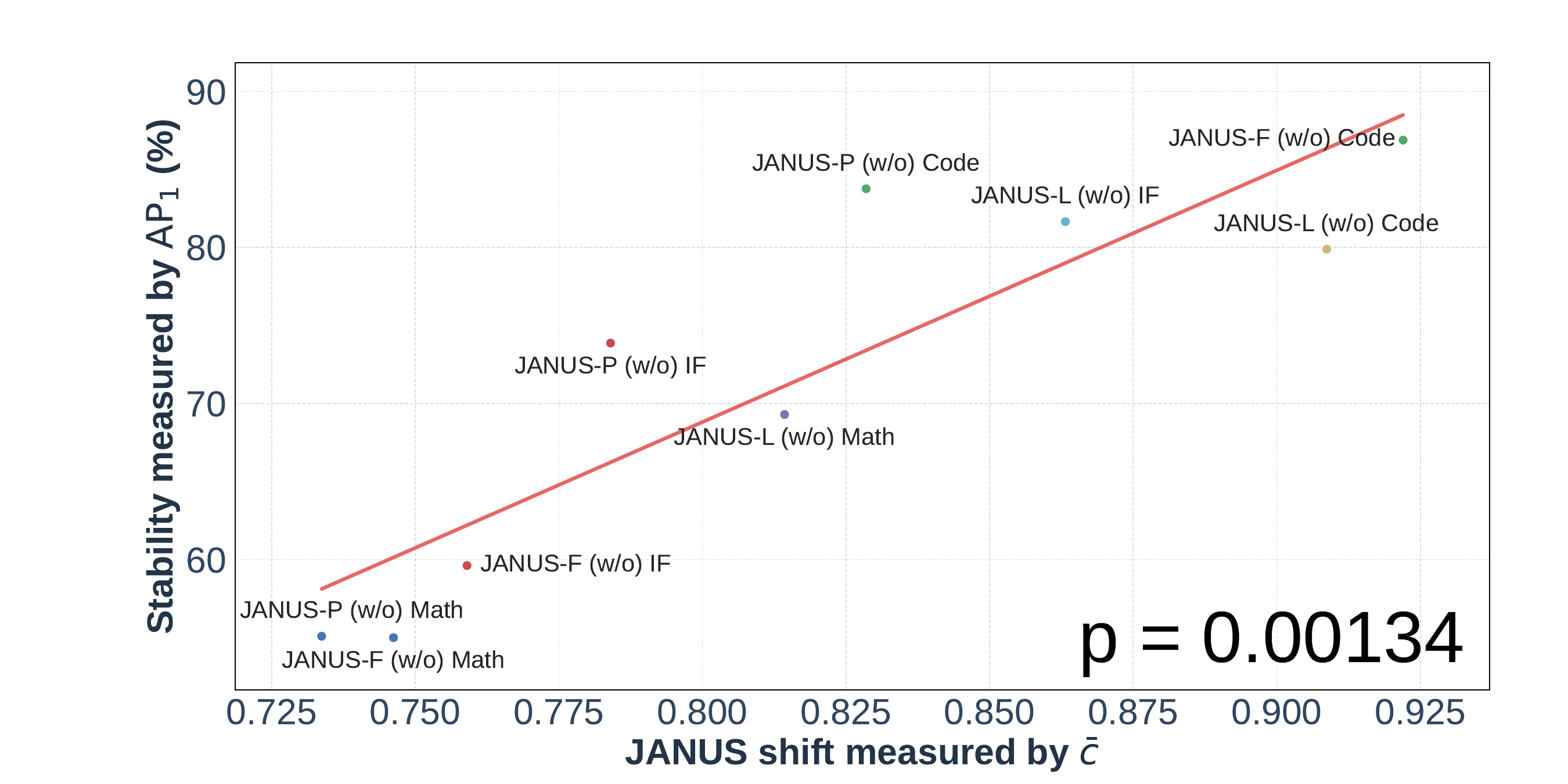}
        \caption{LLaMA-3-8b}
    \end{subfigure}
    \caption{Correlation between the average cosine of principal angles and stability.}
    \label{fig:ap1 vs meancos}
\end{figure}

\section{Conclusion and Limitation}
\label{sec: conclusion}
In this paper, we introduced JANUS, a purely post-hoc and tuning-agnostic weight rectification framework designed to mitigate catastrophic forgetting in large-scale models. 
Grounded in a theoretical paradigm shift from subspace orthogonality to parameter space orthogonality, JANUS effectively recovers compromised knowledge by projecting parameter updates onto the historical Jacobian null space.
We further introduced JANUS shift, defined as the average cosine of principal angles between Jacobian row spaces, which effectively quantifies the extent of forgetting and determines when it is necessary to recompute the JANUS.
Leveraging this metric, we developed a Multi-step Adaptive Rectification mechanism which effectively overcomes the locality of Jacobian approximation.
Powered by our proposed ghost operations and SVD compression, the rectification process is both temporally and spatially efficient.
% To ensure practical scalability, we also developed ghost projection, ghost orientation comparison, and SVD compression techniques.
Empirical evaluations confirm that JANUS serves as a highly efficient, plug-and-play module for diverse FT methods, significantly mitigating the stability-plasticity dilemma by recovering historical knowledge while preserving downstream task adaptation.

The theoretical guarantee of all orthogonality-based methods relies on local first-order approximations. Although our multi-step mechanism mitigates this by monitoring trust regions, navigating highly rugged loss landscapes under extreme domain shifts may necessitate numerous conservative, small step sizes, which inherently extends the overall processing time.
Additionally, while our current study focuses on the pretraining-to-finetuning setting using LLaMA-family models, extending JANUS to multi-task continual learning and broader architectures remains a promising direction for future work.

\section{Acknowledgments and Disclosures}
This study is supported by the National Natural Science Foundation of China under Grant No. 92582205 and the Beijing Natural Science Foundation (L257002).
The authors declare no competing interests.
\newpage
\bibliographystyle{unsrt}
\bibliography{references}

% References follow the acknowledgments in the camera-ready paper. Use unnumbered first-level heading for
% the references. Any choice of citation style is acceptable as long as you are
% consistent. It is permissible to reduce the font size to \verb+small+ (9 point)
% when listing the references.
% Note that the Reference section does not count towards the page limit.
% \medskip

%%%%%%%%%%%%%%%%%%%%%%%%%%%%%%%%%%%%%%%%%%%%%%%%%%%%%%%%%%%%
\newpage

\appendix
\section{Further Theoretical Analysis}
\subsection{Proof of Theorem~\ref{thm:global_forgetting_free}}
\label{sec:proof of thm global_forgetting_free}
\GlobalForgettingFree*

\begin{proof}
Considering the parameter perturbations across all $L$ layers simultaneously, the total change in loss for the $i$-th data point ($1 \le i \le m$) can be approximated as:
\begin{align}
    \Delta\mathcal{L}_{A}^{(i)} &\approx \begin{bmatrix}
        {\bm{g}^{(i,1)}}^\top & {\bm{g}^{(i,2)}}^\top & \dots & {\bm{g}^{(i,L)}}^\top
    \end{bmatrix}
    \begin{bmatrix}
        \Delta \bm{\theta}^{(1)}_\text{proj} \\
        \Delta \bm{\theta}^{(2)}_\text{proj} \\
        \vdots \\
        \Delta \bm{\theta}^{(L)}_\text{proj}
    \end{bmatrix} \notag\\
    &= \sum_{j=1}^L {\bm{g}^{(i,j)}}^\top \Delta \bm{\theta}^{(j)}_\text{proj} \notag \\
    &=0. \tag{According to Eq.~\eqref{eq:layer_wise_ortho}}
\end{align}
% \begin{align} 
%     \begin{bmatrix}
%         \Delta\mathcal{L}_{A}^{(1)} \\
%         \Delta\mathcal{L}_{A}^{(2)} \\
%         \vdots \\
%         \Delta\mathcal{L}_{A}^{(m)}
%     \end{bmatrix}
%     &\approx 
%     \begin{bmatrix}
%         \bm{J}^{(1)} & \bm{J}^{(2)} & \dots & \bm{J}^{(L)}
%     \end{bmatrix}
%     \begin{bmatrix}
%         \Delta \bm{\theta}^{(1)}_\text{proj} \\
%         \Delta \bm{\theta}^{(2)}_\text{proj} \\
%         \vdots \\
%         \Delta \bm{\theta}^{(L)}_\text{proj}
%     \end{bmatrix} \notag  \tag{Recall Eq.~\eqref{eq: J^j}}\\
%     &= \sum_{j=1}^L \bm{J}^{(j)} \Delta \bm{\theta}^{(j)}_\text{proj} \notag\\
%     &= \sum_{j=1}^L 
%     \begin{bmatrix}
%         {\bm{g}^{(1,j)}}^\top \Delta \bm{\theta}^{(j)}_\text{proj} \\
%         {\bm{g}^{(2,j)}}^\top \Delta \bm{\theta}^{(j)}_\text{proj} \\
%         \vdots \\
%         {\bm{g}^{(m,j)}}^\top \Delta \bm{\theta}^{(j)}_\text{proj}
%     \end{bmatrix}\notag \\
%     &= 
%     \begin{bmatrix}
%         0 \\
%         0 \\
%         \vdots \\
%         0
%     \end{bmatrix}. \tag{According to Eq.~\eqref{eq:layer_wise_ortho}}
% \end{align}
This completes the proof.
\end{proof}

\subsection{Beyond First-Order: A Gauss-Newton Perspective of JANUS}
\label{sec:gauss_newton_perspective}

A natural question arises regarding our formulation: why should we project the parameter updates into the null space of the full Jacobian matrix $J$, which consists of individual gradients for each data point, rather than simply projecting them to be orthogonal to the average gradient? We seek to answer this question by formulating an equivalent least-squares problem to the original loss.

Let $\mathcal{L}_1\left(\bm{\theta}\right) = \frac{1}{m} \sum_{i=1}^m \mathcal{L}_A^{(i)}\left(\bm{\theta}\right)$
be the original loss, and let $\bm{\theta}^*$ denote the optimal solution of $\min_{\bm{\theta}} \mathcal{L}_1\left(\bm{\theta}\right)$. A corresponding loss function in the least-squares form can be formulated as
\begin{equation}
    \mathcal{L}_2\left(\bm{\theta}\right) = \frac{1}{2m} \sum_{i=1}^m \left( \mathcal{L}_A^{(i)}\left(\bm{\theta}\right) - \mathcal{L}_A^{(i)}\left(\bm{\theta}^*\right) \right)^2.
\end{equation}
This formulation is equivalent to $\mathcal{L}_1$ in the sense that the optimal solution of $\min_{\bm{\theta}} \mathcal{L}_2\left(\bm{\theta}\right)$ is also strictly $\bm{\theta}^*$, since $\mathcal{L}_2\left(\bm{\theta}^*\right)=0$.
Given that the initial parameter $\bm{\theta}_0$ is the outcome of pretraining on the past task $\mathcal{T}_A$, we can assume that $\mathcal{L}_A^{(i)}\left(\bm{\theta}_0\right) \approx \mathcal{L}_A^{(i)}\left(\bm{\theta}^*\right)$. Consequently, according to the Gauss-Newton approximation \cite{nocedal2006numerical}, the Hessian matrix of $\mathcal{L}_2\left(\bm{\theta}\right)$ around $\bm{\theta}_0$ can be derived as
\begin{equation}
\begin{aligned}
\label{eq:gauss newton approx}
    \bm{H}_{\mathcal{L}_2}\left(\bm{\theta}\right) &= \frac{1}{m} \sum_{i=1}^m \bm{g}^{(i)}\left(\bm{\theta}\right) \bm{g}^{(i)}\left(\bm{\theta}\right)^\top + \frac{1}{m} \sum_{i=1}^m \left( \mathcal{L}_A^{(i)}\left(\bm{\theta}\right) -\mathcal{L}_A^{(i)}\left(\bm{\theta}^*\right) \right) \bm{H}_{\mathcal{L}_A^{(i)}}\left(\bm{\theta}\right) \\
    &\approx \frac{1}{m} \sum_{i=1}^m \bm{g}^{(i)}\left(\bm{\theta}\right) \bm{g}^{(i)}\left(\bm{\theta}\right)^\top \\
    &= \frac{1}{m} \bm{J}^\top \bm{J},
\end{aligned}
\end{equation}
where $\bm{g}^{(i)}=\left[{\bm{g}^{(i,1)}}^\top, {\bm{g}^{(i,2)}}^\top,\dots,{\bm{g}^{(i,L)}}^\top\right]^\top \in \mathbb{R}^{n}$ and $\bm{J} = \left[ \bm{g}^{(1)}, \bm{g}^{(2)}, \dots, \bm{g}^{(m)} \right]^\top \in \mathbb{R}^{m \times n}$ are the global gradient and Jacobian, respectively. 

Eq.~\eqref{eq:gauss newton approx} indicates that the null space of the Jacobian $J$ closely approximates that of the equivalent Hessian $\bm{H}_{\mathcal{L}_2}$.
Therefore, by rectifying the parameter update $\Delta \bm{\theta}$ into $\mathcal{N}\left(\bm{J}\right)$, we effectively constrain the update to move along the flat loss valleys of the equivalent problem.
While directly computing the true Hessian $\bm{H}_{\mathcal{L}_1}\left(\bm{\theta}\right)$ incurs a prohibitive $O\left(n^2\right)$ memory complexity for large models, our Jacobian projection requires only first-order computation and memory. In essence, it is derived \textit{from} first-order information, yet operates \textit{beyond} first-order limitations.

The ablation study in Tab.~\ref{tab:ablation mean jac} highlights the necessity of using Jacobian matrices.
When simplified to average gradients, the stability of JANUS-P collapses to the level of PiSSA, confirming that individual gradient information is critical for precise weight rectification.

\begin{table}[htbp]
    \centering
    \caption{Ablation on the use of Jacobian matrices instead of average gradient.}
    \label{tab:ablation mean jac}
    \small 
    \resizebox{\textwidth}{!}{
    \begin{tabular}{l|c|cccc|ccc|c}
        \toprule
        Method & \#Param & TriviaQA & NQ open & WebQS & $\text{AP}_1$(\%) & GSM8k & Math & $\text{AP}_2$(\%) & $\text{AP}$(\%) \\
        \midrule
        LLaMA-2-7b & -- & \textcolor{gray}{52.52} & \textcolor{gray}{18.95} & \textcolor{gray}{5.81} & \textcolor{gray}{100.00} & -- & -- & -- & -- \\ 
        \midrule
        FF & 6.7B & 19.19 & 0.86 & 4.58 & 39.97 & \textcolor{gray}{60.20} & \textcolor{gray}{12.56} & \textcolor{gray}{100.00} & 69.98 \\
        PiSSA & 320M & 42.92 & 3.21 & 6.69 & 71.27 & {51.86} & {7.76} & {73.96} & {72.62} \\
        \textbf{JANUS-P} & 320M & {47.95} & {16.73} & {9.01} & {111.55} & {51.86} & {7.52} & {73.01} & {92.28} \\
        \textbf{JANUS-P (avg. grad)} & 320M & 42.92 & 3.27 & 6.59 & 70.80 & 51.78 & 7.54 & 73.02 & 71.91 \\
        \bottomrule
    \end{tabular}}
\end{table}

\newpage
\section{Details of Efficient Calculation and Storage of the Jacobian}
\label{sec:details of ghost}
\subsection{Ghost Projection}
Given a batch of $m$ data points, the input of the $j$-th linear layer can be formulated as a tensor $\bm X^{(j)} = \left[\bm x_{i,k}^{(j)}\right] \in \mathbb{R}^{m\times T\times d_{\text{in}}}$, where $\bm x_{i,k}^{(j)} \in \mathbb{R}^{d_{\text{in}}}$ denotes the input vector corresponding to the $k$-th token of the $i$-th sample. Similarly, we denote the corresponding output matrix of this layer as $\bm Y^{(j)}=\left[\bm y_{i,k}^{(j)}\right] \in \mathbb{R}^{m\times T\times d_{\text{out}} }$, comprising the output vectors $\bm y_{i,k}^{(j)} \in \mathbb{R}^{d_{\text{out}}}$.
According to the chain rule and the typical loss formulations in LLMs, the gradient of the individual sample loss $\mathcal{L}_A^{(i)}$ with respect to the weight $\bm \theta^{(j)}$ can be computed as:
\begin{equation}
\begin{aligned}
    \bm g^{(i,j)}&=
    \text{rvec}\left(\nabla_{\bm W^{(j)}}\mathcal{L}_A^{(i)}\right) \\
    &= \text{rvec}\left(\sum_{k=1}^{T} \nabla_{\bm y_{i,k}^{(j)}} \mathcal{L}_A^{(i)} \left( \bm x_{i,k}^{(j)} \right)^\top\right) \\
    &= \text{rvec}\left(\sum_{k=1}^{T} \nabla_{\bm y_{i,k}^{(j)}} \mathcal{L}_1 \left( \bm x_{i,k}^{(j)} \right)^\top\right) \\
    &\triangleq\text{rvec}\left(\sum_{k=1}^{T} \bm a_{i,k}^{(j)} \left( \bm x_{i,k}^{(j)} \right)^\top\right),
\end{aligned}
\end{equation}
where $\mathcal{L}_1\left(\bm{\theta}\right) = \frac{1}{m} \sum_{i=1}^m \mathcal{L}_A^{(i)}\left(\bm{\theta}\right)$ is the batch-level loss, and the last step holds because $\bm y_{i,k}^{(j)}$ has no dependency on the losses of the other data points. 

In standard LLM frameworks, the default output loss is typically averaged over all valid (non-padded) tokens.
Directly backpropagating this averaged loss would scale the gradients down by a factor of $1/T_{\text{valid}}$, causing the magnitude of the resulting Jacobian matrix to fluctuate arbitrarily depending on the varying sequence lengths of different batches. 
To strictly adhere to our mathematical derivation which requires the additive accumulation of token-wise gradients, we explicitly "un-reduce" the loss by multiplying it by the total number of valid tokens before backpropagation.
By performing a single backward pass with respect to this unreduced loss, we can efficiently obtain the exact pre-activation gradients $\bm A^{(j)} = \left[\bm a_{i,k}^{(j)}\right] \in \mathbb{R}^{m\times T\times d_{\text{out}}}$ along with the layer inputs $\bm X^{(j)}$.

These two tensors jointly enable the exact construction of the per-sample gradients, which conceptually form the full Jacobian matrix:
\begin{equation}
    \bm J^{(j)} = \begin{bmatrix}
    {\bm g^{(1,j)}}^\top \\
    {\bm g^{(2,j)}}^\top \\
    \vdots \\
    {\bm g^{(m,j)}}^\top
    \end{bmatrix}
    = \begin{bmatrix}
    \text{rvec}\left( \sum_{k=1}^{T} \bm a_{1,k}^{(j)} \left( \bm x_{1,k}^{(j)} \right)^\top \right)^\top \\
    \text{rvec}\left( \sum_{k=1}^{T} \bm a_{2,k}^{(j)} \left( \bm x_{2,k}^{(j)} \right)^\top \right)^\top \\
    \vdots \\
    \text{rvec}\left( \sum_{k=1}^{T} \bm a_{m,k}^{(j)} \left( \bm x_{m,k}^{(j)} \right)^\top \right)^\top
    \end{bmatrix}.
\end{equation}
Crucially, all necessary computations for the JANUS projection and orientation comparison can be executed directly utilizing the tensors $\bm X^{(j)}$ and $\bm A^{(j)}$, without the need to instantiate and store the full Jacobian at all.

Specifically, the JANUS projection computes the parameter update as follows:
\begin{equation}
\label{eq:proj_appdx}
    \Delta \bm{\theta}^{(j)}_\text{proj} = \Delta \bm{\theta}^{(j)} - {\bm{J}^{(j)}}^\top \underbrace{ \left( \bm{J}^{(j)} {\bm{J}^{(j)}}^\top \right)^{-1} \bm{J}^{(j)} \Delta \bm{\theta}^{(j)} }_{\bm v^{(j)}}.
\end{equation}
This relies on three core operations: the evaluation of the Gram matrix $\bm J^{(j)}{\bm J^{(j)}}^\top \in \mathbb{R}^{m\times m}$, the Jacobian-vector product (JVP) $\bm J^{(j)}\Delta\bm \theta^{(j)} \in \mathbb{R}^{m}$, and the (transposed) vector-Jacobian product (VJP) ${\bm J^{(j)}}^\top\bm v^{(j)} \in \mathbb{R}^{d_\text{in}d_\text{out}}$.

First, the $(u, v)$-th entry of the Gram matrix can be decoupled into the Hadamard product of two smaller dot products:
\begin{equation}
\label{eq: gram}
    \left( \bm J^{(j)}{\bm J^{(j)}}^\top \right)_{u, v} = {\bm g^{(u,j)}}^\top \bm g^{(v,j)} = \sum_{k=1}^{T} \sum_{k'=1}^{T} \left( \left(\bm a_{u,k}^{(j)}\right)^\top \bm a_{v,k'}^{(j)} \right) \left( \left(\bm x_{u,k}^{(j)}\right)^\top \bm x_{v,k'}^{(j)} \right).
\end{equation}
For improved numerical precision, we compute the inverse of the Gram matrix in \texttt{float64} precision.
Second, the $u$-th element of the JVP vector can be efficiently computed via vector-matrix-vector multiplication:
\begin{equation}
    \left(\bm J^{(j)}\Delta\bm \theta^{(j)}\right)_u = {\bm g^{(u,j)}}^\top \Delta\bm \theta^{(j)} = \sum_{k=1}^{T} \left(\bm a_{u,k}^{(j)}\right)^\top \Delta \bm W^{(j)} \bm x_{u,k}^{(j)}.
\end{equation}
Finally, given the intermediate vector $\bm v^{(j)} = \left[v_1^{(j)}, \dots, v_m^{(j)}\right]^\top \in \mathbb{R}^m$, the VJP vector is inherently a linear combination of the per-sample gradient vectors.
Let $\bm P^{(j)} \in \mathbb{R}^{d_{\text{out}} \times d_{\text{in}}}$ be the matrix form of the VJP vector.
It can be computed as 
\begin{equation}
    \bm P^{(j)} = \sum_{i=1}^{m} v_i^{(j)} \sum_{k=1}^{T} \bm a_{i,k}^{(j)} \left(\bm x_{i,k}^{(j)}\right)^\top.
\end{equation}

In practice, these formulations are highly parallelizable and can be strictly reduced to hardware-efficient tensor contractions (e.g., \texttt{torch.einsum}).

\subsection{Ghost Orientation Comparison}
The JANUS orientation comparison quantifies the orientation shift between the current and trial Jacobian matrices, denoted as $\bm{J}_1^{(j)}$ and $\bm{J}_2^{(j)}$, by computing the mean cosine of the principal angles between their respective row spaces. Let $\bm{G}_{1}^{(j)} = \bm{J}_1^{(j)} {\bm{J}_1^{(j)}}^\top$ and $\bm{G}_{2}^{(j)} = \bm{J}_2^{(j)} {\bm{J}_2^{(j)}}^\top$ denote the individual Gram matrices, and $\bm{G}_{12}^{(j)} = \bm{J}_1^{(j)} {\bm{J}_2^{(j)}}^\top$ be the cross-Gram matrix.
All these matrices can be evaluated efficiently without instantiating the Jacobians, as derived in Eq.~\eqref{eq: gram}.

Next, letting $\bm{G}_1^{(j)} = \bm{L}_1^{(j)} \bm{\Lambda}_1^{(j)} {\bm{L}_1^{(j)}}^\top$ and $\bm{G}_2^{(j)} = \bm{L}_2^{(j)} \bm{\Lambda}_2^{(j)} {\bm{L}_2^{(j)}}^\top$ represent the eigendecompositions of the individual Gram matrices, the cosine of the $i$-th principal angle corresponds to the $i$-th singular value of the interaction matrix:
\begin{equation}
    \bm M^{(j)} = \left(\bm{\Lambda}_1^{(j)}\right)^{-1/2} {\bm{L}_1^{(j)}}^\top \bm{G}_{12}^{(j)} \bm{L}_2^{(j)} \left(\bm{\Lambda}_2^{(j)}\right)^{-1/2}.
\end{equation}
Since $\bm G_{1}^{(j)}$, $\bm G_{2}^{(j)}$, $\bm G_{12}^{(j)}$, and $\bm M^{(j)}$ are all low-dimensional matrices of size $m \times m$, performing these eigendecompositions and the subsequent SVD incurs negligible computational overhead.

\subsection{Sequence-level SVD Compression}
Caching the dense tensors $\bm X^{(j)}$ and $\bm A^{(j)}$ for all layers across a large sequence length $T$ can still pose a substantial memory burden.
However, leveraging the intrinsic low-rank property of $\bm X^{(j)}$ and $\bm A^{(j)}$, we can significantly compress these tensors along the sequence dimension using rank-$r$ Truncated SVD.

Specifically, let $\bm X_i^{(j)} \in \mathbb{R}^{T\times d_{\text{in}}}$ and $\bm A_i^{(j)} \in \mathbb{R}^{T\times d_{\text{out}}}$ be the layer inputs and pre-activation gradients for the $i$-th sample.
We first perform QR decomposition on their transposes:
\begin{equation}
    {\bm{X}_i^{(j)}}^\top = \bm{Q}_X \bm{R}_X, \quad {\bm{A}_i^{(j)}}^\top = \bm{Q}_A \bm{R}_A,
\end{equation}
where $\bm{Q}_X, \bm{Q}_A$ are orthogonal matrices and $\bm{R}_X, \bm{R}_A$ are upper triangular matrices. The gradient for the $i$-th sample is therefore $\bm{g}^{(i,j)} =\text{rvec}\left({\bm{A}_i^{(j)}}^\top \bm{X}_i^{(j)}\right)=\text{rvec}\left(\bm{Q}_A \left(\bm{R}_A \bm{R}_X^\top\right) \bm{Q}_X^\top\right)$.
By computing the rank-$r$ truncated SVD on the core matrix $\bm{M} = \bm{R}_A \bm{R}_X^\top \approx \bm{U}_r \bm{\Sigma}_r \bm{V}_r^\top$, we can symmetrically distribute the singular values to form two compressed, low-dimensional tensors $\hat{\bm{X}}^{(j)} \in \mathbb{R}^{m\times r\times d_{\text{in}}}$ and $\hat{\bm{A}}^{(j)} \in \mathbb{R}^{m\times r\times d_{\text{out}}}$. 
For the $i$-th sample, the compressed slices are strictly defined as:
\begin{equation}
    \hat{\bm{X}}_i^{(j)} = \left(\bm{Q}_X \bm{V}_r \bm{\Sigma}_r^{1/2}\right)^\top, \quad \hat{\bm{A}}_i^{(j)} = \left(\bm{Q}_A \bm{U}_r \bm{\Sigma}_r^{1/2}\right)^\top.
\end{equation}
By construction, their inner product directly recovers the optimal rank-$r$ approximation of the original gradient matrix, successfully preserving the principal information along the sequence dimension:
\begin{equation}
   \left(\hat{\bm{A}}_i^{(j)}\right)^\top \hat{\bm{X}}_i^{(j)} = \bm{Q}_A \bm{U}_r \bm{\Sigma}_r \bm V_r^\top\bm Q_X^\top \approx \bm{Q}_A \left(\bm{R}_A \bm{R}_X^\top\right) \bm{Q}_X^\top = \left(\bm{A}_i^{(j)}\right)^\top \bm{X}_i^{(j)}.
\end{equation}
% \begin{equation}
% \begin{aligned}
%     \sum_{k=1}^{r} \hat{\bm{a}}^{(j)}_{i,k}\left(\hat{\bm{x}}^{(j)}_{i,k}\right)^\top &= \left(\hat{\bm{A}}_i^{(j)}\right)^\top \hat{\bm{X}}_i^{(j)} \\
%     &= \bm{Q}_A \bm{U}_r \bm{\Sigma}_r \bm V_r\bm Q_X^\top \\
%     &\approx \bm{Q}_A \left(\bm{R}_A \bm{R}_X^\top\right) \bm{Q}_X^\top \\
%     &= \left(\bm{A}_i^{(j)}\right)^\top \bm{X}_i^{(j)} \\
%     &= \sum_{k=1}^{T} {\bm{a}}^{(j)}_{i,k}\left({\bm{x}}^{(j)}_{i,k}\right)^\top.
% \end{aligned}
% \end{equation}
\subsection{Time Cost Profiling}
\label{sec: time cost profiling}
To evaluate computational efficiency, we break down the total execution time of the rectification process for JANUS-P on LLaMA-2-7b and the Math task into four key components: (1) forward and backward passes, (2) ghost projection and orientation comparison, (3) SVD compression, and (4) disk I/O.
The results are listed in Tab.~\ref{tab: time profile}.
The first thing worth noting is that forward and backward passes account for less than 5\% of the total execution time.
This efficiency is directly attributed to our ghost operations, which bypass the need for expensive per-sample backpropagation.
On the other hand, disk I/O constitutes a significant portion of the overall time cost (28.19\%).
This is because in our current implementation, the compressed tensors $\hat{\bm{X}}^{(j)}$ and $\hat{\bm{A}}^{(j)}$ are temporarily cached to disk.
However, this overhead could be substantially mitigated in future optimizations through parallel loading or memory offloading.
Furthermore, while our current implementation employs exact SVD for tensor compression, adopting randomized approximate SVD may have the potential to significantly accelerate this stage while preserving performance.

\begin{table}[htbp]
    \centering
    \caption{Time cost profile of JANUS-P on LLaMA-2-7b and the Math task.}
    \label{tab: time profile}
    \begin{tabular}{lcc}
        \toprule
        Component & Time (s) & Percentage (\%) \\
        \midrule
        SVD compression & 524.85 & 36.07 \\
        Disk I/O & 410.29 & 28.19 \\
        Ghost projection and orientation comparison & 351.60 & 24.16\\
        Others & 103.62 & 7.12 \\
        Forward and backward passes & 64.84 & 4.46 \\
        \midrule
        Total & 1455.20 & 100.00 \\
        \bottomrule
    \end{tabular}
\end{table}

\newpage
\section{Implementation Details}
\label{sec: implementation details}
\subsection{Datasets and Benchmarks}
Throughout our experiments, we utilize the training dataset released by PiSSA\footnote{\url{https://huggingface.co/datasets/fxmeng/pissa-dataset}\label{pissa footnote}} \cite{meng2024pissa}.
This dataset comprises several subsets, including ``metamath'', ``python'', and ``conversation'', which correspond to the Math, Code, and IF tasks, respectively.
Specifically, the training data for the Math task is sourced from MetaMathQA \cite{yu2023metamath}, the Code task from CodeFeedback \cite{zheng2024opencodeinterpreter}, and the IF task from WizardLM-Evol-Instruct \cite{xuwizardlm}. 

For evaluation, we employ publicly available benchmarks: TriviaQA, NQ open, and WebQS for world knowledge\footnote{\url{https://github.com/EleutherAI/lm-evaluation-harness}}; GSM8k and MATH for the Math task\textsuperscript{\ref{pissa footnote}}; HumanEval and MBPP for the Code task\footnote{\url{https://github.com/bigcode-project/bigcode-evaluation-harness}}; and MTBench for the IF task\footnote{\url{https://github.com/lm-sys/FastChat}}.

\subsection{Training Configuration}
For LLaMA-2-7b, we follow the experimental setup established in \cite{yang2024corda, tang2026put, meng2024pissa}.
Specifically, models are optimized using the AdamW optimizer with a batch size of 128, a maximum sequence length of 512, and a learning rate of $2 \times 10^{-5}$ regulated by a cosine annealing schedule (warmup ratio of 0.03).
No weight decay is applied.

On LLaMA-3-8b, we observed that a learning rate of $2 \times 10^{-5}$ was insufficient to induce significant forgetting. To evaluate the robustness of our method under more substantial weight shifts, we increased the learning rate to $5 \times 10^{-5}$ for all methods except FF.
For FF on LLaMA-3-8b, the training process is highly unstable under large learning rates.
Hence, we employed a learning rate of $5 \times 10^{-6}$ for the Math and IF tasks, and $2 \times 10^{-6}$ for the Code task. 

All models are trained on the first 100,000 conversations of the dataset for one epoch, with the loss calculated exclusively on the response tokens. Training and evaluation are performed on 4 NVIDIA A100 80GB GPUs, while the rectification process is executed on a single NVIDIA A100 80GB GPU. We use \texttt{float16} precision for FF and \texttt{bfloat16} for the base weights of other FT methods, with adapter parameters maintained in \texttt{float32}.

% \subsection{Rectification Configuration}
% Unless otherwise specified, we set the default acceptance threshold $\tau$ to 0.95 and the decay factor $\beta$ to 0.7. The choice of $\beta=0.7$ ensures that the step size is approximately halved every two decay iterations.

\subsection{Evaluation Configuration}
We follow the default evaluation protocols for all benchmarks. For world knowledge and the Math task, no sensitive hyperparameters are involved in the evaluation process. 

For the Code and IF tasks, we explicitly specify the standard settings used to ensure transparency.
For the Code task, we strictly adhere to the officially recommended configurations\footnote{\url{https://github.com/bigcode-project/bigcode-evaluation-harness/blob/main/docs/README.md}}: specifically, a temperature of 0.2 with $n=200$ samples for HumanEval, and a temperature of 0.1 with $n=15$ samples for MBPP. All reported metrics for these benchmarks are \texttt{pass@1}. 

For the IF task, GPT-4 is employed as the judge model.
Although MTBench supports multi-turn dialogues, we report performance based only on the first turn, as the training data consists exclusively of single-turn conversations\textsuperscript{\ref{pissa footnote}}.

To facilitate reproduction, an evaluation script for each benchmark is provided in our repository.

\subsection{Data Replay}
For CorDA and LoRA-Null, the replay buffer consists of 256 samples randomly drawn from NQ open with a maximum length of 1024 tokens, consistent with their original implementations.
It should be noted that since individual entries in NQ open are typically much shorter than 1024 tokens, each sample comprises multiple concatenated data points.
For JANUS, we replay exactly 256 randomly selected data points from the same source.

\newpage
\section{Complete Main Results}
\label{sec: complete main results}
Listed in Tab.~\ref{tab:llama-2-7b-complete} and Tab.~\ref{tab:llama-3-8b-complete} are the complete results on LLaMA-2-7b and LLaMA-3-8b, respectively.
These results indicate that, despite being influenced by the baseline performance of the fine-tuned models, our JANUS rectification consistently achieves superior overall performance, significantly mitigating the stability-plasticity dilemma.

Tab.~\ref{tab:complete single step} provides the full results of the ablation study regarding the multi-step adaptive rectification mechanism. Across various base models, downstream tasks, and fine-tuning methods, this mechanism demonstrates a universal capability to enhance stability.

\begin{table}[htbp]
    \centering
    \caption{Complete performance comparison of various methods on LLaMA-2-7b across Math, Code, and IF tasks. For each column within a task, \textbf{bold} and \underline{underlined} values indicate the highest and second-highest scores among fine-tuned methods.}
    \label{tab:llama-2-7b-complete}
    \small 
    \tabcolsep=3.2pt % 进一步微调列间距以适应宽度
    \resizebox{\textwidth}{!}{
\begin{tabular}{l|c|r@{}l@{\ }r@{}l@{\ }r@{}l@{\ }r@{}l@{\ }r@{}l@{\ }r@{}l@{\ }r@{}l@{\ }r@{}l}
    \toprule
    Method & \#Param
    & \multicolumn{2}{c}{TriviaQA}
    & \multicolumn{2}{c}{NQ open}
    & \multicolumn{2}{c}{WebQS}
    & \multicolumn{2}{c}{$\text{AP}_1(\%)$}
    & \multicolumn{2}{c}{Bench 1}
    & \multicolumn{2}{c}{Bench 2}
    & \multicolumn{2}{c}{$\text{AP}_2(\%)$}
    & \multicolumn{2}{c}{$\text{AP}(\%)$} \\
    \midrule

    LLaMA-2-7b & --
    & \textcolor{gray}{52.52} &
    & \textcolor{gray}{18.95} &
    & \textcolor{gray}{5.81} &
    & \textcolor{gray}{100.00} &
    & \multicolumn{2}{c}{--}
    & \multicolumn{2}{c}{--}
    & \multicolumn{2}{c}{--}
    & \multicolumn{2}{c}{--} \\

    \midrule
    \multicolumn{18}{c}{\textbf{Task: Math} (Bench 1: GSM8k, Bench 2: Math)} \\
    \midrule

    FF & 6.7B
    & 19.19 &
    & 0.86 &
    & 4.58 &
    & 39.97 &
    & \textcolor{gray}{60.20} &
    & \textcolor{gray}{12.56} &
    & \textcolor{gray}{100.00} &
    & 69.98 & \\

    LoRA & 320M
    & 41.78 &
    & 1.50 &
    & 6.40 &
    & 65.87 &
    & 40.71 &
    & 4.76 &
    & 52.76 &
    & 59.32 & \\

    PiSSA & 320M
    & 41.57 & \textcolor{gray}{\scriptsize $\pm$1.50}
    & 3.34 & \textcolor{gray}{\scriptsize $\pm$0.18}
    & 6.15 & \textcolor{gray}{\scriptsize $\pm$0.52}
    & 67.55 & \textcolor{gray}{\scriptsize $\pm$3.85}
    & \underline{51.73} & \textcolor{gray}{\scriptsize $\pm$0.29}
    & \underline{7.43} & \textcolor{gray}{\scriptsize $\pm$0.35}
    & \underline{72.56} & \textcolor{gray}{\scriptsize $\pm$1.37}
    & {70.05} & \textcolor{gray}{\scriptsize $\pm$2.34} \\

    CorDA & 320M
    & 41.94 &
    & 7.09 &
    & 7.14 &
    & 80.05 &
    & 43.97 &
    & 6.18 &
    & 61.12 &
    & 70.59 & \\

    MiLoRA & 320M
    & 44.98 &
    & 3.13 &
    & 7.14 &
    & 75.02 &
    & 40.64 &
    & 5.04 &
    & 53.82 &
    & 64.42 & \\

    LoRA-Null & 320M
    & 44.64 &
    & 6.04 &
    & 7.23 &
    & 80.44 &
    & 42.30 &
    & 5.76 &
    & 58.06 &
    & 69.25 & \\

    \textbf{JANUS-F} & 6.7B
    & 35.57 &
    & 9.97 &
    & \underline{7.92} &
    & 85.55 &
    & \textbf{58.76} &
    & \textbf{12.28} &
    & \textbf{97.69} &
    & \textbf{91.62} & \\

    \textbf{JANUS-L} & 320M
    & \textbf{47.73} &
    & \textbf{17.78} &
    & 7.82 &
    & \textbf{106.43} &
    & 40.94 &
    & 5.00 &
    & 53.91 &
    & 80.17 & \\

    \textbf{JANUS-P} & 320M
    & \underline{47.31} & \textcolor{gray}{\scriptsize $\pm$0.57}
    & \underline{16.05} & \textcolor{gray}{\scriptsize $\pm$0.65}
    & \textbf{8.35} & \textcolor{gray}{\scriptsize $\pm$0.61}
    & \underline{106.16} & \textcolor{gray}{\scriptsize $\pm$4.98}
    & {51.20} & \textcolor{gray}{\scriptsize $\pm$0.57}
    & {7.29} & \textcolor{gray}{\scriptsize $\pm$0.29}
    & {71.54} & \textcolor{gray}{\scriptsize $\pm$1.49}
    & \underline{88.85} & \textcolor{gray}{\scriptsize $\pm$3.01} \\

    \midrule
    \multicolumn{18}{c}{\textbf{Task: Code} (Bench 1: HumanEval, Bench 2: MBPP)} \\
    \midrule

    FF & 6.7B
    & 37.61 &
    & 6.76 &
    & 6.59 &
    & 73.57 &
    & \textcolor{gray}{33.88} &
    & \textcolor{gray}{28.41} &
    & \textcolor{gray}{100.00} &
    & 86.78 & \\

    LoRA & 320M
    & 43.52 &
    & 9.22 &
    & 5.61 &
    & 76.03 &
    & 18.22 &
    & 23.51 &
    & 68.27 &
    & 72.15 & \\

    PiSSA & 320M
    & 44.93 & \textcolor{gray}{\scriptsize $\pm$0.25}
    & 10.43 & \textcolor{gray}{\scriptsize $\pm$0.65}
    & 5.73 & \textcolor{gray}{\scriptsize $\pm$0.18}
    & 79.71 & \textcolor{gray}{\scriptsize $\pm$2.20}
    & \underline{21.51} & \textcolor{gray}{\scriptsize $\pm$0.51}
    & \underline{25.11} & \textcolor{gray}{\scriptsize $\pm$0.30}
    & \underline{75.94} & \textcolor{gray}{\scriptsize $\pm$1.24}
    & 77.83 & \textcolor{gray}{\scriptsize $\pm$1.61} \\

    CorDA & 320M
    & 44.23 &
    & 13.57 &
    & 5.81 &
    & 85.28 &
    & 19.25 &
    & 21.89 &
    & 66.93 &
    & 76.10 & \\

    MiLoRA & 320M
    & 42.58 &
    & 11.83 &
    & 5.17 &
    & 77.50 &
    & 16.74 &
    & 21.39 &
    & 62.35 &
    & 69.92 & \\

    LoRA-Null & 320M
    & 46.03 &
    & 14.57 &
    & 5.91 &
    & 88.75 &
    & 18.59 &
    & 23.57 &
    & 68.92 &
    & 78.83 & \\

    \textbf{JANUS-F} & 6.7B
    & 48.56 &
    & \textbf{16.68} &
    & \textbf{9.79} &
    & \textbf{116.33} &
    & \textbf{33.05} &
    & \textbf{28.43} &
    & \textbf{98.81} &
    & \textbf{107.57} & \\

    \textbf{JANUS-L} & 320M
    & \textbf{50.40} &
    & \underline{16.43} &
    & 7.78 &
    & 105.52 &
    & 17.35 &
    & 23.95 &
    & 67.76 &
    & 86.64 & \\

    \textbf{JANUS-P} & 320M
    & \underline{50.21} & \textcolor{gray}{\scriptsize $\pm$0.45}
    & {16.39} & \textcolor{gray}{\scriptsize $\pm$0.43}
    & \underline{8.56} & \textcolor{gray}{\scriptsize $\pm$0.31}
    & \underline{109.83} & \textcolor{gray}{\scriptsize $\pm$2.51}
    & {21.01} & \textcolor{gray}{\scriptsize $\pm$0.88}
    & {25.10} & \textcolor{gray}{\scriptsize $\pm$0.48}
    & {75.19} & \textcolor{gray}{\scriptsize $\pm$0.80}
    & \underline{92.51} & \textcolor{gray}{\scriptsize $\pm$0.98} \\

    \midrule
    \multicolumn{18}{c}{\textbf{Task: IF} (Bench 1: MTBench, Bench 2: --)} \\
    \midrule

    FF & 6.7B
    & 20.38 &
    & 5.01 &
    & 4.97 &
    & 50.26 &
    & \textcolor{gray}{4.56} &
    & \multicolumn{2}{c}{--}
    & \textcolor{gray}{100.00} &
    & 75.13 & \\

    LoRA & 320M
    & 43.79 &
    & 8.25 &
    & 5.86 &
    & 75.92 &
    & 3.47 &
    & \multicolumn{2}{c}{--}
    & 76.10 &
    & 76.01 & \\

    PiSSA & 320M
    & 43.23 & \textcolor{gray}{\scriptsize $\pm$0.70}
    & 8.75 & \textcolor{gray}{\scriptsize $\pm$0.45}
    & 6.40 & \textcolor{gray}{\scriptsize $\pm$0.32}
    & 79.53 & \textcolor{gray}{\scriptsize $\pm$2.88}
    & \textbf{3.73} & \textcolor{gray}{\scriptsize $\pm$0.54}
    & \multicolumn{2}{c}{--}
    & \textbf{81.80} & \textcolor{gray}{\scriptsize $\pm$11.94}
    & {80.67} & \textcolor{gray}{\scriptsize $\pm$6.34} \\

    CorDA & 320M
    & 45.63 &
    & \underline{17.04} &
    & 6.89 &
    & 98.46 &
    & 3.30 &
    & \multicolumn{2}{c}{--}
    & 72.37 &
    & 85.42 & \\

    MiLoRA & 320M
    & 45.02 &
    & 10.28 &
    & 6.64 &
    & 84.75 &
    & 2.78 &
    & \multicolumn{2}{c}{--}
    & 60.96 &
    & 72.86 & \\

    LoRA-Null & 320M
    & 47.55 &
    & 12.96 &
    & 6.89 &
    & 92.51 &
    & 3.56 &
    & \multicolumn{2}{c}{--}
    & 78.07 &
    & 85.29 & \\

    \textbf{JANUS-F} & 6.7B
    & 32.78 &
    & 10.64 &
    & 6.40 &
    & 76.24 &
    & \underline{3.71} &
    & \multicolumn{2}{c}{--}
    & \underline{81.36} &
    & 78.80 & \\

    \textbf{JANUS-L} & 320M
    & \textbf{49.59} &
    & \textbf{17.45} &
    & \textbf{7.48} &
    & \textbf{105.08} &
    & 3.41 &
    & \multicolumn{2}{c}{--}
    & 74.78 &
    & \textbf{89.93} & \\

    \textbf{JANUS-P} & 320M
    & \underline{48.06} & \textcolor{gray}{\scriptsize $\pm$0.20}
    & {15.58} & \textcolor{gray}{\scriptsize $\pm$0.82}
    & \underline{7.38} & \textcolor{gray}{\scriptsize $\pm$0.25}
    & \underline{100.27} & \textcolor{gray}{\scriptsize $\pm$1.18}
    & {3.45} & \textcolor{gray}{\scriptsize $\pm$0.25}
    & \multicolumn{2}{c}{--}
    & {75.58} & \textcolor{gray}{\scriptsize $\pm$5.59}
    & \underline{87.93} & \textcolor{gray}{\scriptsize $\pm$2.26} \\

    \bottomrule
\end{tabular}
    }
\end{table}

\begin{table}[htbp]
    \centering
    \caption{Complete performance comparison of various methods on LLaMA-3-8b across Math, Code, and IF tasks. For each column within a task, \textbf{bold} and \underline{underlined} values indicate the highest and second-highest scores among fine-tuned methods.}
    \label{tab:llama-3-8b-complete}
    \small 
    \tabcolsep=3.2pt % 调整列间距以适应宽度
    \resizebox{\textwidth}{!}{
\begin{tabular}{l|c|r@{}l@{\ }r@{}l@{\ }r@{}l@{\ }r@{}l@{\ }r@{}l@{\ }r@{}l@{\ }r@{}l@{\ }r@{}l}
    \toprule
    Method & \#Param
    & \multicolumn{2}{c}{TriviaQA}
    & \multicolumn{2}{c}{NQ open}
    & \multicolumn{2}{c}{WebQS}
    & \multicolumn{2}{c}{$\text{AP}_1(\%)$}
    & \multicolumn{2}{c}{Bench 1}
    & \multicolumn{2}{c}{Bench 2}
    & \multicolumn{2}{c}{$\text{AP}_2(\%)$}
    & \multicolumn{2}{c}{$\text{AP}(\%)$} \\
    \midrule

    LLaMA-3-8b & --
    & \textcolor{gray}{61.66} &
    & \textcolor{gray}{21.86} &
    & \textcolor{gray}{9.94} &
    & \textcolor{gray}{100.00} &
    & \multicolumn{2}{c}{--}
    & \multicolumn{2}{c}{--}
    & \multicolumn{2}{c}{--}
    & \multicolumn{2}{c}{--} \\

    \midrule
    \multicolumn{18}{c}{\textbf{Task: Math} (Bench 1: GSM8k, Bench 2: Math)} \\
    \midrule

    FF & 8.37B
    & 41.43 &
    & 5.21 &
    & 4.13 &
    & 44.19 &
    & \textcolor{gray}{75.36} &
    & \textcolor{gray}{24.84} &
    & \textcolor{gray}{100.00} &
    & 72.10 & \\

    LoRA & 336M
    & 54.66 & \textcolor{gray}{\scriptsize $\pm$0.58}
    & 7.90 & \textcolor{gray}{\scriptsize $\pm$0.78}
    & 4.60 & \textcolor{gray}{\scriptsize $\pm$0.10}
    & 57.02 & \textcolor{gray}{\scriptsize $\pm$1.33}
    & 75.49 & \textcolor{gray}{\scriptsize $\pm$0.57}
    & 25.16 & \textcolor{gray}{\scriptsize $\pm$0.33}
    & 100.73 & \textcolor{gray}{\scriptsize $\pm$0.41}
    & 78.87 & \textcolor{gray}{\scriptsize $\pm$0.65} \\

    PiSSA & 336M
    & 30.22 &
    & 3.57 &
    & 3.79 &
    & 34.49 &
    & \underline{76.72} &
    & \underline{25.84} &
    & 102.92 &
    & 68.70 & \\

    CorDA & 336M
    & 49.76 &
    & 16.09 &
    & 4.63 &
    & 66.96 &
    & 76.35 &
    & \underline{25.84} &
    & 102.67 &
    & 84.82 & \\

    MiLoRA & 336M
    & 54.15 &
    & 15.32 &
    & 5.41 &
    & 70.78 &
    & 75.28 &
    & 24.42 &
    & 99.10 &
    & 84.94 & \\

    LoRA-Null & 336M
    & 53.93 &
    & 11.86 &
    & 5.66 &
    & 66.22 &
    & 76.50 &
    & \textbf{26.28} &
    & \textbf{103.65} &
    & 84.94 & \\

    \textbf{JANUS-F} & 8.37B
    & \textbf{59.21} &
    & \underline{19.00} &
    & \underline{6.94} &
    & \underline{84.25} &
    & 76.57 &
    & 25.02 &
    & 101.17 &
    & \underline{92.71} & \\

    \textbf{JANUS-L} & 336M
    & \underline{58.49} & \textcolor{gray}{\scriptsize $\pm$0.77}
    & \textbf{20.53} & \textcolor{gray}{\scriptsize $\pm$0.11}
    & \textbf{7.20} & \textcolor{gray}{\scriptsize $\pm$0.63}
    & \textbf{87.07} & \textcolor{gray}{\scriptsize $\pm$2.67}
    & 75.44 & \textcolor{gray}{\scriptsize $\pm$0.92}
    & 25.08 & \textcolor{gray}{\scriptsize $\pm$0.16}
    & 100.54 & \textcolor{gray}{\scriptsize $\pm$0.37}
    & \textbf{93.80} & \textcolor{gray}{\scriptsize $\pm$1.32} \\

    \textbf{JANUS-P} & 336M
    & 53.00 &
    & 17.89 &
    & 6.45 &
    & 77.56 &
    & \textbf{77.71} &
    & 25.74 &
    & \underline{103.37} &
    & 90.47 & \\

    \midrule
    \multicolumn{18}{c}{\textbf{Task: Code} (Bench 1: HumanEval, Bench 2: MBPP)} \\
    \midrule

    FF & 8.37B
    & 56.29 &
    & 14.02 &
    & 6.89 &
    & 74.91 &
    & \textcolor{gray}{43.11} &
    & \textcolor{gray}{45.69} &
    & \textcolor{gray}{100.00} &
    & 87.46 & \\

    LoRA & 336M
    & 58.54 & \textcolor{gray}{\scriptsize $\pm$0.33}
    & 13.30 & \textcolor{gray}{\scriptsize $\pm$0.68}
    & 5.74 & \textcolor{gray}{\scriptsize $\pm$0.13}
    & 71.17 & \textcolor{gray}{\scriptsize $\pm$0.97}
    & 50.05 & \textcolor{gray}{\scriptsize $\pm$0.97}
    & \underline{47.89} & \textcolor{gray}{\scriptsize $\pm$0.51}
    & 110.46 & \textcolor{gray}{\scriptsize $\pm$0.89}
    & 90.82 & \textcolor{gray}{\scriptsize $\pm$0.90} \\

    PiSSA & 336M
    & 54.43 &
    & 15.32 &
    & 6.45 &
    & 74.42 &
    & \underline{56.52} &
    & 46.08 &
    & \textbf{115.98} &
    & 95.20 & \\

    CorDA & 336M
    & 56.26 &
    & 17.37 &
    & 6.15 &
    & 77.52 &
    & 51.99 &
    & 45.88 &
    & 110.51 &
    & 94.02 & \\

    MiLoRA & 336M
    & \underline{60.28} &
    & {18.12} &
    & 8.07 &
    & 87.28 &
    & 50.24 &
    & \textbf{49.01} &
    & 111.90 &
    & \underline{99.59} & \\

    LoRA-Null & 336M
    & 56.96 &
    & 15.15 &
    & 6.59 &
    & 75.99 &
    & 53.44 &
    & {46.72} &
    & 113.11 &
    & 94.55 & \\

    \textbf{JANUS-F} & 8.37B
    & 59.25 &
    & \textbf{20.30} &
    & \textbf{9.30} &
    & \textbf{94.17} &
    & 41.22 &
    & 44.64 &
    & 96.66 &
    & 95.42 & \\

    \textbf{JANUS-L} & 336M
    & \textbf{60.69} & \textcolor{gray}{\scriptsize $\pm$0.13}
    & {19.93} & \textcolor{gray}{\scriptsize $\pm$0.18}
    & 7.78 & \textcolor{gray}{\scriptsize $\pm$0.15}
    & 89.28 & \textcolor{gray}{\scriptsize $\pm$0.27}
    & 49.66 & \textcolor{gray}{\scriptsize $\pm$1.13}
    & 47.43 & \textcolor{gray}{\scriptsize $\pm$0.40}
    & 109.50 & \textcolor{gray}{\scriptsize $\pm$0.98}
    & {99.39} & \textcolor{gray}{\scriptsize $\pm$0.50} \\

    \textbf{JANUS-P} & 336M
    & 57.42 &
    & \underline{19.94} &
    & \underline{8.86} &
    & \underline{91.16} &
    & \textbf{56.87} &
    & 45.61 &
    & \underline{115.87} &
    & \textbf{103.52} & \\

    \midrule
    \multicolumn{18}{c}{\textbf{Task: IF} (Bench 1: MTBench, Bench 2: --)} \\
    \midrule

    FF & 8.37B
    & 42.69 &
    & 10.08 &
    & 3.94 &
    & 51.66 &
    & \textcolor{gray}{5.94} &
    & \multicolumn{2}{c}{--}
    & \textcolor{gray}{100.00} &
    & 75.83 & \\

    LoRA & 336M
    & 58.78 & \textcolor{gray}{\scriptsize $\pm$0.62}
    & 15.38 & \textcolor{gray}{\scriptsize $\pm$0.24}
    & 6.18 & \textcolor{gray}{\scriptsize $\pm$0.25}
    & 75.96 & \textcolor{gray}{\scriptsize $\pm$0.93}
    & 6.05 & \textcolor{gray}{\scriptsize $\pm$0.21}
    & \multicolumn{2}{c}{--}
    & 101.85 & \textcolor{gray}{\scriptsize $\pm$3.50}
    & 88.91 & \textcolor{gray}{\scriptsize $\pm$1.55} \\

    PiSSA & 336M
    & 48.40 &
    & 14.71 &
    & 6.25 &
    & 69.55 &
    & \textbf{6.73} &
    & \multicolumn{2}{c}{--}
    & \textbf{113.30} &
    & 91.43 & \\

    CorDA & 336M
    & 53.18 &
    & 19.25 &
    & 4.68 &
    & 73.80 &
    & 5.88 &
    & \multicolumn{2}{c}{--}
    & 98.99 &
    & 86.39 & \\

    MiLoRA & 336M
    & \textbf{62.49} &
    & 17.92 &
    & \textbf{9.50} &
    & \textbf{92.97} &
    & 5.39 &
    & \multicolumn{2}{c}{--}
    & 90.74 &
    & 91.85 & \\

    LoRA-Null & 336M
    & 53.42 &
    & 15.32 &
    & 4.63 &
    & 67.77 &
    & 5.96 &
    & \multicolumn{2}{c}{--}
    & 100.34 &
    & 84.05 & \\

    \textbf{JANUS-F} & 8.37B
    & 56.21 &
    & \underline{20.17} &
    & 6.20 &
    & 81.93 &
    & 5.57 &
    & \multicolumn{2}{c}{--}
    & 93.77 &
    & 87.85 & \\

    \textbf{JANUS-L} & 336M
    & \underline{61.47} & \textcolor{gray}{\scriptsize $\pm$0.39}
    & \textbf{20.43} & \textcolor{gray}{\scriptsize $\pm$0.29}
    & \underline{8.00} & \textcolor{gray}{\scriptsize $\pm$0.62}
    & \underline{91.23} & \textcolor{gray}{\scriptsize $\pm$2.47}
    & 5.94 & \textcolor{gray}{\scriptsize $\pm$0.06}
    & \multicolumn{2}{c}{--}
    & 100.00 & \textcolor{gray}{\scriptsize $\pm$1.05}
    & \textbf{95.61} & \textcolor{gray}{\scriptsize $\pm$0.71} \\

    \textbf{JANUS-P} & 336M
    & 55.30 &
    & 18.75 &
    & 7.58 &
    & 83.91 &
    & \underline{6.31} &
    & \multicolumn{2}{c}{--}
    & \underline{106.23} &
    & \underline{95.07} & \\

    \bottomrule
\end{tabular}
    }
\end{table}

\begin{table}[htbp]
    \centering
    \caption{Complete results of the ablation study on the multi-step adaptive rectification mechanism.}
    \label{tab:complete single step}
    
    \resizebox{\textwidth}{!}{
    \begin{tabular}{lcccc|cccc}
        \toprule
        \multirow{3}{*}{Methods} & \multicolumn{4}{c}{LLaMA-2-7b} & \multicolumn{4}{c}{LLaMA-3-8b} \\
        \cmidrule(lr){2-5} \cmidrule(lr){6-9}
        & TriviaQA & NQ open & WebQS & $\text{AP}_1$(\%) & TriviaQA & NQ open & WebQS & $\text{AP}_1$(\%) \\
        \midrule
        \multicolumn{9}{c}{\textbf{Task: Math}} \\
        \midrule
        JANUS-F (w/) & \textbf{35.57} & \textbf{9.97} & \textbf{7.92} & \textbf{85.55} & \textbf{59.21} & \textbf{19.00} & \textbf{6.94} & \textbf{84.25} \\
        JANUS-F (w/o) & 28.01 & 4.49 & 5.56 & 57.57 & 50.72 & 8.34 & 4.43 & 54.99 \\
        \midrule
        JANUS-L (w/) & \textbf{47.73} & \textbf{17.78} & \textbf{7.82} & \textbf{106.43} & \textbf{59.38} & \textbf{20.64} & \textbf{7.92} & \textbf{90.13} \\
        JANUS-L (w/o) & 45.24 & 11.00 & 7.09 & 88.74 & 56.87 & 13.38 & 5.41 & 69.29 \\
        \midrule
        JANUS-P (w/) & \textbf{47.95} & \textbf{16.73} & \textbf{9.01} & \textbf{111.55} & \textbf{53.00} & \textbf{17.89} & \textbf{6.45} & \textbf{77.56} \\
        JANUS-P (w/o) & 45.98 & 11.16 & 7.48 & 91.73 & 45.04 & 9.86 & 4.68 & 55.08 \\
        \midrule
        \multicolumn{9}{c}{\textbf{Task: Code}} \\
        \midrule
        JANUS-F (w/) & \textbf{48.56} & \textbf{16.68} & \textbf{9.79} & \textbf{116.33} & \textbf{59.25} & \textbf{20.30} & \textbf{9.30} & \textbf{94.17} \\
        JANUS-F (w/o) & 46.50 & 12.85 & 9.20 & 104.90 & 58.30 & 18.12 & 8.27 & 86.88 \\
        \midrule
        JANUS-L (w/) & \textbf{50.40} & \textbf{16.43} & \textbf{7.78} & \textbf{105.52} & \textbf{60.78} & \textbf{19.94} & \textbf{7.78} & \textbf{89.35} \\
        JANUS-L (w/o) & 49.20 & 12.38 & 7.19 & 94.25 & 59.65 & 17.06 & 6.45 & 79.89 \\
        \midrule
        JANUS-P (w/) & \textbf{50.60} & \textbf{16.26} & \textbf{8.32} & \textbf{108.45} & \textbf{57.42} & \textbf{19.94} & \textbf{8.86} & \textbf{91.16} \\
        JANUS-P (w/o) & 49.53 & 13.16 & 7.73 & 98.93 & 56.25 & 19.09 & 7.23 & 83.76 \\
        \midrule
        \multicolumn{9}{c}{\textbf{Task: IF}} \\
        \midrule
        JANUS-F (w/) & \textbf{32.78} & \textbf{10.64} & \textbf{6.40} & \textbf{76.24} & \textbf{56.21} & \textbf{20.17} & \textbf{6.20} & \textbf{81.93} \\
        JANUS-F (w/o) & 27.78 & 6.12 & 5.41 & 59.43 & 49.33 & 12.63 & 4.08 & 59.61 \\
        \midrule
        JANUS-L (w/) & \textbf{49.59} & \textbf{17.45} & \textbf{7.48} & \textbf{105.08} & \textbf{61.91} & \textbf{20.66} & \textbf{8.66} & \textbf{94.01} \\
        JANUS-L (w/o) & 47.29 & 11.88 & 6.69 & 89.29 & 60.85 & 17.15 & 6.74 & 81.65 \\
        \midrule
        JANUS-P (w/) & \textbf{48.28} & \textbf{14.82} & \textbf{7.38} & \textbf{99.05} & \textbf{55.30} & \textbf{18.75} & \textbf{7.58} & \textbf{83.91} \\
        JANUS-P (w/o) & 46.81 & 10.97 & 6.74 & 87.67 & 50.77 & 16.81 & 6.20 & 73.87 \\
        \bottomrule
    \end{tabular}
    }
\end{table}

\clearpage
\section{More ablations}
\label{sec: more ablations}
\subsection{Acceptance threshold $\tau$ and decay factor $\beta$}
As illustrated in Fig.~\ref{fig:ap1 vs meancos}, a statistically significant positive correlation exists between the recovered stability and the JANUS shift.
To further validate the effectiveness of JANUS shift as a metric of the trust region, we apply the JANUS rectification with various acceptance thresholds $\tau$ on a series of settings, including (1) LLaMA-2-7b + Math + FF, (2) LLaMA-2-7b + Code + PiSSA, (3) LLaMA-3-8b + Math + FF, and (4) LLaMA-3-8b + IF + LoRA, covering across two models, three tasks, and three FT methods.
The performance comparison of these rectified models are summarized in Tab.~\ref{tab:ap1 vs tau} and Fig.~\ref{fig:ap1 vs tau}.
The results demonstrate that increasing $\tau$ effectively enhances stability while maintaining task-specific performance, which justifies adaptively choosing proper step size under the guide of the JANUS shift.
% Although there is a compelling monotonicity between $\text{AP}_1$ and $\tau$ when $\tau\le0.95$, we observe that a substantial $\tau$ such as $0.975$ leads to a drop in $\text{AP}_1$. This may be due to too large an acceptance ratio would result in a

\begin{table}[htbp]
    \centering
    \caption{Performance comparison of different rectified models under varying acceptance threshold. }
    \label{tab:ap1 vs tau}
    \small 
    \tabcolsep=3.2pt
    
    \begin{subtable}[t]{\textwidth}
        \centering
        \caption{LLaMA-2-7b + Math + FF}
        \resizebox{\textwidth}{!}{\begin{tabular}{l|cccc|ccc|c}
            \toprule
            Method & TriviaQA & NQ open & WebQS & $\text{AP}_1$(\%) & GSM8k & Math & $\text{AP}_2$(\%) & $\text{AP}$(\%) \\
            \midrule
            LLaMA-2-7b & \textcolor{gray}{52.52} & \textcolor{gray}{18.95} & \textcolor{gray}{5.81} & \textcolor{gray}{100.00} & -- & -- & -- & -- \\ 
            \midrule
            FF & 19.19 & 0.86 & 4.58 & 39.97 & \textcolor{gray}{60.20} & \textcolor{gray}{12.56} & \textcolor{gray}{100.00} & 69.98 \\
            JANUS-F ($\tau=0.7$) & 28.00& 4.57& 5.56& 57.71& 58.45& 12.32& 97.59& 77.65\\
            JANUS-F ($\tau=0.8$) & 32.91& 7.48& 7.04& 74.43& 58.61& 12.02& 96.53& 85.48\\
            JANUS-F ($\tau=0.9$) & 34.95& 9.64& 7.63& 82.91& 58.83& 12.24& 97.59& 90.25\\
            JANUS-F ($\tau=0.925$) & 35.58& 9.78& 7.87& 84.94& 58.83& 12.26& 97.67& 91.30\\
            JANUS-F ($\tau=0.95$) & 35.57& 9.97& 7.92& 85.55& 58.76& 12.28& 97.69& 91.62\\
            % JANUS-F ($\tau=0.975$) & 6.7B & 35.03& 9.11& 7.63& 82.03& 58.45& 12.24& 97.27& 89.65\\
            \bottomrule
        \end{tabular}}
    \end{subtable}
    
    \begin{subtable}[t]{\textwidth}
        \centering
        \caption{LLaMA-2-7b + Code + PiSSA}
        \resizebox{\textwidth}{!}{\begin{tabular}{l|cccc|ccc|c}
            \toprule
            Method & TriviaQA & NQ open & WebQS & $\text{AP}_1$(\%) & HumanEval & MBPP & $\text{AP}_2$(\%) & $\text{AP}$(\%) \\
            \midrule
            LLaMA-2-7b & \textcolor{gray}{52.52} & \textcolor{gray}{18.95} & \textcolor{gray}{5.81} & \textcolor{gray}{100.00} & -- & -- & -- & -- \\ 
            \midrule
            FF & 19.19 & 0.86 & 4.58 & 39.97 & \textcolor{gray}{33.88} & \textcolor{gray}{28.41} & \textcolor{gray}{100.00} & 86.78 \\
            % PiSSA & 320M & 45.19 & 10.28 & 5.71 & 79.52 & {21.95} & {25.25} & {76.83} & 78.18 \\
            JANUS-P ($\tau=0.7$) & 49.47& 13.07& 7.73& 98.74& 22.05& 25.55& 77.51& 88.12\\
            JANUS-P ($\tau=0.8$) & 49.47& 13.38& 7.73& 99.28& 21.86& 24.73& 75.78& 87.53\\
            JANUS-P ($\tau=0.9$) & 49.47& 13.38& 7.73& 99.28& 21.86& 24.73& 75.78& 87.53\\
            JANUS-P ($\tau=0.925$) & 50.65& 16.34& 8.32& 108.62& 21.67& 25.13& 76.21& 92.42\\
            JANUS-P ($\tau=0.95$) & {50.60} & {16.26} & {8.32} & {108.45} & {21.99} & {24.77} & {76.05} & {92.25}\\
            \bottomrule
        \end{tabular}}
    \end{subtable}

    \begin{subtable}[t]{\textwidth}
        \centering
        \caption{LLaMA-3-8b + Math + FF}
        \resizebox{\textwidth}{!}{\begin{tabular}{l|cccc|ccc|c}
            \toprule
            Method & TriviaQA & NQ open & WebQS & $\text{AP}_1$(\%) & GSM8k & Math & $\text{AP}_2$(\%) & $\text{AP}$(\%) \\
            \midrule
            LLaMA-3-8b & \textcolor{gray}{61.66} & \textcolor{gray}{21.86} & \textcolor{gray}{9.94} & \textcolor{gray}{100.00} & -- & -- & -- & -- \\ 
            \midrule
            
            FF & 41.43 & 5.21 & 4.13 & 44.19 & \textcolor{gray}{75.36} & \textcolor{gray}{24.84} & \textcolor{gray}{100.00} & 72.10 \\
            JANUS-F ($\tau=0.7$) & 50.62 & 8.31 & 4.38 & 54.72 & 76.35 & 24.84 & 100.66 & 77.69 \\
            JANUS-F ($\tau=0.8$) & 57.74 & 16.84 & 6.00 & 77.01 & 76.35 & 24.38 & 99.73 & 88.37 \\
            JANUS-F ($\tau=0.9$) & 59.42 & 19.09 & 6.99 & 84.67 & 76.65 & 24.98 & 101.14 & 92.91 \\
            JANUS-F ($\tau=0.925$) & 59.42 & 19.09 & 6.99 & 84.67 & 76.65 & 24.98 & 101.14 & 92.91 \\
            JANUS-F ($\tau=0.95$) & 59.21 & 19.00 & 6.94 & 84.25 & 76.57 & 25.02 & 101.17 & 92.71 \\
            \bottomrule
        \end{tabular}}
    \end{subtable}
    
    \begin{subtable}[t]{\textwidth}
        \centering
        \caption{LLaMA-3-8b + IF + LoRA}
        \resizebox{\textwidth}{!}{\begin{tabular}{l|cccc|cc|c}
            \toprule
            Method & TriviaQA & NQ open & WebQS & $\text{AP}_1$(\%) & MTBench & $\text{AP}_2$(\%) & $\text{AP}$(\%) \\
            \midrule
            LLaMA-3-8b & \textcolor{gray}{61.66} & \textcolor{gray}{21.86} & \textcolor{gray}{9.94} & \textcolor{gray}{100.00} & -- & -- & -- \\ 
            \midrule
            
            FF & 42.69 & 10.08 & 3.94 & 51.66 & \textcolor{gray}{5.94} & \textcolor{gray}{100.00} & 75.83 \\
            JANUS-L ($\tau=0.7$) & 60.95 & 17.12 & 6.74 & 81.66 & 5.6 & 94.28 & 87.97 \\
            JANUS-L ($\tau=0.8$) & 60.95 & 17.12 & 6.74 & 81.66 & 5.61 & 94.44 & 88.05 \\
            JANUS-L ($\tau=0.9$) & 61.71 & 20.78 & 8.71 & 94.26 & 5.59 & 94.11 & 94.18 \\
            JANUS-L ($\tau=0.925$) & 61.71 & 20.78 & 8.71 & 94.26 & 5.68 & 95.62 & 94.94 \\
            JANUS-L ($\tau=0.95$) & {61.91} & {20.66} & {8.66} & {94.01} & 5.87 & 98.82 & {96.42} \\
            \bottomrule
        \end{tabular}}
    \end{subtable}
\end{table}

\begin{figure}[htbp]
  \centering
  \begin{subfigure}{0.49\textwidth}
  \includegraphics[width=\textwidth,trim=4cm 0cm 2cm 1.5cm, clip]{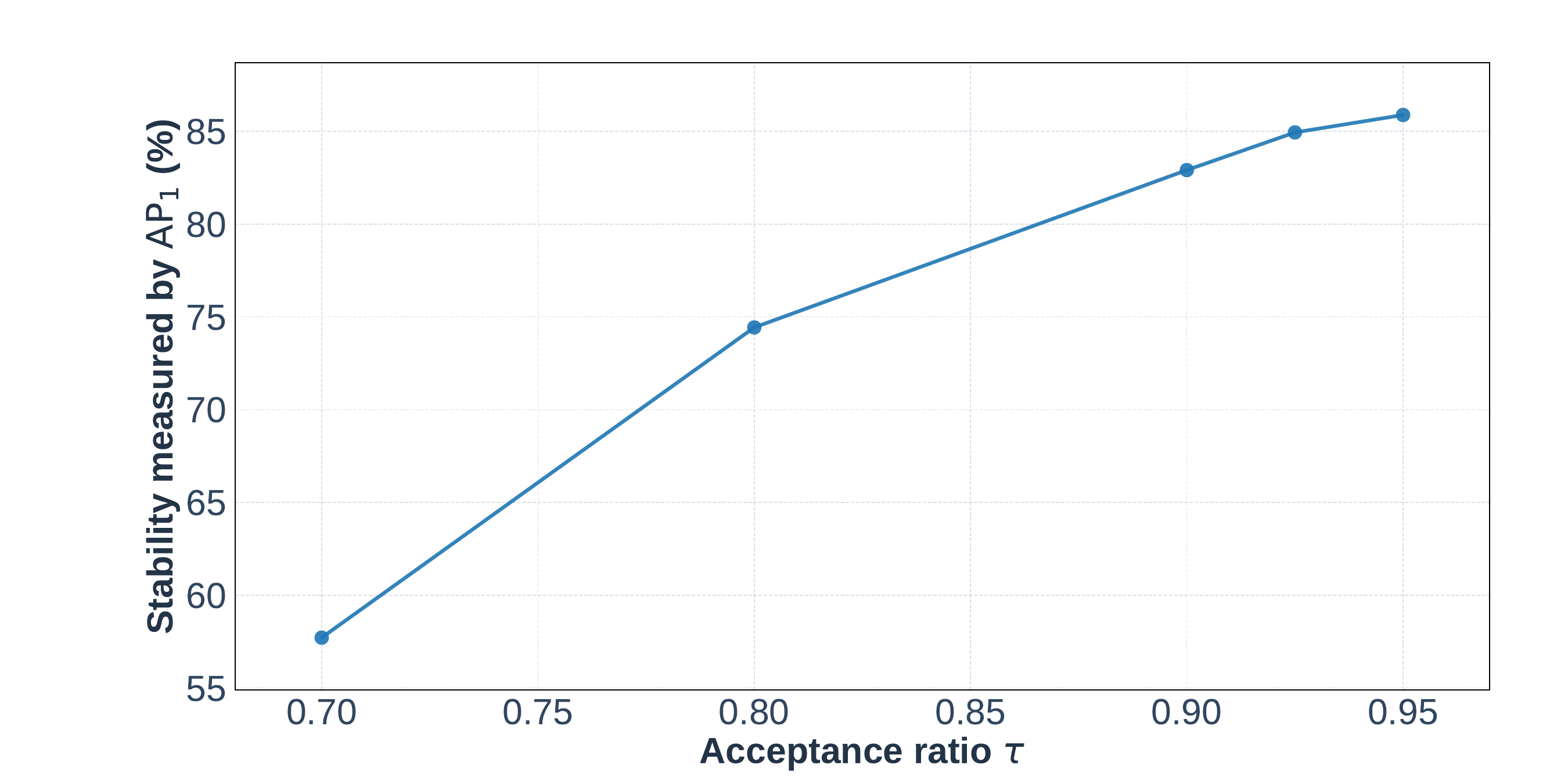}
  \caption{LLaMA-2-7b + Math + FF}
  \end{subfigure}
  \hfill
  \begin{subfigure}{0.49\textwidth}
  \includegraphics[width=\textwidth,trim=4cm 0cm 2cm 1.5cm, clip]{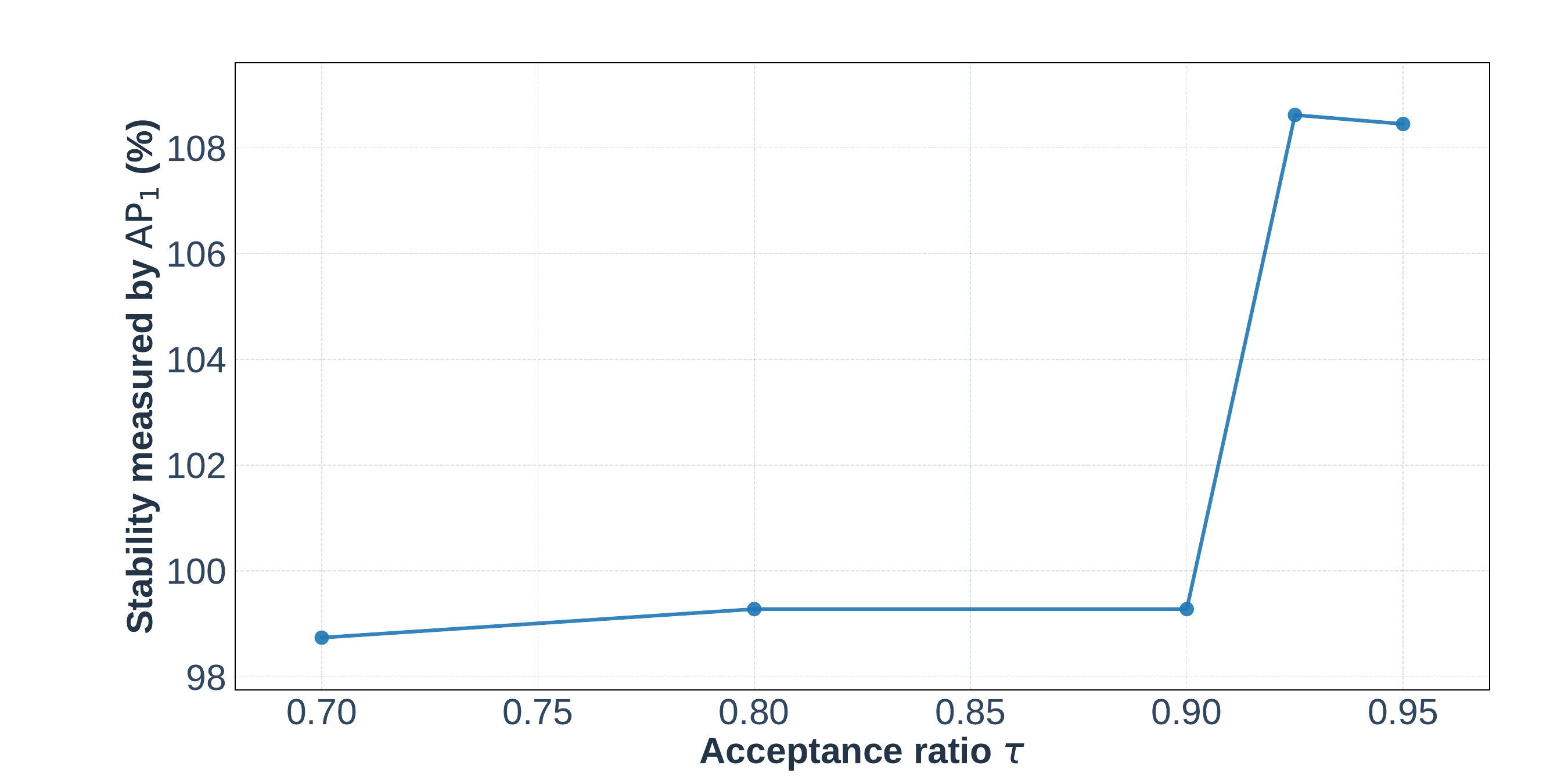}
  \caption{LLaMA-2-7b + Code + PiSSA}
  \end{subfigure}
  \\
  \begin{subfigure}{0.49\textwidth}
  \includegraphics[width=\textwidth,trim=4cm 0cm 2cm 1.5cm, clip]{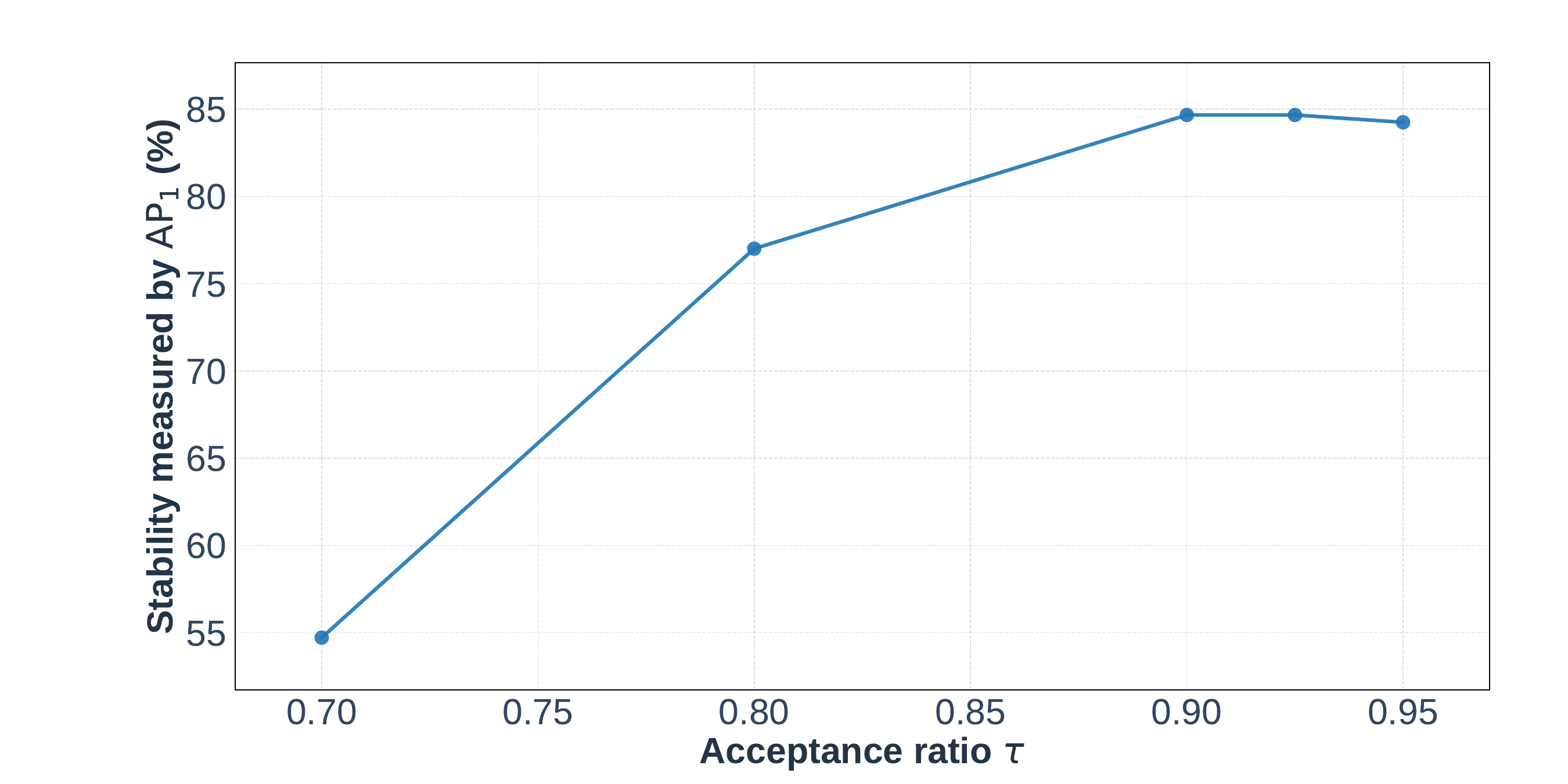}
  \caption{LLaMA-3-8b + Math + FF}
  \end{subfigure}
  \hfill
  \begin{subfigure}{0.49\textwidth}
  \includegraphics[width=\textwidth,trim=4cm 0cm 2cm 1.5cm, clip]{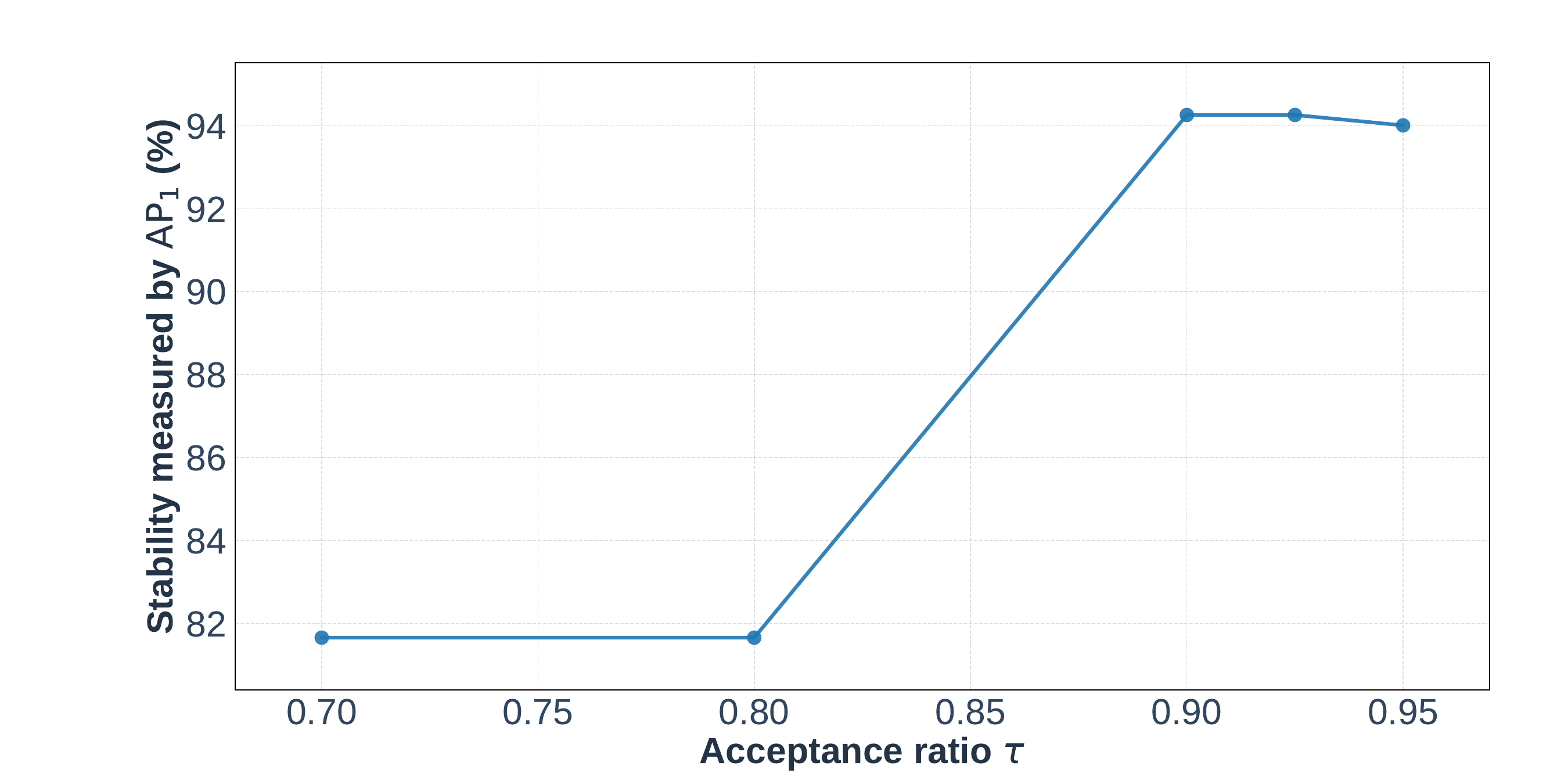}
  \caption{LLaMA-3-8b + IF + LoRA}
  \end{subfigure}
  \caption{Impact of the acceptance threshold $\tau$ on the stability of the rectified model.}
  % todo: 加图
\label{fig:ap1 vs tau}
\end{figure}

Regarding the decay factor $\beta$, Tab.~\ref{tab:ap1 vs beta} demonstrates that our method is highly robust to variations in this parameter, with performance remaining consistently high across the wide tested range.

\begin{table}[htbp]
    \centering
    \caption{Performance comparison of different rectified models under varying decay factor $\beta$. }
    \label{tab:ap1 vs beta}
    \small 
    \tabcolsep=3.2pt
    \begin{subtable}[t]{\textwidth}
        \centering
        \caption{LLaMA-2-7b + Math + FF}
        \resizebox{\textwidth}{!}{\begin{tabular}{l|cccc|ccc|c}
            \toprule
            Method & TriviaQA & NQ open & WebQS & $\text{AP}_1$(\%) & GSM8k & Math & $\text{AP}_2$(\%) & $\text{AP}$(\%) \\
            \midrule
            LLaMA-2-7b & \textcolor{gray}{52.52} & \textcolor{gray}{18.95} & \textcolor{gray}{5.81} & \textcolor{gray}{100.00} & -- & -- & -- & -- \\ 
            \midrule
            FF & 19.19 & 0.86 & 4.58 & 39.97 & \textcolor{gray}{60.20} & \textcolor{gray}{12.56} & \textcolor{gray}{100.00} & 69.98 \\
            JANUS-F ($\beta=0.6$) & 35.62 & 9.81 & 7.78 & 84.50 & 58.61 & 12.20 & 97.25 & 90.87 \\
            JANUS-F ($\beta=0.7$) & 35.57 & 9.97 & 7.92 & 85.55 & 58.76 & 12.28 & 97.69 & 91.62 \\
            JANUS-F ($\beta=0.8$) & 34.64 & 8.95 & 7.63 & 81.50 & 58.98 & 12.10 & 97.16 & 89.33 \\
            JANUS-F ($\beta=0.9$) & 35.46 & 9.58 & 7.78 & 83.99 & 58.76 & 12.06 & 96.81 & 90.40 \\
            \bottomrule
        \end{tabular}}
    \end{subtable}

    \begin{subtable}[t]{\textwidth}
        \centering
        \caption{LLaMA-2-7b + Code + PiSSA}
        \resizebox{\textwidth}{!}{\begin{tabular}{l|cccc|ccc|c}
            \toprule
            Method & TriviaQA & NQ open & WebQS & $\text{AP}_1$(\%) & HumanEval & MBPP & $\text{AP}_2$(\%) & $\text{AP}$(\%) \\
            \midrule
            LLaMA-2-7b & \textcolor{gray}{52.52} & \textcolor{gray}{18.95} & \textcolor{gray}{5.81} & \textcolor{gray}{100.00} & -- & -- & -- & -- \\ 
            \midrule
            FF & 19.19 & 0.86 & 4.58 & 39.97 & \textcolor{gray}{33.88} & \textcolor{gray}{28.41} & \textcolor{gray}{100.00} & 86.78 \\
            % PiSSA & 320M & 45.19 & 10.28 & 5.71 & 79.52 & {21.95} & {25.25} & {76.83} & 78.18 \\
            JANUS-P ($\beta=0.6$) & {50.50} & {15.62} & {8.22} & {106.69} & {21.88} & {25.28} & {76.78} & {91.73}\\
            JANUS-P ($\beta=0.7$) & {50.60} & {16.26} & {8.32} & {108.45} & {21.99} & {24.77} & {76.05} & {92.25}\\
            JANUS-P ($\beta=0.8$) & {50.53} & {16.15} & {8.42} & {108.79} & {21.61} & {24.27} & {74.61} & {91.70}\\
            JANUS-P ($\beta=0.9$) & {50.39} & {15.10} & {8.07} & {104.84} & {21.86} & {24.63} & {75.61} & {90.23}\\
            \bottomrule
        \end{tabular}}
    \end{subtable}

    \begin{subtable}[t]{\textwidth}
        \centering
        \caption{LLaMA-3-8b + Math + FF}
        \resizebox{\textwidth}{!}{\begin{tabular}{l|cccc|ccc|c}
            \toprule
            Method & TriviaQA & NQ open & WebQS & $\text{AP}_1$(\%) & GSM8k & Math & $\text{AP}_2$(\%) & $\text{AP}$(\%) \\
            \midrule
            LLaMA-3-8b & \textcolor{gray}{61.66} & \textcolor{gray}{21.86} & \textcolor{gray}{9.94} & \textcolor{gray}{100.00} & -- & -- & -- & -- \\ 
            \midrule
            
            FF & 41.43 & 5.21 & 4.13 & 44.19 & \textcolor{gray}{75.36} & \textcolor{gray}{24.84} & \textcolor{gray}{100.00} & 72.10 \\
            JANUS-F ($\beta=0.6$) & 59.09 & 18.06 & 6.89 & 82.59 & 76.35 & 24.72 & 100.42 & 91.50 \\
            JANUS-F ($\beta=0.7$) & 59.21 & 19.00 & 6.94 & 84.25 & 76.57 & 25.02 & 101.17 & 92.71 \\
            JANUS-F ($\beta=0.8$) & 59.43 & 18.17 & 7.53 & 85.09 & 75.74 & 24.76 & 100.09 & 92.59 \\
            JANUS-F ($\beta=0.9$) & 59.06 & 17.78 & 6.94 & 82.31 & 76.19 & 24.80 & 100.47 & 91.39 \\
            \bottomrule
        \end{tabular}}
    \end{subtable}

    \begin{subtable}[t]{\textwidth}
        \centering
        \caption{LLaMA-3-8b + IF + LoRA}
        \resizebox{\textwidth}{!}{\begin{tabular}{l|cccc|cc|c}
            \toprule
            Method & TriviaQA & NQ open & WebQS & $\text{AP}_1$(\%) & MTBench & $\text{AP}_2$(\%) & $\text{AP}$(\%) \\
            \midrule
            LLaMA-3-8b & \textcolor{gray}{61.66} & \textcolor{gray}{21.86} & \textcolor{gray}{9.94} & \textcolor{gray}{100.00} & -- & -- & -- \\ 
            \midrule
            
            FF & 42.69 & 10.08 & 3.94 & 51.66 & \textcolor{gray}{5.94} & \textcolor{gray}{100.00} & 75.83 \\
            JANUS-L ($\beta=0.6$) & {61.93} & {20.06} & {8.66} & {93.11} & 5.69 & 95.79 & {94.45} \\
            JANUS-L ($\beta=0.7$) & {61.91} & {20.66} & {8.66} & {94.01} & 5.87 & 98.82 & {96.42} \\
            JANUS-L ($\beta=0.8$) & {61.79} & {20.33} & {8.81} & {93.95} & 5.76 & 96.97 & {95.46} \\
            JANUS-L ($\beta=0.9$) & {61.66} & {19.78} & {8.46} & {91.87} & 5.73 & 96.46 & {94.16} \\
            \bottomrule
        \end{tabular}}
    \end{subtable}
\end{table}

\subsection{Compressed rank $r$}
We apply the JANUS rectification on a LLaMA-2-7b model fine-tuned on the Math task via FF with various compressed ranks $r$.
The results in Tab.~\ref{tab: ablation r} show a noticeable stability effect only at $r=4$.
During compression, we need to perform SVD and approximation on the matrix $\bm{M} = \bm{R}_A \bm{R}_X^\top$; we thus analyze its singular value spectrum.
We save the complete layer inputs $\bm X$ and pre-activation gradients $\bm A$ for \texttt{q\_proj} and \texttt{up\_proj} in layers 0, 15, and 31.
The energy ratios captured by the first $r$ singular values, ${\sum_{i=1}^r\sigma_i^2}/{\sum_i\sigma_i^2}$, are summarized in Tab.~\ref{tab: energy ratio regarding r}.
The results show that the singular values decay rapidly and that low ranks preserve most of the matrix energy.

\begin{table}[htbp]
    \centering
    \caption{Performance comparison of different rectified models under varying compressed rank $r$.}
    \label{tab: ablation r}
    \small 
    \tabcolsep=3.2pt % 调整列间距以适应宽度
    \centering
     \resizebox{\textwidth}{!}{\begin{tabular}{l|cccc|ccc|c}
    \toprule
    Method & TriviaQA & NQ open & WebQS & $\text{AP}_1(\%)$ & GSM8k & Math & $\text{AP}_2(\%)$ & $\text{AP}(\%)$ \\
    \midrule
    LLaMA-2-7b & \textcolor{gray}{52.52} & \textcolor{gray}{18.95} & \textcolor{gray}{5.81} & \textcolor{gray}{100.00} & -- & -- & -- & -- \\
    \midrule
    FF & 19.19 & 0.86 & 4.58 & 39.97 & \textcolor{gray}{60.20} & \textcolor{gray}{12.56} & \textcolor{gray}{100.00} & 69.98 \\
    JANUS-F ($r=32$) & 35.57 & 9.97 & 7.92 & 85.55 & 58.76 & 12.28 & 97.69 & 91.62 \\
    JANUS-F ($r=16$) & 34.75 & 9.14 & 7.73 & 82.48 & 58.91 & 12.06 & 96.94 & 89.71 \\
    JANUS-F ($r=8$) & 34.77 & 9.17 & 7.78 & 82.83 & 58.98 & 12.26 & 97.79 & 90.31 \\
    JANUS-F ($r=4$) & 33.04 & 7.59 & 7.19 & 75.57 & 58.45 & 12.24 & 97.27 & 86.42 \\
    \bottomrule
\end{tabular}}
\end{table}

\begin{table}[htbp]
    \centering
    \caption{Energy ratios of $\bm{M} = \bm{R}_A \bm{R}_X^\top$ captured by the first $r$ singular values.}
    \label{tab: energy ratio regarding r}
    % \small 
    \centering
    \begin{tabular}{l|cccc}
        \toprule
        Layer & $r=4$ (\%) & $r=8$ (\%) & $r=16$ (\%) & $r=32$ (\%) \\
        \midrule
        \texttt{0-q\_proj} & 98.916 & 99.855 & 99.995 & 99.999 \\
        \texttt{0-up\_proj} & 88.351 & 95.574 & 99.412 & 99.998 \\
        \texttt{15-q\_proj} & 80.505 & 93.770 & 97.909 &  99.999 \\
        \texttt{15-up\_proj} & 68.238 & 86.126 & 97.909 & 99.997 \\
        \texttt{31-q\_proj} & 83.841 & 93.876 & 99.428 & 99.998 \\
        \texttt{31-up\_proj} & 60.587 & 81.848 & 97.056 & 99.999 \\
        \bottomrule
    \end{tabular}
\end{table}

\subsection{Replay size}
We apply the JANUS rectification on a LLaMA-2-7b model fine-tuned on the Math task via FF with various replay sizes $m$, as summarized in Tab.~\ref{tab: ablation replay size}.
As the replay-buffer size decreases, recovery of old-task performance also decreases, as expected: 256 samples are already only a small batch, and further reduction weakens the buffer's representativeness of the full task distribution.
We also inspect the singular value spectrum of the full Jacobians constructed by 256 samples from layers 0, 15, and 31 for \texttt{q\_proj} and \texttt{up\_proj}.
The energy ratios captured by the first $r$ singular values are listed in Tab.~\ref{tab: energy ratio regarding m}.
The results empirically show that, gradients across samples are highly correlated, so a small-sample Jacobian can capture dominant directions of the full Jacobian.

\begin{table}[htbp]
    \centering
    \caption{Performance comparison of different rectified models under varying replay size $m$.}
    \label{tab: ablation replay size}
    \small 
    \tabcolsep=3.2pt % 调整列间距以适应宽度
    \centering
     \resizebox{\textwidth}{!}{\begin{tabular}{l|cccc|ccc|c}
    \toprule
    Method & TriviaQA & NQ open & WebQS & $\text{AP}_1(\%)$ & GSM8k & Math & $\text{AP}_2(\%)$ & $\text{AP}(\%)$ \\
    \midrule
    LLaMA-2-7b & \textcolor{gray}{52.52} & \textcolor{gray}{18.95} & \textcolor{gray}{5.81} & \textcolor{gray}{100.00} & -- & -- & -- & -- \\
    \midrule
    FF & 19.19 & 0.86 & 4.58 & 39.97 & \textcolor{gray}{60.20} & \textcolor{gray}{12.56} & \textcolor{gray}{100.00} & 69.98 \\
    JANUS-F ($m=256$) & 35.57 & 9.97 & 7.92 & 85.55 & 58.76 & 12.28 & 97.69 & 91.62 \\
    JANUS-F ($m=128$) & 32.91& 7.59& 7.53& 77.44& 58.68& 12.26& 97.54& 87.49\\
    JANUS-F ($m=64$) & 30.00& 4.35& 6.40& 63.41& 58.83& 12.14& 97.19& 80.30\\
    \bottomrule
\end{tabular}}
\end{table}

\begin{table}[htbp]
    \centering
    \caption{Energy ratios of $\bm J$ captured by the first $r$ singular values.}
    \label{tab: energy ratio regarding m}
    % \small 
    \centering
    \begin{tabular}{l|cccc}
        \toprule
        Layer & $r=16$ (\%)& $r=32$ (\%)& $r=64$ (\%)& $r=128$ (\%)\\
        \midrule
        \texttt{0-q\_proj} & 97.10& 98.46& 99.32& 99.81\\
        \texttt{0-up\_proj} & 79.77& 86.27& 92.17& 97.08\\
        \texttt{15-q\_proj} & 57.21& 66.89& 77.67&  89.80\\
        \texttt{15-up\_proj} & 57.51& 64.12& 73.50& 86.32\\
        \texttt{31-q\_proj} & 72.81& 77.87& 84.49& 92.43\\
        \texttt{31-up\_proj} & 58.08& 63.86& 72.75& 85.60\\
        \bottomrule
    \end{tabular}
\end{table}

\subsection{Replay distribution}
We apply the JANUS rectification on a LLaMA-2-7b model fine-tuned on the Math task via FF with varying replay buffers sampled from TriviaQA and WebQS, as listed in Tab.~\ref{tab: ablation replay dist}.
The results show that sampling replay data from any one of the three tasks protects that task reasonably well, while cross-task generalization varies with the source distribution.

\begin{table}[H]
    \centering
    \caption{Performance comparison of different rectified models under varying replay distribution.}
    \label{tab: ablation replay dist}
    \small 
    \tabcolsep=3.2pt % 调整列间距以适应宽度
    \centering
     \resizebox{\textwidth}{!}{\begin{tabular}{l|cccc|ccc|c}
    \toprule
    Method & TriviaQA & NQ open & WebQS & $\text{AP}_1(\%)$ & GSM8k & Math & $\text{AP}_2(\%)$ & $\text{AP}(\%)$ \\
    \midrule
    LLaMA-2-7b & \textcolor{gray}{52.52} & \textcolor{gray}{18.95} & \textcolor{gray}{5.81} & \textcolor{gray}{100.00} & -- & -- & -- & -- \\
    \midrule
    FF & 19.19 & 0.86 & 4.58 & 39.97 & \textcolor{gray}{60.20} & \textcolor{gray}{12.56} & \textcolor{gray}{100.00} & 69.98 \\
    JANUS-F (NQ open)& 35.57 & 9.97 & 7.92 & 85.55 & 58.76 & 12.28 & 97.69 & 91.62 \\
    JANUS-F (TriviaQA)& 35.66& 5.21& 6.30& 67.94& 59.29& 12.12& 97.49& 82.72\\
    JANUS-F (WebQS)& 28.43& 2.33& 6.00& 56.57& 59.14& 12.22& 97.77& 77.17\\
    \bottomrule
\end{tabular}}
\end{table}
% \clearpage
% \section{Visualization of A Complete Multi-step Adaptive Rectification Process}

\clearpage
\section*{NeurIPS Paper Checklist}

\begin{enumerate}

\item {\bf Claims}
    \item[] Question: Do the main claims made in the abstract and introduction accurately reflect the paper's contributions and scope?
    \item[] Answer: \answerYes{} % Replace by \answerYes{}, \answerNo{}, or \answerNA{}.
    \item[] Justification: \ref{sec: intro}

\item {\bf Limitations}
    \item[] Question: Does the paper discuss the limitations of the work performed by the authors?
    \item[] Answer: \answerYes{} % Replace by \answerYes{}, \answerNo{}, or \answerNA{}.
    \item[] Justification: \ref{sec: conclusion}

\item {\bf Theory assumptions and proofs}
    \item[] Question: For each theoretical result, does the paper provide the full set of assumptions and a complete (and correct) proof?
    \item[] Answer: \answerYes{} % Replace by \answerYes{}, \answerNo{}, or \answerNA{}.
    \item[] Justification: \ref{sec:theoretical_analysis}, \ref{sec:proof of thm global_forgetting_free}

    \item {\bf Experimental result reproducibility}
    \item[] Question: Does the paper fully disclose all the information needed to reproduce the main experimental results of the paper to the extent that it affects the main claims and/or conclusions of the paper (regardless of whether the code and data are provided or not)?
    \item[] Answer: \answerYes{} % Replace by \answerYes{}, \answerNo{}, or \answerNA{}.
    \item[] Justification: \ref{sec: experimental setup}, \ref{sec: implementation details}

\item {\bf Open access to data and code}
    \item[] Question: Does the paper provide open access to the data and code, with sufficient instructions to faithfully reproduce the main experimental results, as described in supplemental material?
    \item[] Answer: \answerYes{} % Replace by \answerYes{}, \answerNo{}, or \answerNA{}.
    \item[] Justification: \ref{sec: experimental setup}, \ref{sec: implementation details}

\item {\bf Experimental setting/details}
    \item[] Question: Does the paper specify all the training and test details (e.g., data splits, hyperparameters, how they were chosen, type of optimizer) necessary to understand the results?
    \item[] Answer: \answerYes{} % Replace by \answerYes{}, \answerNo{}, or \answerNA{}.
    \item[] Justification: \ref{sec: experimental setup}, \ref{sec: implementation details}

\item {\bf Experiment statistical significance}
    \item[] Question: Does the paper report error bars suitably and correctly defined or other appropriate information about the statistical significance of the experiments?
    \item[] Answer: \answerYes{} % Replace by \answerYes{}, \answerNo{}, or \answerNA{}.
    \item[] Justification: \ref{sec: main results}, \ref{sec: complete main results}

\item {\bf Experiments compute resources}
    \item[] Question: For each experiment, does the paper provide sufficient information on the computer resources (type of compute workers, memory, time of execution) needed to reproduce the experiments?
    \item[] Answer: \answerYes{} % Replace by \answerYes{}, \answerNo{}, or \answerNA{}.
    \item[] Justification: \ref{sec: implementation details}
    
\item {\bf Code of ethics}
    \item[] Question: Does the research conducted in the paper conform, in every respect, with the NeurIPS Code of Ethics \url{https://neurips.cc/public/EthicsGuidelines}?
    \item[] Answer: \answerYes{} % Replace by \answerYes{}, \answerNo{}, or \answerNA{}.
    \item[] Justification: This paper conforms with the NeurIPS Code of Ethics in every respect.

\item {\bf Broader impacts}
    \item[] Question: Does the paper discuss both potential positive societal impacts and negative societal impacts of the work performed?
    \item[] Answer: \answerNA{} % Replace by \answerYes{}, \answerNo{}, or \answerNA{}.
    \item[] Justification: Our study addresses the catastrophic forgetting problem in fine-tuning large models. We only use public available datasets. There is no societal impact of the work performed.
    
\item {\bf Safeguards}
    \item[] Question: Does the paper describe safeguards that have been put in place for responsible release of data or models that have a high risk for misuse (e.g., pre-trained language models, image generators, or scraped datasets)?
    \item[] Answer: \answerNA{} % Replace by \answerYes{}, \answerNo{}, or \answerNA{}.
    \item[] Justification: This paper poses no such risks.

\item {\bf Licenses for existing assets}
    \item[] Question: Are the creators or original owners of assets (e.g., code, data, models), used in the paper, properly credited and are the license and terms of use explicitly mentioned and properly respected?
    \item[] Answer: \answerYes{} % Replace by \answerYes{}, \answerNo{}, or \answerNA{}.
    \item[] Justification: \ref{sec: experimental setup}, \ref{sec: implementation details}

\item {\bf New assets}
    \item[] Question: Are new assets introduced in the paper well documented and is the documentation provided alongside the assets?
    \item[] Answer: \answerNA{} % Replace by \answerYes{}, \answerNo{}, or \answerNA{}.
    \item[] Justification: This paper does not release new assets.

\item {\bf Crowdsourcing and research with human subjects}
    \item[] Question: For crowdsourcing experiments and research with human subjects, does the paper include the full text of instructions given to participants and screenshots, if applicable, as well as details about compensation (if any)? 
    \item[] Answer: \answerNA{} % Replace by \answerYes{}, \answerNo{}, or \answerNA{}.
    \item[] Justification: This paper does not involve crowdsourcing nor research with human subjects.

\item {\bf Institutional review board (IRB) approvals or equivalent for research with human subjects}
    \item[] Question: Does the paper describe potential risks incurred by study participants, whether such risks were disclosed to the subjects, and whether Institutional Review Board (IRB) approvals (or an equivalent approval/review based on the requirements of your country or institution) were obtained?
    \item[] Answer: \answerNA{} % Replace by \answerYes{}, \answerNo{}, or \answerNA{}.
    \item[] Justification: This paper does not involve crowdsourcing nor research with human subjects.

\item {\bf Declaration of LLM usage}
    \item[] Question: Does the paper describe the usage of LLMs if it is an important, original, or non-standard component of the core methods in this research? Note that if the LLM is used only for writing, editing, or formatting purposes and does \emph{not} impact the core methodology, scientific rigor, or originality of the research, declaration is not required.
    %this research? 
    \item[] Answer: \answerNA{} % Replace by \answerYes{}, \answerNo{}, or \answerNA{}.
    \item[] Justification: The LLM is used only for writing. This paper does not involve LLMs as any important, original, or non-standard components.

\end{enumerate}

\end{document}